\documentclass{article}

    \PassOptionsToPackage{numbers, compress}{natbib}

    \usepackage[final]{neurips_2022}

\usepackage[utf8]{inputenc} 
\usepackage[T1]{fontenc}    
\usepackage{hyperref}       
\usepackage{url}            
\usepackage{booktabs}       
\usepackage{amsfonts}       
\usepackage{nicefrac}       
\usepackage{microtype}      
\usepackage{color}
\usepackage{soul}
\usepackage{caption}
\usepackage{xcolor}         

\usepackage{pifont}

\usepackage{amsmath}
\usepackage{graphicx}
\usepackage{titlesec}
\usepackage{pifont}
\usepackage[caption=false]{subfig}
\usepackage{tikz}
\usepackage{multicol}
\usepackage{wrapfig}
\usepackage{blindtext}
\usepackage{pifont}
\usepackage{enumitem}
\usepackage{bm}
\usepackage{amssymb}
\usepackage{amsthm}
\title{Fairness Reprogramming}
\usepackage{mathrsfs}
\usepackage{MnSymbol,bbding,pifont}
\usepackage{multirow}
\usepackage{makecell}

\newtheorem{theorem}{Theorem}

\newtheorem{lemma}{Lemma}[theorem]

\newcommand{\e}[1]{{\small $#1$}}

\newcommand{\ours}{\textsc{FairReprogram}}

\definecolor{Sijia_color}{rgb}{0.858, 0.188, 0.478}

\newcommand{\adv}{\textsc{AdvIn}}
\newcommand{\advp}{\textsc{AdvPost}}
\newcommand{\erm}{\textsc{ERM}}
\usepackage{todonotes}
\usepackage{pifont}
\usepackage{tablefootnote}

\usepackage{colortbl}
\definecolor{grey}{rgb}{0.8,0.8,0.8}
\definecolor{aqua}{rgb}{0, 1, 1}
\definecolor{steel}{rgb}{0.2734, 0.5078, 0.7031}
\definecolor{slate}{rgb}{0.1836, 0.3086, 0.3086}

\newcommand{\hlg}[2]{\setlength{\fboxsep}{0.3pt}\colorbox{green!#2}{\rule[-.05\baselineskip]{0pt}{.7\baselineskip}{#1}}}
\newcommand{\hlr}[2]{\setlength{\fboxsep}{0.3pt}\colorbox{red!#2}{\rule[-.05\baselineskip]{0pt}{.7\baselineskip}{#1}}}

\newcommand{\CR}[1]{{#1}}

\author{
  Guanhua Zhang\thanks{Equal contribution}\\
  UC Santa Barbara\\
  \texttt{guanhua@ucsb.edu}\\
  \And
  Yihua Zhang\footnotemark[1] \\
  Michigan State University\\
  \texttt{zhan1908@msu.edu}\\
  \And 
  Yang Zhang\\
  MIT-IBM Watson AI Lab \\
  \texttt{yang.zhang2@ibm.com}\\
  \And
  Wenqi Fan\\
  The Hong Kong Polytechnic University\\
  \texttt{wenqifan@polyu.edu.hk}\\
  \And 
  Qing Li\\
  The Hong Kong Polytechnic University\\
  \texttt{csqli@comp.polyu.edu.hk}\\
  \And 
  Sijia Liu\\
  Michigan State University \&   MIT-IBM Watson AI Lab\\
  \texttt{liusiji5@msu.edu}\\
  \And 
  Shiyu Chang\\
  UC Santa Barbara\\
  \texttt{chang87@ucsb.edu}\\
}

\begin{document}

\maketitle

\begin{abstract}

Despite a surge of recent advances in promoting machine Learning (ML) fairness, the existing mainstream approaches mostly require retraining or finetuning the entire weights of the neural network to meet the fairness criteria.  However, this is often infeasible in practice for those large-scale trained models due to large computational and storage costs, low data efficiency, and model privacy issues.  In this paper, we propose a new generic fairness learning paradigm, called \textsc{FairReprogram}, which incorporates the model reprogramming technique.  Specifically, \textsc{FairReprogram} considers the case where models can not be changed and appends to the input a set of perturbations, called the fairness trigger, which is tuned towards the fairness criteria under a min-max formulation.  We further introduce an information-theoretic framework that explains why and under what conditions fairness goals can be achieved using the fairness trigger.  We show both theoretically and empirically that the fairness trigger can effectively obscure demographic biases in the output prediction of fixed ML models by providing false demographic information that hinders the model from utilizing the correct demographic information to make the prediction.  Extensive experiments on both NLP and CV datasets demonstrate that our method can achieve better fairness improvements than retraining-based methods with far less data dependency under two widely-used fairness criteria. \CR{Codes are available at \url{https://github.com/UCSB-NLP-Chang/Fairness-Reprogramming.git}.}

\end{abstract}

\section{Introduction}

Fairness in machine learning (ML) has become a critical concern. Due to the biases in data collection, the output prediction is often spuriously correlated with some demographic attributes, which are thus undesirably incorporated into the decision-making process of machine learning models. For example, it is found that some abusive language detection systems tend to classify texts that contain mere mentioning of certain minority groups, \emph{e.g.,} homosexual groups, as abusive content, even though the texts themselves are not abusive at all~\cite{Dixon2018MeasuringAM,Park2018ReducingGB}.  Despite the recent advances in fairness promoting learning methods \cite{Zafar2017FairnessCM,Agarwal2018ARA,Kamishima2012FairnessAwareCW,Baharlouei2019RenyiFI,PrezSuay2017FairKL}, the existing mainstreaming approaches mostly require retraining or finetuning the entire model parameters towards an extra fairness objective. However, this is often infeasible in practice, particularly for those well-trained large-scale models, due to the huge computation and storage costs. In addition, for machine learning models that are deployed as a service, model retraining is hindered by limited access to the model parameters.

Recently, model reprogramming has emerged as an alternative technique to model finetuning. In particular, model reprogramming considers the pre-trained model fixed, and instead modifies their input to re-purpose the model towards different objectives. 
For example, it is shown that a well-crafted input perturbation can re-program an ImageNet classifier to solve the task of counting squares in an image~\cite{elsayed2018adversarial, tsai2020transfer}.
It is also shown that by learning task-specific embedding prompts concatenated to the inputs, pre-trained language models can achieve better performances than full-parameter tuning in natural language understanding tasks~\citep{hambardzumyan2021warp,Ding2022OpenPromptAO,Liu2021PretrainPA} 
Compared with finetuning methods, model reprogramming enjoys lower cost, better scalability, and requires less access to the model parameters. Hence here come our research questions - \emph{Can model reprogramming techniques be applied to fairness objectives? If so, why and how would it work?}

In this paper, we revisit the model reprogramming and propose a novel generic fairness learning paradigm, called \textsc{FairReprogram}. In particular, \textsc{FairReprogram} perturbs the input by appending to the input a global constant vector/feature, called the \emph{fairness trigger}, which is optimized towards the fairness objective under a min-max framework.  \textsc{FairReprogram} is a generic framework that works for various tasks and domains. We further introduce an information-theoretic framework that explains why and under what conditions fairness goals can be achieved using a constant fairness trigger.  We show theoretically and empirically that the fairness trigger can effectively obscure demographic biases in the output prediction of fixed ML models by providing false demographic information that hinders the model from utilizing the correct demographic information to make predictions.

We perform extensive experiments across various NLP and CV datasets with in-the-wild biases. The results show that \textsc{FairReprogram} can consistently achieve better fairness improvement with the retraining-based methods under the two widely-used fairness notions, but with far less trade-off in accuracy.  For example, with comparable accuracy, our method can outperform the retraining based baseline with 10.5\% and 36.5\% lower bias scores over two fairness criteria in the \texttt{CelebA} dataset with the hair color prediction task and gender as demographic information.  
In addition, our method demonstrates great transferability and interpretability. Our theoretical analysis and empirical findings can provide useful insights toward more practical, scalable, and flexible fairness learning paradigms.

\section{Related Work}

\paragraph{Fairness in ML}
Fairness problems in ML models have received increasing attention from both industry~\cite{Holstein2019ImprovingFI} and academia~\cite{Chouldechova2018TheFO,Sun2019MitigatingGB,Mehrabi2021ASO,Field2021ASO}.  
There has been a myriad of fairness definitions in the literature~\cite{Dwork2012FairnessTA,Chouldechova2018TheFO,Makhlouf2020SurveyOC,Hashimoto2018FairnessWD}.
Among them, group fairness notions are one of the most popular~\cite{Calders2010ThreeNB,Hardt2016EqualityOO,Rz2021GroupFI}, which require ML models to perform similarly for different demographic groups.   
In this paper, we mainly focus on the two most widely-used group fairness definitions, demographic parity~\cite{Calders2010ThreeNB} and equalized odds~\cite{Hardt2016EqualityOO}, but it is worth mentioning that our method is general for other fairness notions.  
Existing fairness promoting methods can be broadly categorized into pre-processing, in-processing, and post-processing methods~\cite{Caton2020FairnessIM}.  
Pre-processing methods calibrate the training data to remove the spurious correlations and train fair model on the modified data~\cite{Kamiran2011DataPT,Zemel2013LearningFR,Feldman2015CertifyingAR,Calmon2017OptimizedPF,Park2018ReducingGB,Dixon2018MeasuringAM,Grover2020FairGM,Zhang2020DemographicsSN}.  
In-processing methods work on training ML with extra fairness-aware regularization~\cite{Zafar2017FairnessCM,Agarwal2018ARA,Kamishima2012FairnessAwareCW,Baharlouei2019RenyiFI,PrezSuay2017FairKL,Roh2020FRTrainAM,Zhang2019FAHTAA}.  
For example, an adversarial framework is introduced to train model parameters to meet fairness requirements~\cite{Zhang2018MitigatingUB}.  
In our method, we adopt a similar adversarial loss but optimize the fairness triggers with a fixed model.
Despite the effectiveness, these methods usually consider training fair models from scratch and do not directly apply to already-trained models.  
Post-processing methods focus on calibrating trained ML models to be fair~\cite{Caton2020FairnessIM}.  
Many of them modify the model outputs to meet the fairness criteria~\cite{Dwork2012FairnessTA,Hardt2016EqualityOO,Petersen2021PostprocessingFI,Wei2020OptimizedST,Awasthi2020EqualizedOP,Woodworth2017LearningNP,Mishler2021FairnessIR,Zhao2017MenAL,Kim2019MultiaccuracyBP,Lohia2019BiasMP, Lohia2021PrioritybasedPB,Dwork2018DecoupledCF,Chzhen2019LeveragingLA}.  
For example,  the model outputs are directly modified to meet equalized odds by solving an optimization problem~\cite{Hardt2016EqualityOO}.  
Alternatively, a boosting-based method is introduced to calibrate model outputs \cite{Kim2019MultiaccuracyBP}.

\paragraph{Model reprogramming} 
Model reprogramming~\cite{bahng2022visual, tsai2020transfer, elsayed2018adversarial, neekhara2018adversarial, neekhara2022cross, zheng2021adversarial} aims to repurpose an already trained neural network for different tasks. 
Different from the typical transfer learning that requires modifying the structure and parameters of the given pre-trained model, reprogramming technology instead designs a trainable program appended to the input, while keeping the pre-trained model intact.  
The model reprogramming technology can be designed in the form of an input-agnostic perturbation~\cite{bahng2022visual, elsayed2018adversarial} or a trainable input transformation function together with the label mapping from the source domain to the target domain~\cite{tsai2020transfer}.
In particular, the feasibility of designing a universal input perturbation to reprogram a well-trained ImageNet classifier to the CIFAR-10 dataset is demonstrated in the white-box setting~\cite{elsayed2018adversarial}.  
As an exploration to implement reprogramming in the discrete scenario, another work~\cite{neekhara2018adversarial} successfully reprograms the text classification neural network for alternate classification tasks. 
This work also shows the possibility of developing reprogramming in the black-box setting, where the reprogrammer may not have the access to the parameters of the target model. 
Recent work~\cite{neekhara2022cross} shows the possibility of repurposing deep neural networks designed for image classifiers for the natural language processing and other sequence classification tasks.
It is argued the success of the reprogramming lies in the size of the average input gradient and the input dimension is crucial to the performance of the reprogrammer~\cite{zheng2021adversarial}. 
\CR{It is also shown that generative models like FairGANs~\cite{Xu2018FairGANFG} can be transfered to other tasks by reprogramming with variational auto-encoders~\cite{Nobile2022ReprogrammingFW}.}
A highly related topic to model reprograming is prompt learning in NLP~\cite{Ding2022OpenPromptAO}. 
It is shown that by designing designated text prompts appended to inputs, pre-trained language models could be re-directed to perform well under downstream tasks in a few-shot setting~\cite{Gao2021MakingPL}.
Prompt-based tuning methods have become the mainstream and achieve better performance than fine-tuning in many scenarios~\cite{Li2021PrefixTuningOC,Schick2021ItsNJ,Shin2020ElicitingKF,Diao2022BlackboxPL}. 
Seminal works
about prompt learning can be found in~\cite{Ding2022OpenPromptAO,Liu2021PretrainPA}.  
\textit{However},
nearly all existing methods focus on using model reprogramming to improve accuracy in domain-transfer tasks and to our best knowledge, our work is the first to generalize model reprogramming to improve fairness of a trained model.

\section{Fairness Reprogramming}

In this section, we will introduce the {\ours} algorithms.
As some notations, upper-cased letters, \e{\bm X} and \e{X}, denote random vectors and random variables, respectively; lower-cased letters, \e{\bm x} or \e{x}, denote deterministic vectors and scalars respectively. \e{p_{\bm X}(\cdot)} or  \e{p(\bm X)} denote the probability density function of the (discrete) random variable \e{\bm X}.

\subsection{Problem Formulation}

Consider a classification task, where \e{\bm X} represents the input feature, and \e{Y} represents the output label. In addition, there exists some sensitive attributes or demographic group, \e{Z}, that may be spuriously correlated with \e{Y}. There is a pre-trained classifier, \e{f^*(\cdot)}, that predicts \e{Y} from \e{\bm X}, \emph{i.e.} \e{\hat{Y} = f^*(\bm X)}. The weights of the classifier are considered fixed (hence the superscript \e{*}). Unfortunately, due to the spurious correlation between \e{Z} and \e{Y}, the classifier may be biased against certain demographics.

Our goal is to improve the fairness of the classifier by modifying the input \e{\bm X}, rather than modifying the classifier's fixed weights. In particular, we aim to achieve either of the following fairness criteria.
\begin{equation}
    \mbox{\textbf{Equalized Odds:}} \quad \hat{Y} \perp Z | Y, \quad \mbox{or} \quad \mbox{\textbf{Demographic Parity:}} \quad \hat{Y} \perp Z,
    \label{eq:fair_criteria}
\end{equation}
where \e{\perp} denotes independence. The following two subsection will explain how to modify input and design the optimization objective respectively.

\subsection{Modifying the Input Features}

Input modification primarily involves appending a \emph{fairness trigger} to the input. Formally, the input modification takes the following generic form:
\begin{equation}
    \tilde{\bm X} = m(\bm X; \bm \theta, \bm \delta) = [\bm \delta, g(\bm X; \bm \theta)],
    \label{eq:input_mod}
\end{equation}
where \e{\tilde{\bm X}} denotes the modified input; \e{[\cdot]} denotes vector concatenation. As can be observed, the input modification consists of two steps. First, \e{\bm X} is fed through a transformation function \e{g(\cdot; \bm \theta)}, where \e{\bm \theta} represents the hyper-parameters of the transformation function. The actual form of \e{g(\cdot; \bm \theta)} is contingent upon different applications and modalities, but a general requirement is that \e{g(\cdot; \bm \theta)} should largely retain the information necessary for classification. The second step is to append a fairness trigger, \e{\bm \delta}, to the input, which is a vector that can be optimized over. It is important to note that \e{\bm \delta} is a \emph{constant} -- different inputs get appended the same trigger. Although it does not seem intuitive, we will soon show that a constant trigger is all you need to achieve fair prediction on all different inputs.

Below are specific forms of transformations (Eq.~\eqref{eq:input_mod}) we use.

\textbf{Text Classification} \quad In text classification, \e{\bm X} represents a sequence of input token embeddings. To modify the input, we simply append a fixed number of embeddings after the input text. In this case, \e{g(\cdot; \bm \theta)} is the identity mapping, and \e{\bm \delta} corresponds to the appended embeddings.

\begin{wrapfigure}{r}{0.35\textwidth}
\centerline{
\begin{tabular}{cc}
\hspace*{-2mm}
\includegraphics[width=.15\textwidth,height=!]{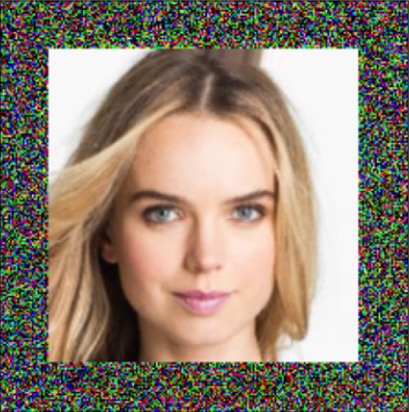} 
&\hspace*{-2mm}
\includegraphics[width=.15\textwidth,height=!]{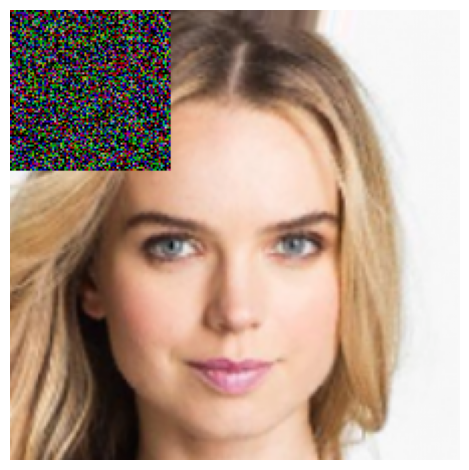} 
\\
\footnotesize{(a) Border trigger}
& 
\hspace*{-2mm}\footnotesize{(b) Patch trigger} 
\end{tabular}}
\caption{\footnotesize{Demonstration of 
the border and patch trigger applied on an image from \texttt{CelebA}\,\cite{liu2015faceattributes}. }}
\vspace*{-3mm}
\label{fig: trigger}
\end{wrapfigure}

\textbf{Image Classification} \quad In image classification, \e{\bm X} represents the (vectorized) input image. Unlike text classification, where the input can have a variable length, the length of the input to the image classification network is fixed. We thus apply the following two approaches to append the trigger, as shown in Fig.~\ref{fig: trigger}. The first approach, called the \emph{patch approach}, removes a patch from the original image, and appends a trigger the same size as the patch to the patch location (as shown in Fig.~\ref{fig: trigger}(a)). In this case, \e{g(\cdot; \bm \theta)} is a function that removes the patch dimension and retain the rest, with \e{\bm \theta} representing the patch location; \e{\bm \delta} represents the trigger feature that replaces the patch. 
The second approach, called the \emph{border approach}, shrinks the image to a smaller image, and then appends the trigger at the border (as shown in Fig.~\ref{fig: trigger}(b)). In this case, \e{g(\cdot; \bm \theta)} is a function that shrinks the image, and \e{\bm \delta} represents the trigger feature at the border. 

\subsection{Optimization Objective}

Our optimization objective is as follows
\begin{equation}
    \min_{\bm \delta, \bm \theta} \,\,\, \mathcal{L}_{util} (\mathcal{D}_{tune}, f^* \circ m) + \lambda \mathcal{L}_{fair} (\mathcal{D}_{tune}, f^* \circ m),
    \label{eq:loss}
\end{equation}
where \e{m = m(\cdot; \bm \theta, \bm \delta)} represents the input modification function as in Eq.~\eqref{eq:input_mod}; \e{\circ} represents nested functions; \e{\mathcal{D}_{tune}} represents the dataset that are used to train the fairness trigger. Note that this is different from the dataset where the classifier, \e{f^*}, is pre-trained.

The first loss term, \e{\mathcal{L}_{util}}, is the utility loss function of the task. For classification tasks, \e{\mathcal{L}_{util}} is usually the cross-entropy loss, \emph{i.e.},
\begin{equation}
    \mathcal{L}_{util}(\mathcal{D}_{tune}, f^* \circ m) = \mathbb{E}_{\bm X, Y \sim \mathcal{D}_{tune}} [\textrm{CE}(Y, f^*(m(\bm X)))],
    \label{eq:loss_util}
\end{equation}
where {\small CE}\e{(\cdot, \cdot)} denotes the cross-entropy loss.

The second loss term, \e{\mathcal{L}_{fair}}, encourages the prediction to follow the fairness criteria as in Eq.~\eqref{eq:fair_criteria}. According to Eq.~\eqref{eq:fair_criteria}, \e{\mathcal{L}_{fair}} should measure how much information about \e{Z} is in \e{\hat{Y}}. To measure this, we introduce another network, called the discriminator, \e{d(\cdot; \bm \phi)}, where \e{\bm \phi} represents its parameters. If the equalized odds criterion is applied,  then \e{d(\cdot; \bm \phi)} should predict \e{Z} from \e{\hat{Y}} and \e{Y}; if the demographic parity criterion is applied, then the input to \e{d(\cdot; \bm \phi)} would just be \e{\hat{Y}}. In the following, we will focus on equalize odds criterion for conciseness. Then, the information of \e{Z} can be measured by maximizing the \emph{negative} cross-entropy loss for the prediction of \e{Z} over the discriminator parameters, \emph{i.e.},
\begin{equation}
    \mathcal{L}_{fair} (\mathcal{D}_{tune}, f^* \circ m) = \max_{\bm \phi} \mathbb{E}_{\bm X, Y, Z \sim \mathcal{D}_{tune}} [-\textrm{CE}(Z, d(f^*(m(\bm X)), Y; \bm \phi))].
    \label{eq:loss_fair}
\end{equation}
By plugging Eqs.~\eqref{eq:loss_util} and \eqref{eq:loss_fair} into \eqref{eq:loss}, we can see that the entire optimization objective becomes a min-max framework, where the discriminator tries to improve its prediction of \e{Z} while the fairness trigger tries to make the prediction worse. As shown in \cite{Zhang2018MitigatingUB}, when the discriminator cannot predict \e{Z} better than chance, the aforementioned fairness criteria can be achieved.

\subsection{Why Does It Work?}
\label{sec:perspective}

It is not immediately straightforward why a \emph{global} trigger can obscure the demographic information for \emph{any} input. In this section, we will propose an information-theoretic framework that illustrates one of the mechanisms through which the trigger can remove the demographic information.

\begin{figure}
    \centering
    \includegraphics[width=\textwidth]{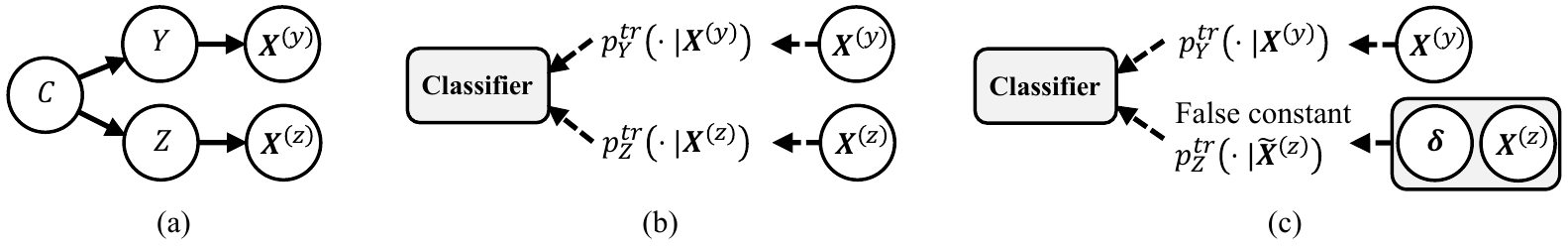}
    \caption{Illustration of why fairness trigger works. (a) The data generation process. (b) The information flow from data to the classifier through the sufficient statistics. (c) Fairness trigger strongly indicative of a demographic group can confuse the classifier with a false demographic posterior, and thus preventing the classifier from using the correct demographic information.}
    \label{fig:theory}
\end{figure}

Our theoretical framework builds upon the data generation process as shown in Fig.~\ref{fig:theory}(a). Specifically, we assume that \e{\bm X} consists of a set of features, \emph{i.e.} \e{\bm X = [\bm X_1, \cdots, \bm X_T]}, where \e{T} is the total number of features. In text classification, a feature can be a word or a word piece; in image classification, a feature can be specific shapes, colors, patterns, \emph{etc}. Assume that these features can be divided into two groups. The first group, denoted as \e{\bm X^{(y)}}, consists of features that are directly governed by the output label \e{Y}; the second group, denoted as \e{\bm X^{(z)}}, consists of featuers that are directly governed by the demographic information \e{Z}. \e{Z} and \e{Y} can be spuriously correlated, \emph{i.e.} there can be common confounders, \e{C}, between \e{Z} and \e{Y}. As a result, both \e{\bm X^{(y)}} and \e{\bm X^{(z)}} are indicative of \e{Y}. 

To further simplify our theoretical analysis, we consider a bag-of-feature scenario, where each feature in \e{\bm X^{(y)}} is drawn from the vocabulary set \e{\mathcal{X}^{(y)}}, and each feature in \e{\bm X^{(z)}} is drawn from the vocabulary set \e{\mathcal{X}^{(z)}}. There should not be any overlap between the two vocabulary sets, \emph{i.e.} \e{\mathcal{X}^{(y)} \cap \mathcal{X}^{(z)} = \varnothing}. Otherwise it violates our assumption that demographic-related features are biased features.

It can be shown (in Appendix~\ref{app: proof}) that the posterior distributions, \e{p_Y(\cdot | \bm X^{(y)})} and \e{p_Z(\cdot | \bm X^{(z)})}, are the sufficient statistics of \e{\bm X^{(y)}} and \e{\bm X^{(z)}} respectively for inferring \e{Y}. In other words, these two posterior distributions summarize all the information about \e{\bm X^{(y)}} and \e{\bm X^{(z)}} that the classifier needs to know to predict \e{Y}. Therefore, we assume that the classifier takes the following generic form
\begin{equation}
    \hat{Y} = f^*(\bm X) = h(p^{tr}_Y(\cdot | \bm X^{(y)}), p^{tr}_Z(\cdot | \bm X^{(z)})).
    \label{eq:classifier_assump}
\end{equation}
Note that we add a superscript, \e{tr}, to emphasize that the probability distributions are over the data set where the classifier is trained, because the classifier has never been trained on inputs modified with the fairness trigger. Eq.~\eqref{eq:classifier_assump} encompasses many common decision functions. For example, it can be shown (in Appendix~\ref{app: proof}) that the posterior distribution \e{p(Y | \bm X)}, which is the minimizer of the cross-entropy loss, is a special case of Eq.~\eqref{eq:classifier_assump}.

As illustrated in Fig.~\ref{fig:theory}(b), \e{p_Y(\cdot | \bm X^{(y)})} and \e{p_Z(\cdot | \bm X^{(z)})} provide two sets of information from input features. \e{p_Y(\cdot | \bm X^{(y)})} provides the \emph{unbiased} information, because a desirable fair classifier should rely only upon \e{p_Y(\cdot | \bm X^{(y)})} to make a decision. On the other hand, \e{p_Z(\cdot | \bm X^{(z)})} provides the \emph{biased} information, because it conveys the demographic information. In other words, the fairness goals can be achieved by cutting off the biased information path. Therefore, our research question boils down to: is it possible to cut off the biased information path with a global fairness trigger \e{\bm \delta}?

Without loss of generality, assume that \e{\bm \delta} consists of only one feature. Consider the case where \e{\bm \delta} is a demographic feature, \emph{i.e.} \e{\bm\delta \in \mathcal{X}^{(z)}}. In this case, we assume the transformed input as defined in Eq.~\eqref{eq:input_mod} can also be divided into two groups:
\begin{equation}
    \tilde{\bm X} = [\tilde{\bm X}^{(y)}, \tilde{\bm X}^{(z)}], \quad \mbox{where} \quad \tilde{\bm X}^{(y)} = g(\bm X^{(y)}), \quad \tilde{\bm X}^{(z)} = [\bm \delta,  g(\bm X^{(z)})].
    \label{eq:modify_assump}
\end{equation}

The following theorem states our main conclusion:
\begin{theorem}
Under the assumptions in Eq.~\eqref{eq:classifier_assump} and \eqref{eq:modify_assump}, and some additional regularity conditions\footnote{Formal assumptions stated in the appendix.}, if the fairness trigger \e{\bm \delta} is indicative of a certain demographic group \e{z}, then
\begin{equation}
    \lim_{p^{tr}(Z=z|\bm X_0^{(z)}=\bm \delta) \rightarrow 1} \textrm{MI}(\hat{\tilde{Y}}, Z | Y) = 0,
\end{equation}
where \e{\textrm{MI}} means mutual information; \e{\hat{\tilde{Y}} = f^*(\tilde{\bm X})} is the classifier's prediction after input is modified.
\label{thm:main}
\end{theorem}
\e{p^{tr}(Z=z|\bm X_0^{(z)}=\bm \delta) \rightarrow 1} means that the fairness trigger is very strongly indicative of the demographic group \e{z}. Therefore, Thm.~\ref{thm:main} essentially states that if the prepended trigger feature is very strongly indicative of a certain demographic group, then equalized odds can be achieved. A formal proof is presented in Appendix~\ref{app: proof}. Here we would like to give an intuitive explanation. When \e{p^{tr}(Z=z|\bm X_0^{(z)}=\bm \delta) \rightarrow 1}, it will also happen that \e{p^{tr}(Z=z|\bm X^{(z)} = \tilde{\bm X}^{(z)}) \rightarrow 1}. In other words, the fairness trigger \e{\bm \delta} would overshadow the rest of the demographic features and `trick' the classifier into believing all the different inputs belong to the same demographic group \e{z}. As a result, the second argument in Eq~\eqref{eq:classifier_assump} would reduce to a constant (1 for demographic group \e{z} and 0 elsewhere), effectively blocking the biased information path, as shown in Fig.~\ref{fig:theory}(c). Note that the premise for the fairness trigger to work is that the classifier has never seen the modified input. Otherwise, the classifier will be able to learn to ignore the constant trigger and still elicit the true demographic information from input.

\section{Experiments}
In this section, we evaluate the effectiveness of {\ours} on both NLP and CV applications in terms of accuracy, fairness, performances under low-data regime, transferability and interpretability.

\subsection{Experiment Setup}
\label{sec: setup}
\paragraph{Datasets}We consider the following two commonly used  NLP and CV datasets:
\begin{itemize}[leftmargin=9pt,topsep=0pt]
    \vspace{-1mm}
    \item
    \texttt{\textbf{Civil Comments}}~\cite{Koh2021WILDSAB,AI2019JigsawUB}:
    The dataset contains 448k texts with labels that depict the toxicity of each input. The demographic information of each text is provided.  
    \vspace{-1mm}
    
    \item 
    \textbf{\texttt{CelebA}}~\cite{liu2015faceattributes}: 
    The dataset contains over 200k human face images and each contains 39 binary attribute annotations.  We follow the conventional setting \cite{liu2015faceattributes} that adopts the hair color prediction task in our experiment and uses gender annotation as the demographic information.\,\cite{xu2020investigating, dash2020counterfactual, hwang2020fairfacegan}   
    \vspace{-1mm}
\end{itemize}

For both datasets, we split the entire data into a training set, a tuning set, a validation set, and a testing set.  The training set is used for the base model training, \emph{i.e.}, to obtain a biased model for reprogramming.  The tunning set and validation set are used for trigger training and hyper-parameter selection.  We report our results on the testing set.   It is worth mentioning that there is no overlapping data between different sets and the size of the tuning set is much smaller than the training one.  Specifically, we set the size ratio between the tunning set and the training as $1/5$ and $1/100$ for \texttt{Civil Comments} and \texttt{CelebA}, respectively.  
The full statistics of the datasets can be found in Appendix~\ref{app: data_statistic}.

\paragraph{Metrics} Besides the model accuracy, we introduce two empirical fairness metrics, one under each of the two fairness criteria as in Eq.~\eqref{eq:fair_criteria}. For binary classification, the metrics are calculated as:
\begin{equation*}
    \mbox{\textbf{DP:}}\sum_{z\in \mathcal{Z}} |\text{p}(\hat{Y}=1)-\text{p}(\hat{Y}=1|Z=z)|, \quad \mbox{\textbf{EO:}} \sum_{z\in \mathcal{Z}} (|\text{FPR}-\text{FPR}_z| + |\text{FNR} - \text{FNR}_z|)/2,
\end{equation*}
where DP and EO stand for demographic parity and equalized odds respectively. $\text{FPR}$ and $\text{FNR}$ are the false positive/negative rate, and the subscript $z$ denotes the score is calculated within a specific demographic group $Z=z$.  For example, $\text{FPR}_{male}$ indicates the false positive rate calculated over all examples with the ``male'' annotation. For a multi-class setting, the bias scores are first calculated similarly using one-vs-all for each class and then averaged across different classes.  
All reported results are the average of three different random runs.  
It can be shown that these metrics are non-negative, and will become zero when their corresponding fairness criteria are achieved.
For better elaboration, we report the negative bias scores in our experiments,
so the larger these negative scores are, 
the better the model satisfies the corresponding fairness criteria.

\paragraph{Baselines and implementation details}
We consider the following models for comparison:

$\bullet$ \textsc{Base}: the base model to be reprogrammed, trained with the cross-entropy loss on the training set.  

$\bullet$ {\adv}~\cite{Zhang2018MitigatingUB}: an in-processing adversarial training method that optimizes both model accuracy and fairness using the training set.  

$\bullet$ {\advp}: a post-processing variant of {\adv}, which fine-tunes the \textsc{Base} model with the same fairness-aware adversarial objectives as {\adv}, but using the (low-resource) tunning set only.  

For NLP experiments, we use a pre-trained \textsc{Bert} \cite{Devlin2019BERTPO} to obtain the \textsc{Base} and {\adv} models.  We use \textsc{AdamW} \cite{loshchilov2017decoupled} as the optimizer, and set the learning rate to $10^{-5}$ for all baselines and $0.1$ for {\ours}.  For CV experiments, we consider a \textsc{ResNet-18} \cite{he2016deep} that pre-trained on ImageNet.  
The discriminator used in {\adv}, {\advp} and {\ours} is a three-layer MLP, and the parameters are optimized using \textsc{Adam} with a learning rate of $0.01$.  We pick the best model based on the accuracy (for the \textsc{Base}) or the bias scores (for all other debiasing methods) of the validation set.  \CR{We refer to Appendix~\ref{app: training_detail} for more details and Appendix\,\ref{app: additional_results} for more baseline studies.}

Next we introduce the implementation details of the triggers for different variants of {\ours}. For image classification task, we adopt the border and patch trigger as shown in Fig.\,\ref{fig: trigger}, termed \textsc{FairReprogram (Border)} and \textsc{FairReprogram (Patch)} correspondingly. We define the trigger size as the width of the trigger frame for border trigger and the width of the square patch for patch trigger. Unless otherwise stated, the default trigger size for each setting are $20$ and $80$. 

For text classification task, we introduce a probability vector $\bm v_i 
$ to control the selection of trigger word for each position $i$. 
Specifically, we have the trigger $\bm \delta_i = \bm E \bm v_i$ where $\bm E 
$ represents the pretrained word embedding matrix of \textsc{Bert}.
Then we simply concatenate $\bm \delta$ after all input texts\footnote{The trigger is appended as a suffix after all input tokens but before [SEP] for \textsc{Bert}.} in the embeddings space as the fairness trigger. We introduce two types of trigger. The first type, called \textsc{FairReprogram (Soft)} , uses continuous \e{\bm v_i}'s, and each  \e{\bm v_i} is projected onto the continuous probability simplex using the bisection algorithm after each training step. The second type, called \textsc{FairReprogram (Hard)}, discretizes each $\bm v_i$ into a one-hot vector $\hat{\bm v}_i$ via $\mathop{\arg\max}$ operation.  
We adopt the straight through technique~\cite{Bengio2013EstimatingOP} to update $\bm v_i$ during training.
The triggers found by \textsc{FairReprogram (Hard)} enjoy better interpretability as they correspond to a sequence of word tokens.
Unless specified otherwise, we set the trigger word number as five for our experiments.

\subsection{Results}

\begin{figure}[!t]
\centerline{
\begin{tabular}{cccc}
    \hspace*{-2mm} \includegraphics[width=.25\textwidth,height=!]{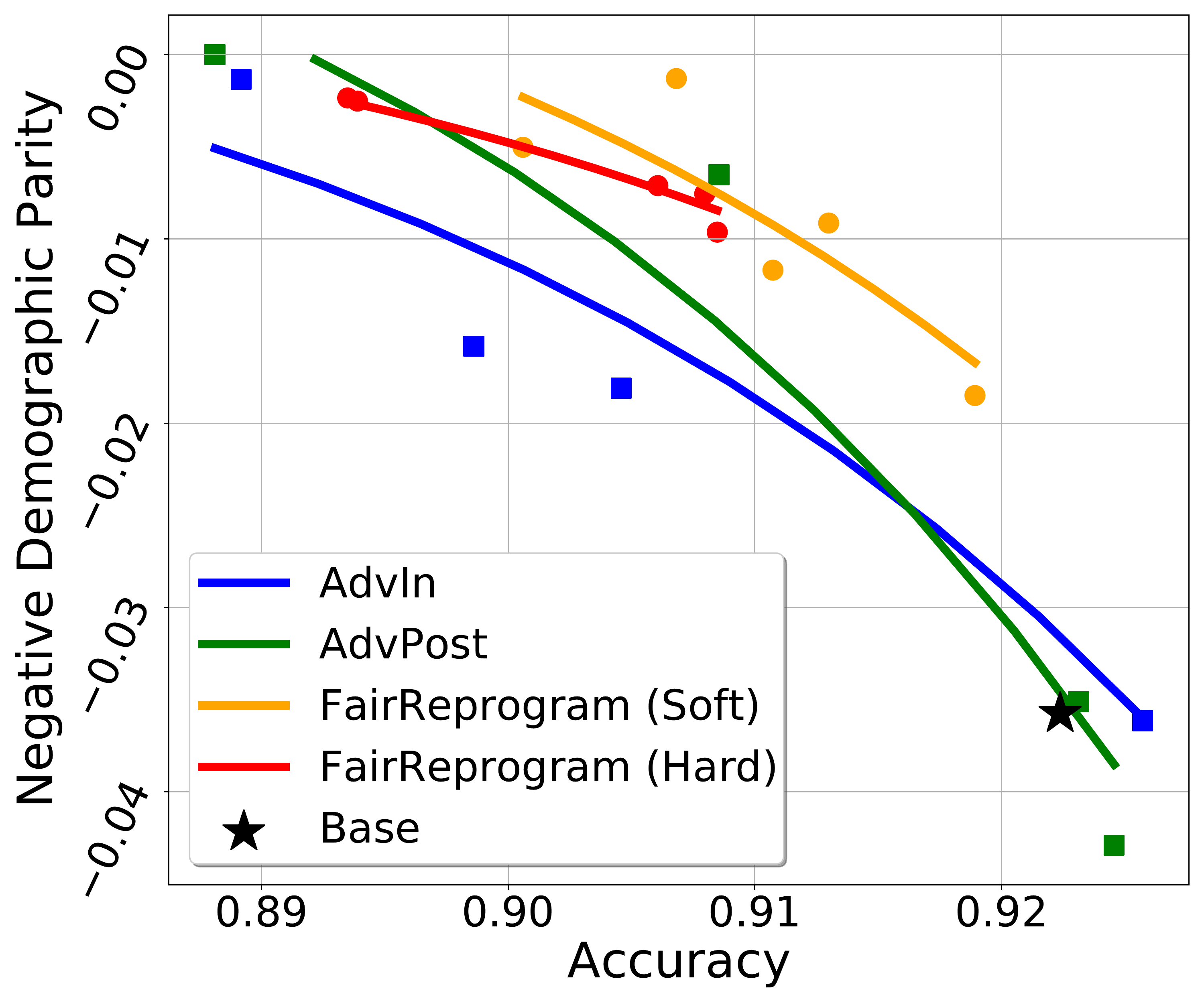} &
    \hspace*{-4mm}  \includegraphics[width=.25\textwidth,height=!]{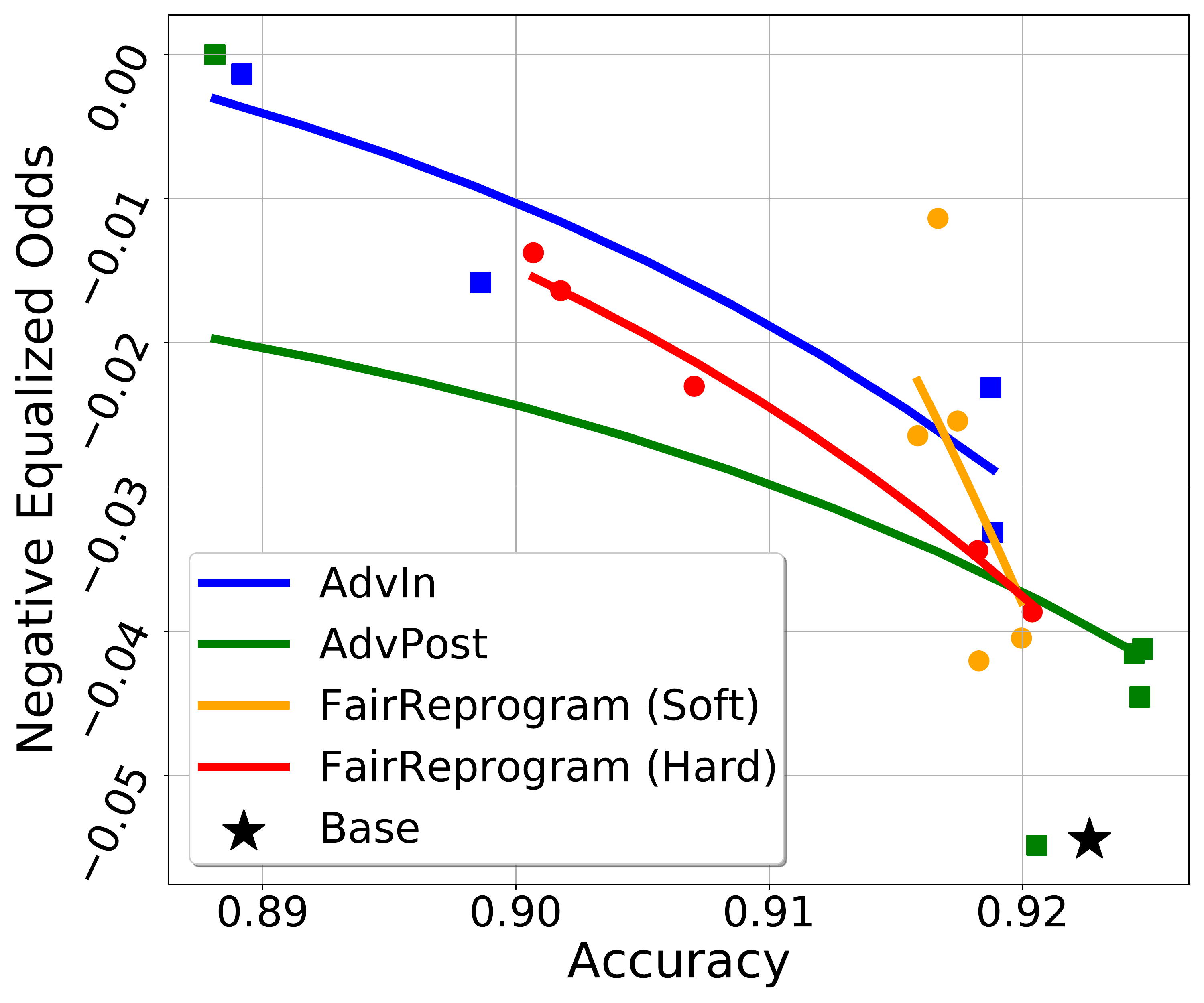} &
    \hspace*{-4mm}  \includegraphics[width=.25\textwidth,height=!]{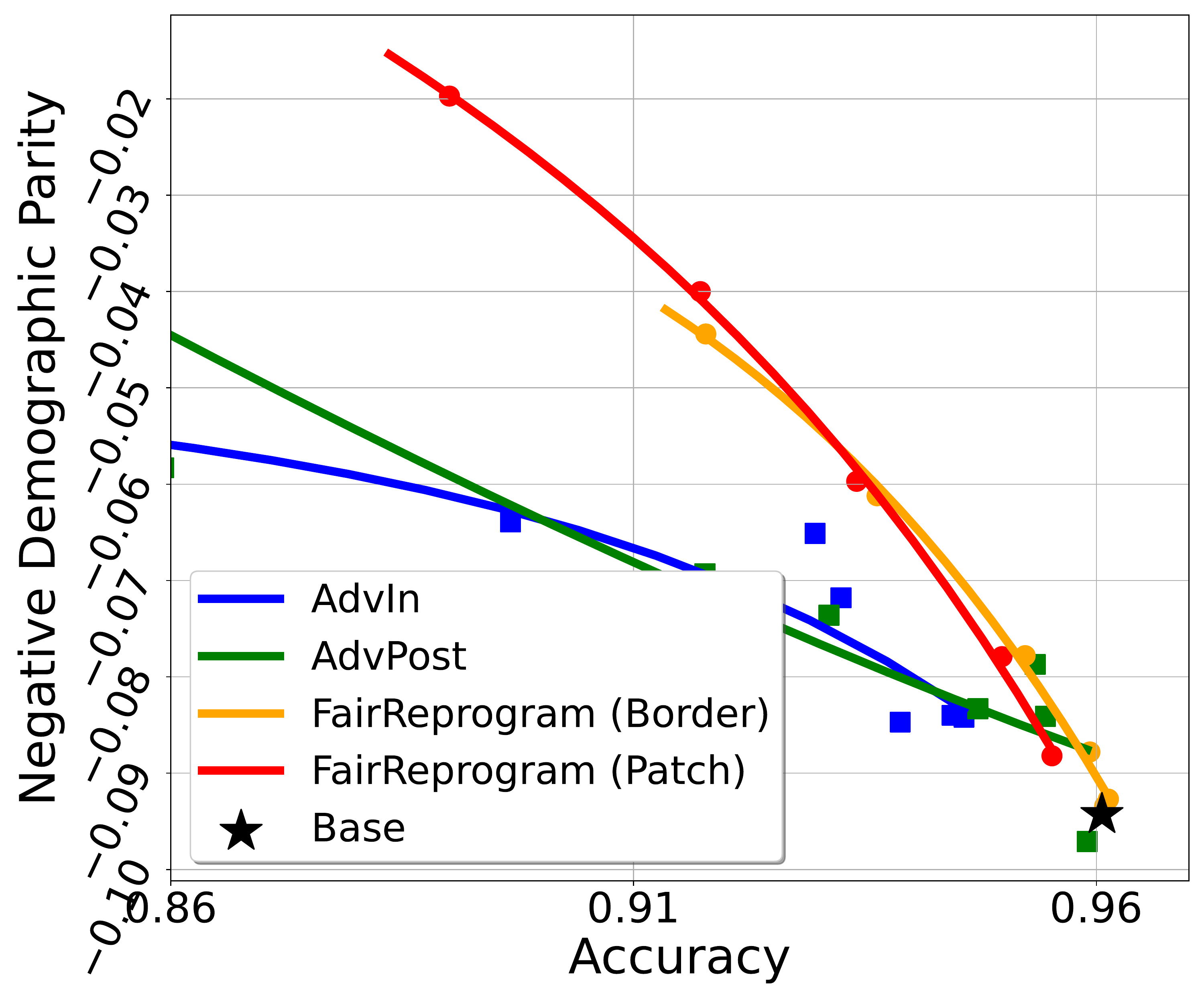} &
    \hspace*{-4mm} \includegraphics[width=.25\textwidth,height=!]{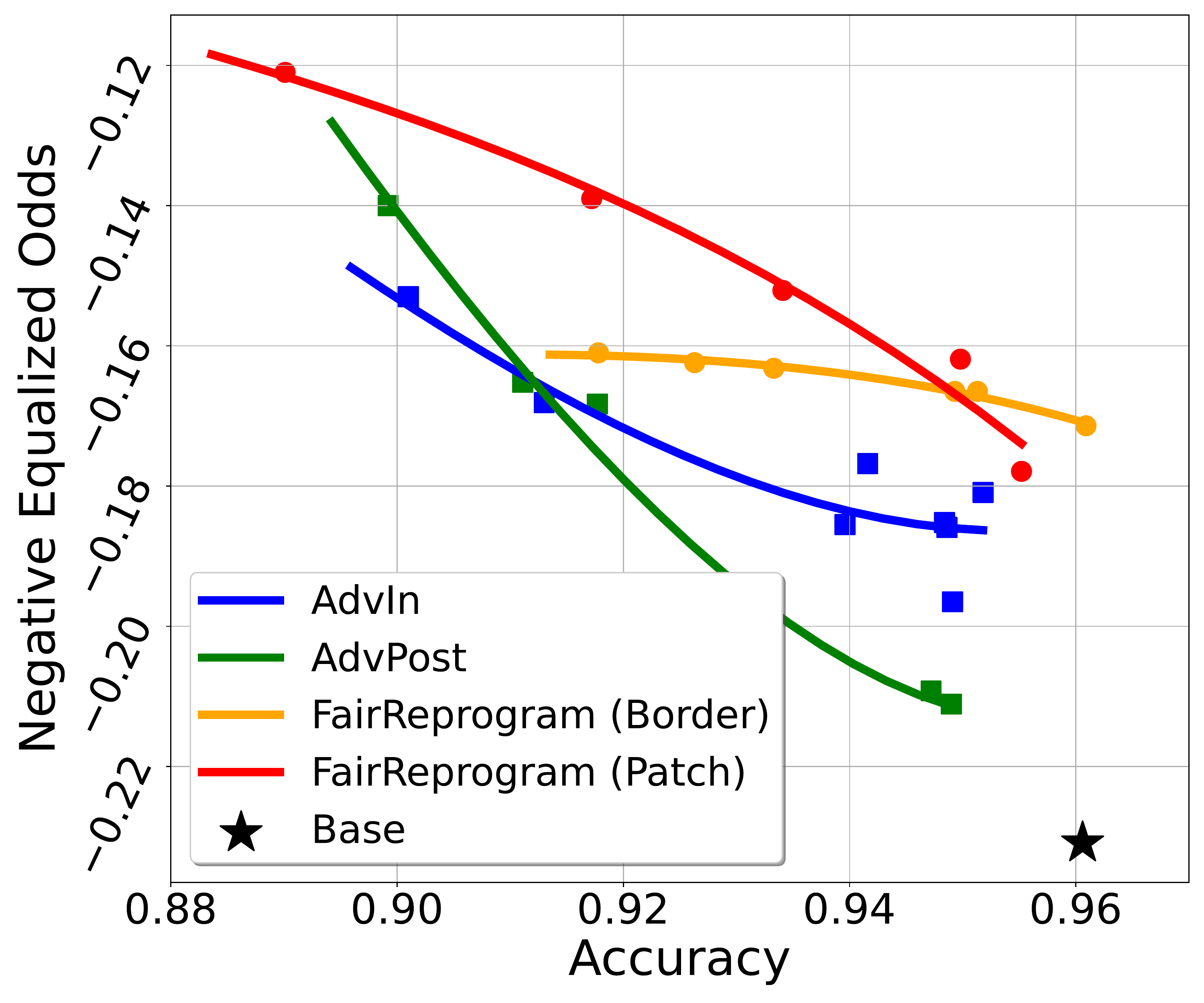} \\
    \multicolumn{2}{c}{\footnotesize{(a) \texttt{Civil Comments}}} & \multicolumn{2}{c}{\footnotesize{(b) \texttt{CelebA}}}
\end{tabular}}
\caption{\footnotesize{Results on (a) \texttt{Civil Comments} and (b) \texttt{CelebA}. We report the negative DP (left) and the negative EO (right) scores.  For each method, we vary the trade-off parameter $\lambda$ (as shown in \eqref{eq:loss}) to record the performance.  The closer a dot to the upper-right corner, the better the model is.   We consider five different $\lambda$s for each method. 
The solid curve is the fitted polynomial with order 30. 
}}
\label{fig: exp_overview}
\end{figure}

\begin{figure}[!t]
\centerline{
\begin{tabular}{cccc}
    \hspace*{-2mm} \includegraphics[width=.25\textwidth,height=!]{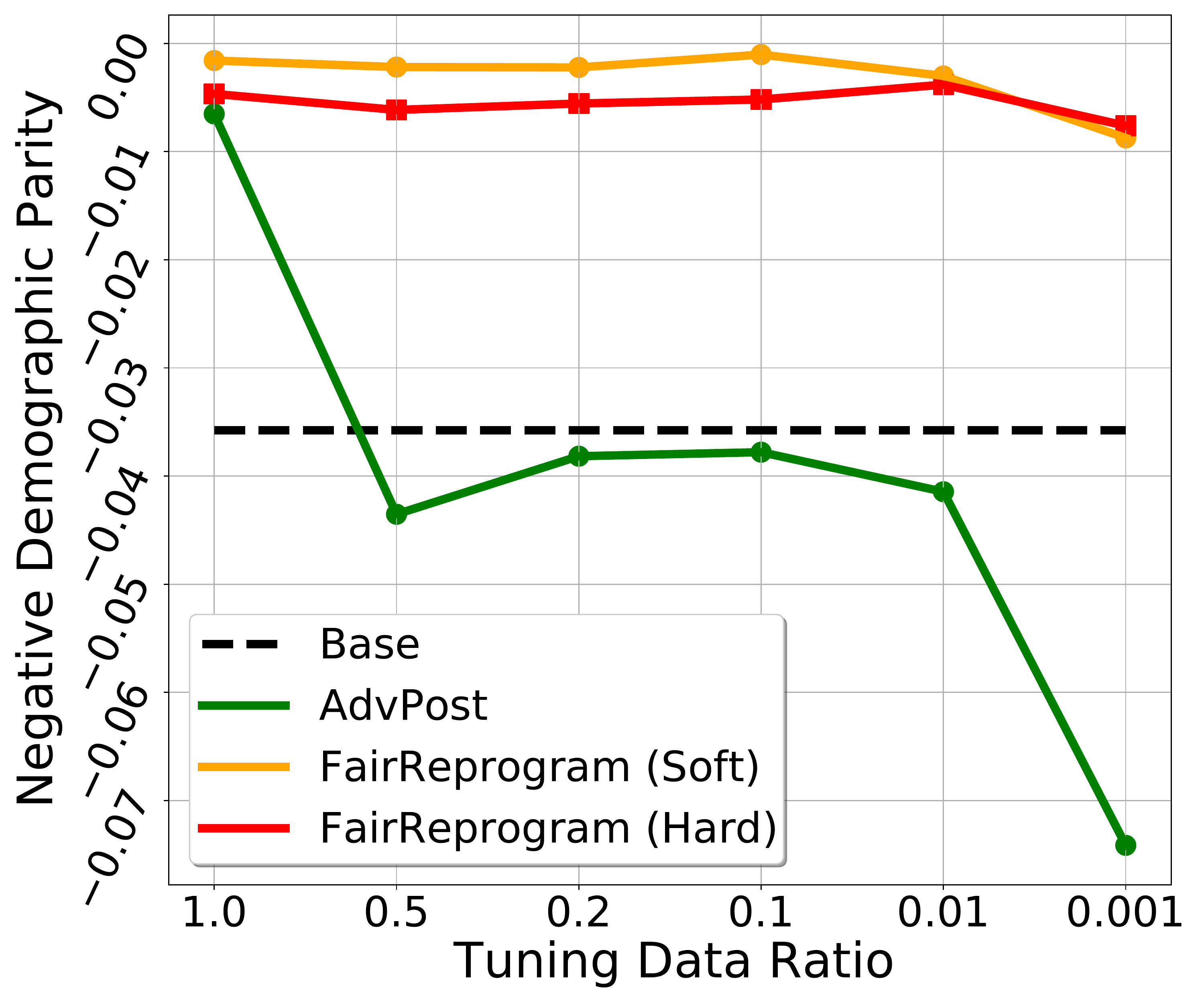} &
    \hspace*{-4mm}  \includegraphics[width=.25\textwidth,height=!]{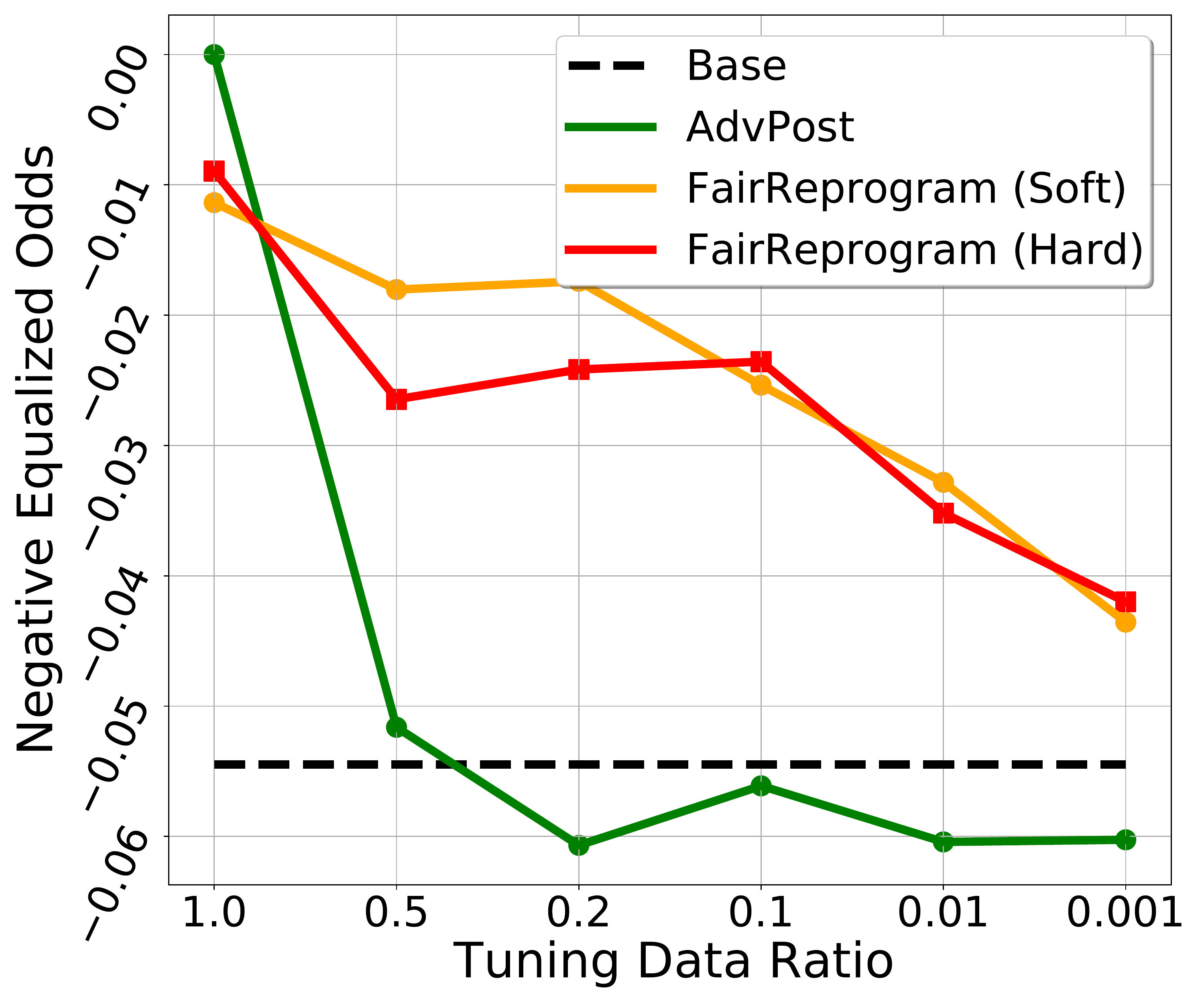} &
    \hspace*{-4mm}  \includegraphics[width=.25\textwidth,height=!]{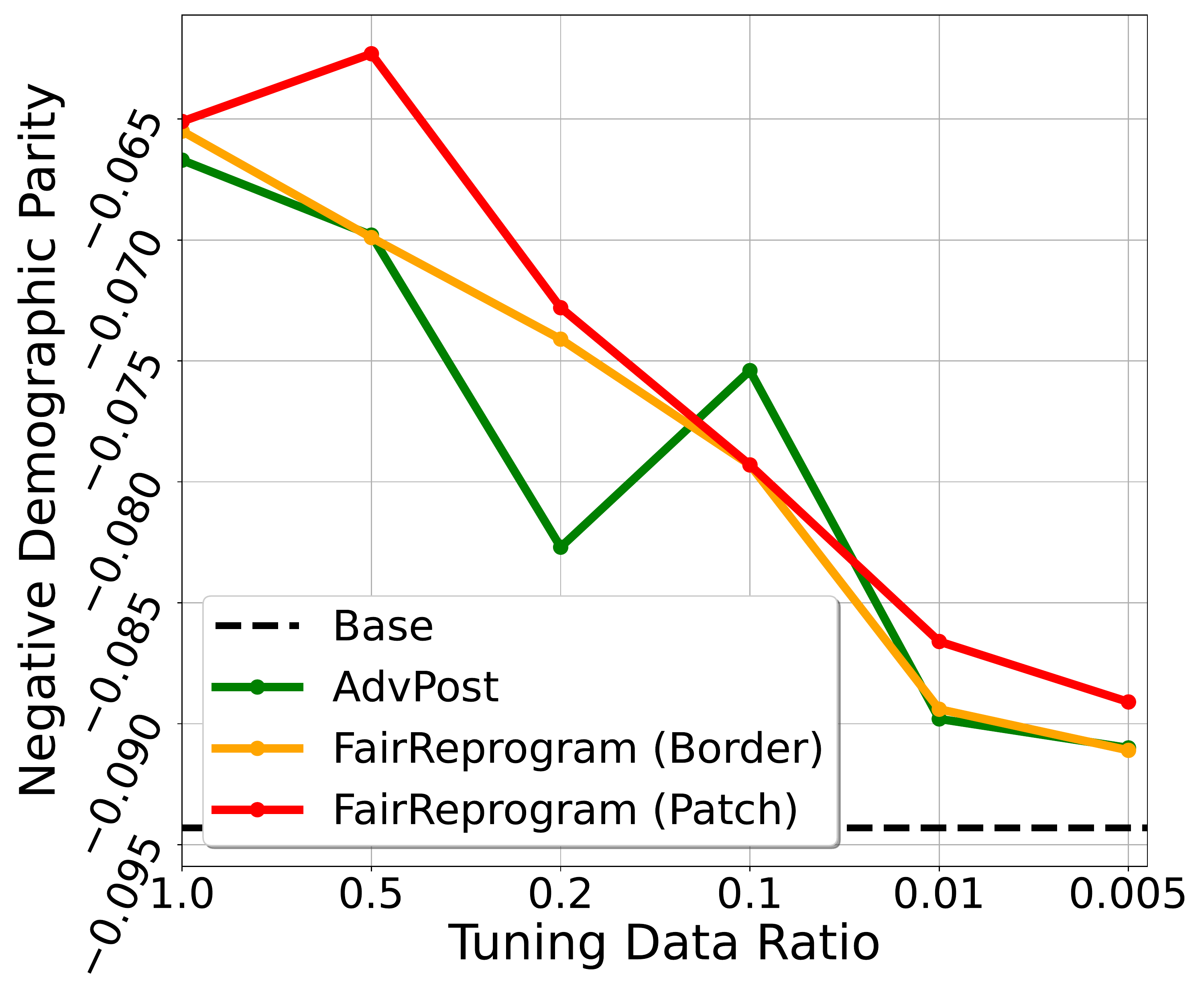} &
    \hspace*{-4mm} \includegraphics[width=.25\textwidth,height=!]{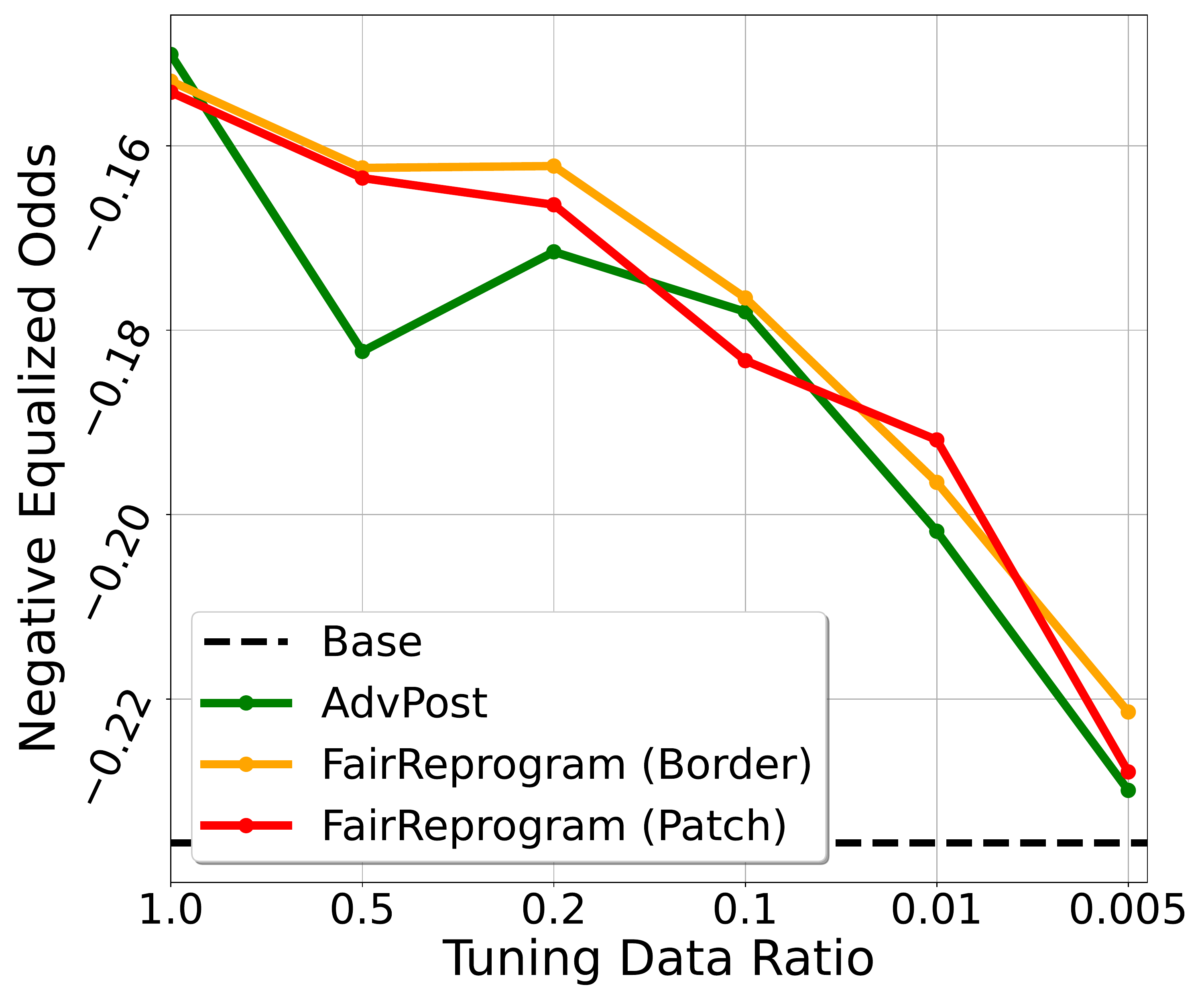} \\
    \multicolumn{2}{c}{\footnotesize{(a) \texttt{Civil Comments}}} & \multicolumn{2}{c}{\footnotesize{(b) \texttt{CelebA}}}
\end{tabular}}
\caption{\footnotesize{
Results on (a) \texttt{Civil Comments} and (b) \texttt{CelebA} with different tuning data ratio. We report the negative DP (left) and negative EO (right) scores.
We consider a fixed \textsc{Base} model trained with training set, whose negative bias scores are presented as a black dashed line.
Then we train other methods with different tuning data ratio to promote fairness of the \textsc{Base} model.
}}
\label{fig: limit_data}
\end{figure}
Fig.~\ref{fig: exp_overview} shows the performance of the proposed \textsc{FairReprogram} with other baselines on both NLP (subfigure (a)) and CV (subfigure (b)) datasets using DP (left) and EO (right) metrics. In each subfigure, the data samples of the same method (dots in the same color) are generated by explicit changing the adversary weight $\lambda$ in \eqref{eq:loss}, which controls the trade-off between fairness and accuracy. We further fit the data with polynomial regression to present the curves. Appendix\,\ref{app: training_detail} shows the detailed $\lambda$ choices for different methods.
Here are our key observations.  \underline{{First}}, our method improves the fairness of the \textsc{Base} model.  In particular, our methods (both \textcolor{orange}{orange} and \textcolor{red}{red} curves) achieve higher negative DP and EO scores with a comparable classification accuracy.    \underline{Second}, our method enjoys a better fairness-accuracy trade-off compared with all other baselines.  Specifically, the curves of our method lie farther to the upper-right corner of the plots, which implies that our method improves model fairness with fewer sacrifices on accuracy.   It is also worth noting that although {\adv} achieves good fairness scores, it uses much more data for training.

\paragraph{Limited data setting} 
We further evaluate {\advp} and {\ours} with decreasing the number of data in the tuning set.
Specifically, we fix a $\lambda$ for each method such that all methods  achieve comparable bias score with full tuning set.
The detailed $\lambda$ choices are provided in Appendix~\ref{app: training_detail}.
Then we apply these methods to subsets of the tuning set with different proportions.
The results are shown in Fig.~\ref{fig: limit_data}. There are two key observations.
\underline{First}, our method can consistently improve fairness upon \textsc{Base} model even with 1\% tuning data, indicating a high data efficiency of {\ours}.
\underline{Second}, {\ours} achieves better fairness than {\advp} does when tuning data number decreases.
For example, in Fig.~\ref{fig: limit_data} (a), the curve of our method is significantly above the {\advp} as tuning data decreases.
When the tuning set size is extremely small, {\advp} significantly deteriorates and even underperforms the \textsc{Base} model.

\paragraph{Transferability} 
\begin{figure}[!t]
\centerline{
\begin{tabular}{cccc}
    \hspace*{-2mm} \includegraphics[width=.25\textwidth,height=!]{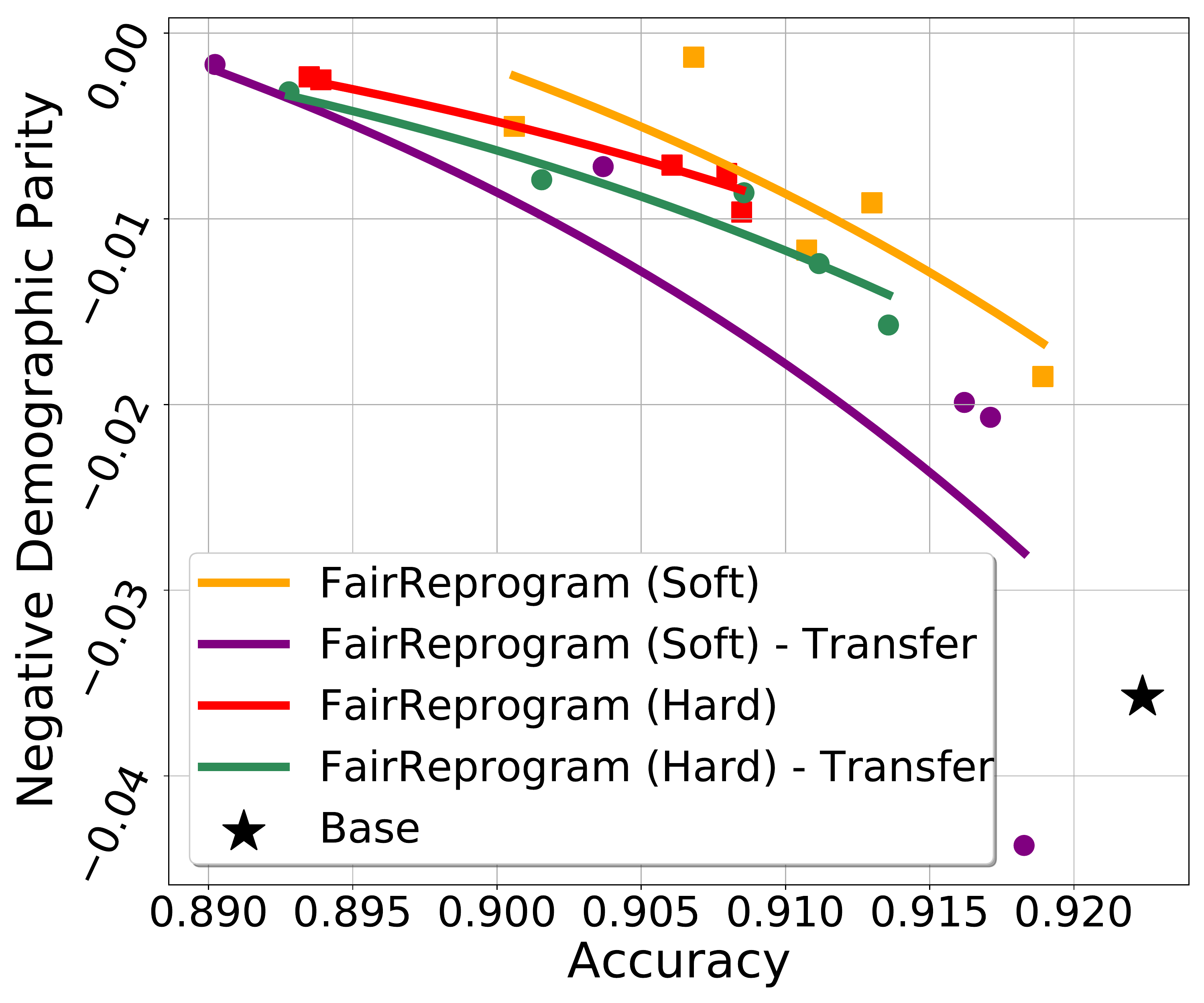} &
    \hspace*{-4mm}  \includegraphics[width=.25\textwidth,height=!]{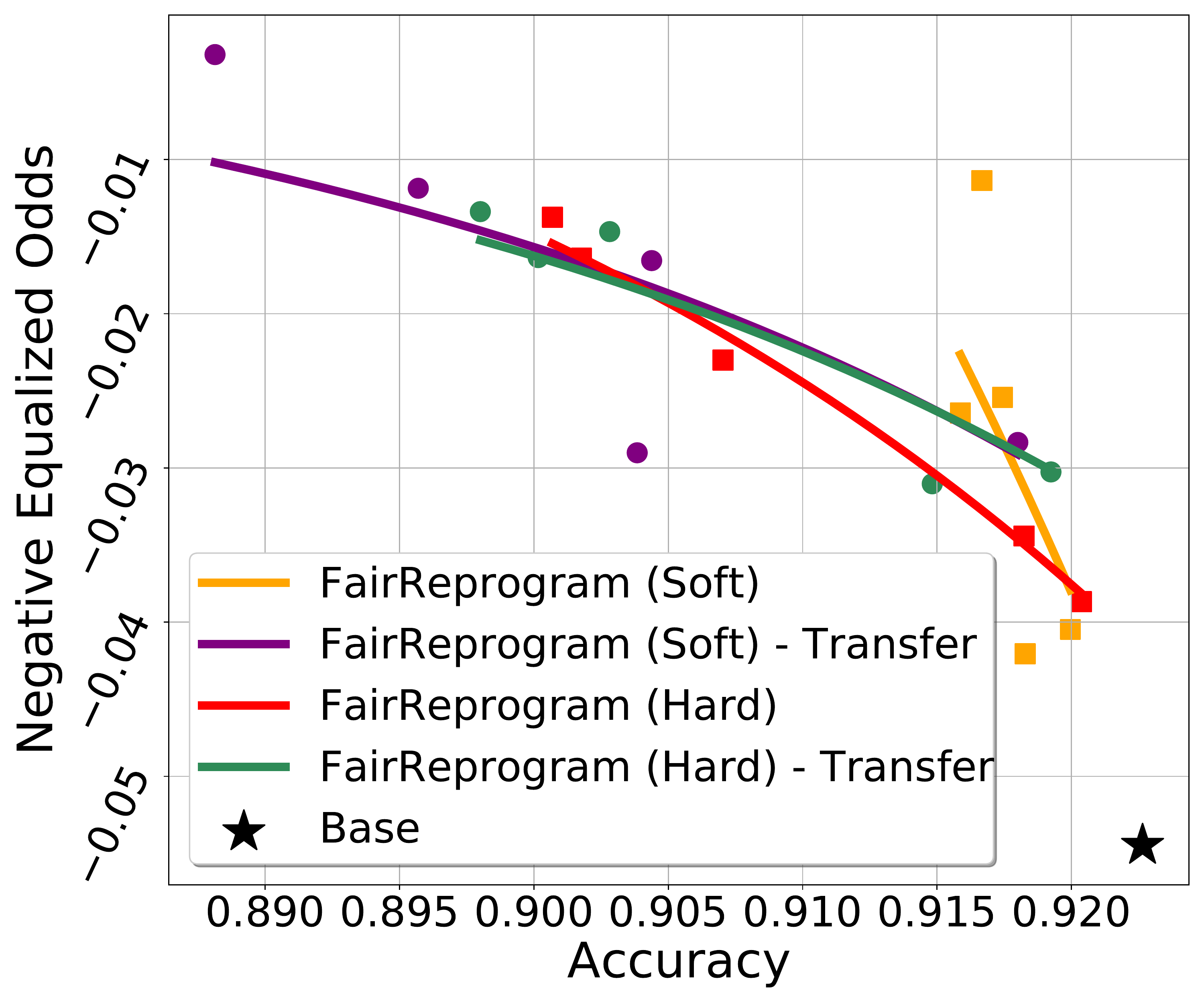} &
    \hspace*{-4mm}  \includegraphics[width=.25\textwidth,height=!]{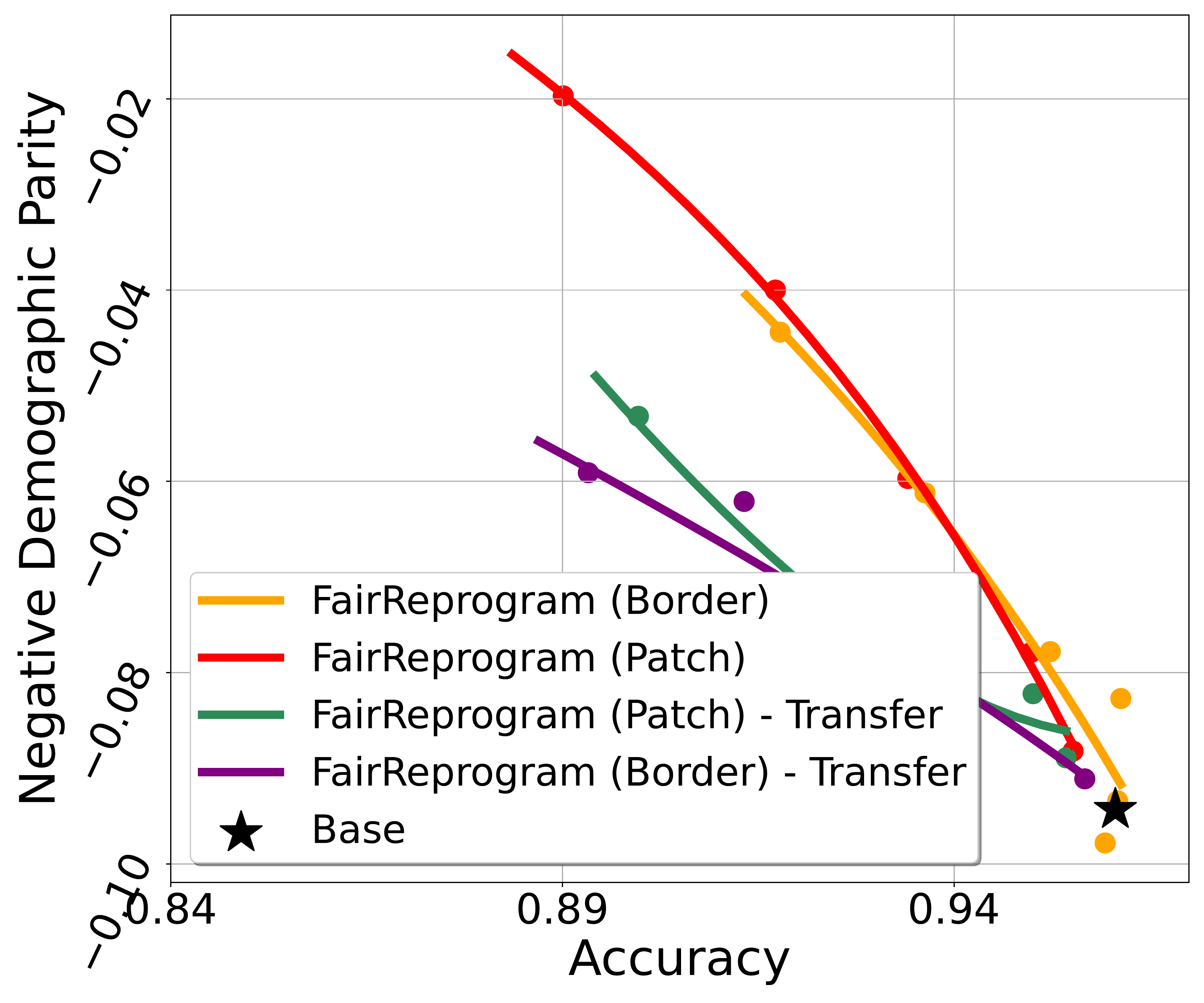} &
    \hspace*{-4mm} \includegraphics[width=.25\textwidth,height=!]{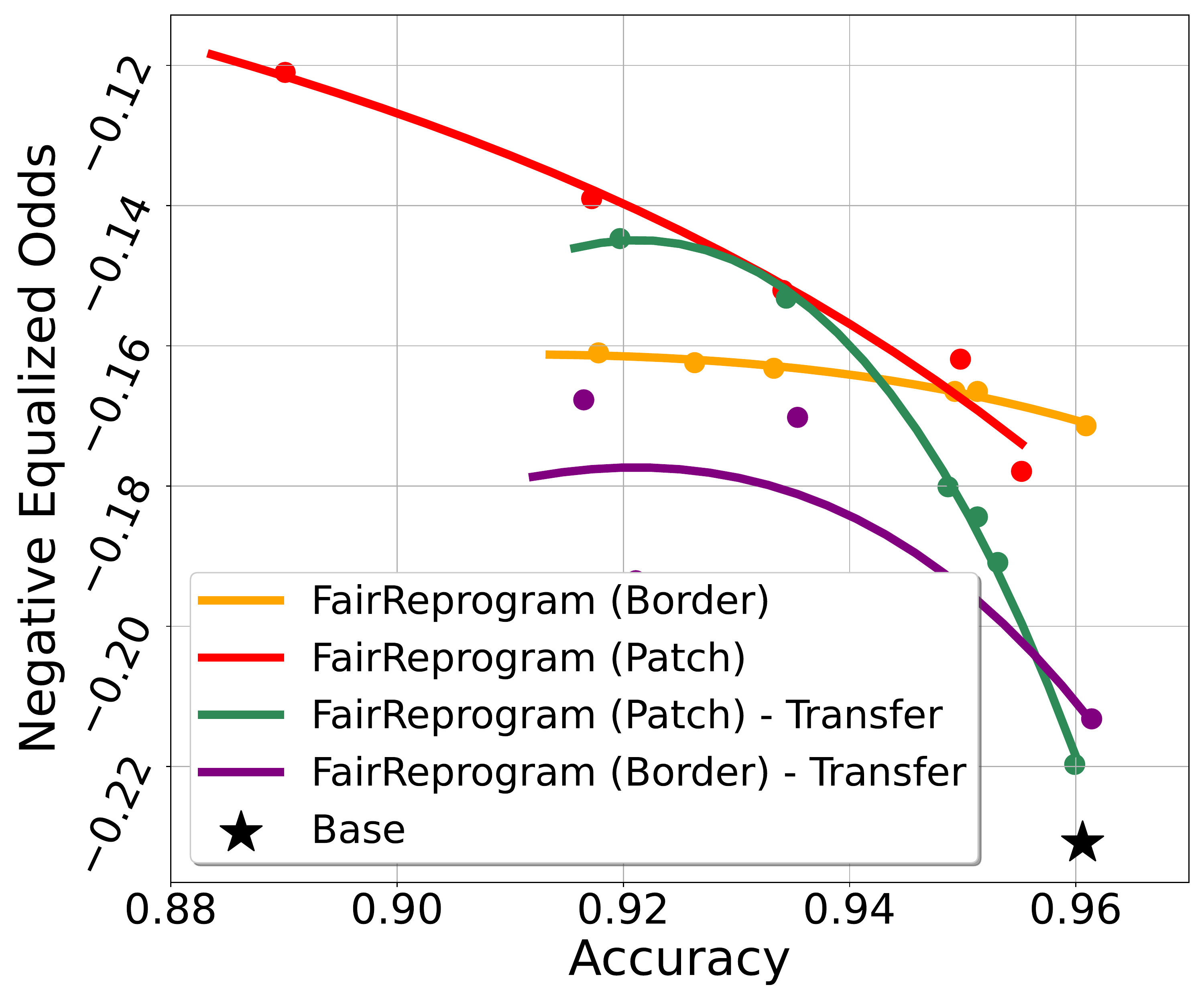} \\
    \multicolumn{2}{c}{\footnotesize{(a) \texttt{Civil Comments}}} & \multicolumn{2}{c}{\footnotesize{(b) \texttt{CelebA}}}
\end{tabular}}
\caption{\footnotesize{ 
Results 
in the transfer setting. We report negative DP (left) and negative EO (right) scores. 
The triggers are firstly trained in a \textsc{Base} model.
Then we evaluate the triggers based on another unseen \textsc{Base} model.
We change the parameter $\lambda$ to trade-off accuracy with fairness and draw the curves in the same way with Fig.~\ref{fig: exp_overview}.
The \FiveStar\,  point corresponds to the average of all \textsc{Base} models with different random seeds.
}}
\label{fig: transfer}
\end{figure}

\begin{wrapfigure}{r}{80mm}
    \vspace*{-6mm}
    \centering
    \includegraphics[width=\linewidth]{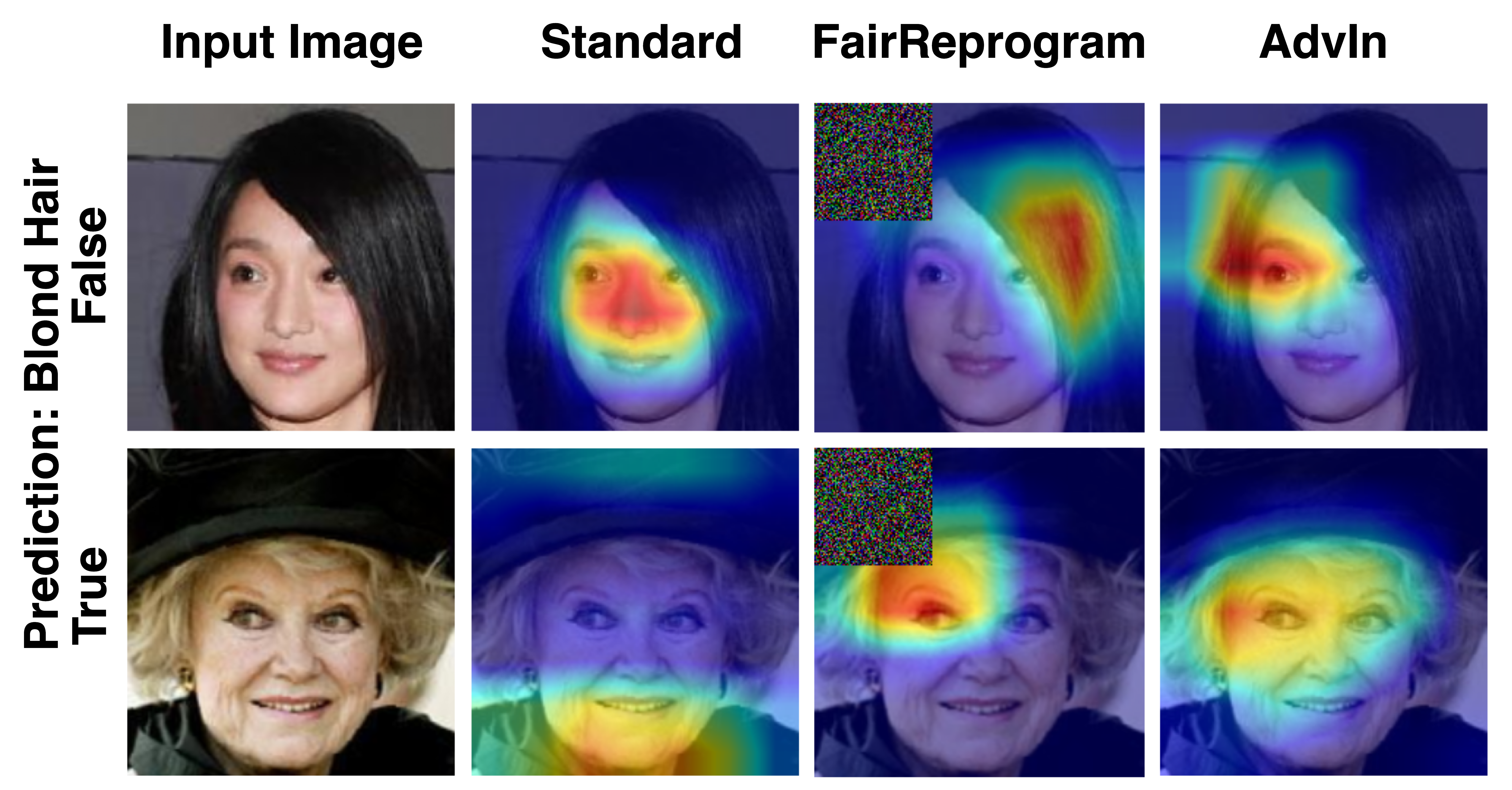}
    \vspace*{-4.5mm}
    \caption{\footnotesize{Gradient-based saliency map visualized with \textsc{Grad Cam}~\cite{selvaraju2017grad} of different methods. 
    The highlighted zones (marked in red) depicting regions exerting major influence on the predicted labels (non-blond hair v.s. blond hair) in each row, which also depict the attention of the model on the input image.
    }}
    \label{fig: vis_cv}
    \vspace*{-4mm}
\end{wrapfigure}

Next, we show the transferability of the fairness triggers found by {\ours}.
We first tune the triggers with a \textsc{Base} source model and then apply the trigger on a target model trained with a different random seed.
The results are shown in Fig.~\ref{fig: transfer}.
As can be seen, {\ours} achieves comparable fairness-accuracy trigger on both the source model and the target model, indicating our method has a good transferability.
This intriguing property brings two benefits of our method: 
\ding{172} if ML model parameters are infeasible (\emph{e.g.} when ML models are provided as services), the users could train a surrogate model and tune the trigger based on it to promote fairness of the original model; 
\ding{173} when ML model parameters are updated with new data (\emph{e.g.} online learning), the user could still use the original trigger for fixing fairness problems.
We further elaborate the results of {\ours} for transferring to different tasks and model architectures in Appendix~\ref{app:transfer}.

\paragraph{Input saliency attribution.}
Fig.~\ref{fig: vis_nlp} and \ref{fig: vis_cv} compare the saliency maps of some example inputs with and without the fairness triggers. Specifically, 
For the NLP applications, we extract a subset of \texttt{Civil Comments} with religion-related demographic annotations,
and apply IG~\cite{Sundararajan2017AxiomaticAF} to localize word pieces that contribute most to the text toxicity classification.  For the CV application, we use GradCam~\cite{selvaraju2017grad} to identify class-discriminative regions of {\texttt{CelebA}}'s  test images.
As shown in Fig.~\ref{fig: vis_nlp}, our fairness trigger consists of a lot of religion-related words (\textit{e.g.}, diocesan, hebrew, parish). Meanwhile, the predicted toxicity score of the benign text starting from `muslims' significantly reduces. These observations verify our theoretical hypothesis that the fairness trigger is strongly indicative of a certain demographic group to prevent the classifier from using the true demographic information.
In addition, Fig.~\ref{fig: vis_cv} presents the input saliency maps on two input images with respect to their predicted labels, non-blond hair and blond hair, respectively. As can be observed, when there is no fairness trigger, the saliency region incorrectly concentrates on the facial parts, indicating the classifier is likely to use biased information, such as gender, for its decision. With the fairness trigger, the saliency region moves to the hair parts, which matches the behavior of \textsc{AdvIn}. These results confirm that our fairness trigger can drive models to make fairer predictions.

\begin{wraptable}{r}{0.5\textwidth}
\vspace*{-4mm}
\centering
\caption{\footnotesize{Predictions of the demographic classifier on a null input with triggers generated by different $\lambda$. The demographic prediction for CV triggers indicate the predicted score for \emph{Male} and \emph{Female}, and it is \emph{Christian}, \emph{Muslim} and \emph{Other religion} for NLP.
}}
\label{tab: demo_classifier}
\resizebox{0.5\textwidth}{!}{
\begin{tabular}{cc}
\toprule[1pt]
\midrule
Trigger & Demographic Prediction\\
\midrule
\includegraphics[width=0.05\textwidth]{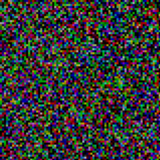} & 0.85, 0.15\\
\includegraphics[width=0.05\textwidth]{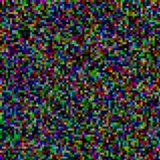} & 0.92, 0.08\\
\includegraphics[width=0.05\textwidth]{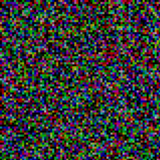} & 0.80, 0.20\\
\midrule
same diocesan bula rev proceedings & 0.96, 0.11, 0.02 \\
soto cc rib hebrew armenian & 0.51, 0.08, 0.81 \\
paul long course parish body & 0.98, 0.04, 0.03 \\
\midrule
\bottomrule[1pt]
\end{tabular}}
\end{wraptable}

\begin{table*}[t]
	\small
	\begin{tabular}{p{0.77\textwidth}c}
        \emph{Text (Non-toxic)} & \emph{Predicted Toxicity} \\
		\arrayrulecolor{grey}  
		\midrule
        \hlg{\ul{\textbf{muslims}}}{20} \hlr{need}{0} \hlr{to}{0} \hlr{take}{0} \hlr{a}{0} \hlr{look}{0} \hlr{in}{0} \hlg{\ul{\textbf{the}}}{15} \hlr{\ul{\textbf{mirror}}}{20} & 0.149 \\
        \hlg{\ul{\textbf{muslims}}}{20} \hlr{need}{0} \hlr{to}{0} \hlr{take}{0} \hlr{a}{0} \hlr{look}{0} \hlr{in}{0} \hlr{the}{0} \hlr{mirror}{0} \,\,\,\,
        \hlr{same}{0} \hlr{\ul{\textbf{diocesan}}}{15} \hlg{\ul{\textbf{bula}}}{20} \hlr{\ul{\textbf{rev}}}{8} \hlr{\ul{\textbf{proceedings}}}{8} & 0.069 \\
        \hlg{\ul{\textbf{muslims}}}{20} \hlr{need}{0} \hlr{to}{0} \hlr{take}{0} \hlr{a}{0} \hlr{look}{0} \hlr{in}{0} \hlr{the}{0} \hlr{mirror}{0}  \,\,\,\,
        \hlr{soto}{0} \hlr{\ul{\textbf{cc}}}{10} \hlr{rib}{0} \hlr{\ul{\textbf{hebrew}}}{10} \hlr{armenian}{0} & 0.054 \\
        \hlg{\ul{\textbf{muslims}}}{20} \hlr{need}{0} \hlr{to}{0} \hlr{take}{0} \hlr{a}{0} \hlr{look}{0} \hlr{in}{0} \hlr{the}{0} \hlg{\ul{mirror}}{0}  \,\,\,\,
        \hlr{\ul{\textbf{paul}}}{10} \hlr{long}{0} \hlr{course}{0} \hlr{\ul{\textbf{parish}}}{10} \hlr{\ul{\textbf{body}}}{10} & 0.073 \\
	\end{tabular}
    \captionof{figure}{\small{
    A text example from \texttt{Civil Comments} with \textsc{Integrated Gradient}~\cite{Sundararajan2017AxiomaticAF,kokhlikyan2020captum}  highlighting important words that influence \textsc{ERM} model predictions.
    The text is concatenated with three triggers generated with different adversary weight.
    \hlg{\ul{\textbf{Green highlights}}}{20} the words that lean to toxic predictions and 
    \hlr{\ul{\textbf{Red highlights}}}{20} non-toxic leaning words.
    The model prediction tends to be correct after adding the triggers.
    }}
    \label{fig: vis_nlp}
\end{table*}

To further verify that the triggers encode demographic information, we trained a demographic classifier to predict the demographics from the input (texts or images) without triggers.
The obtained demographic classifiers can accurately identify the demographics contained in the inputs and achieve over 0.99 AUC for identifying demographics in the validation datasets.
Then, we use the demographic classifier to predict the demographic information of a null image/text\footnote{We use an empty string as the null text and an all-black image as the null image.} with the trigger.
Specifically, we select three triggers generated with different $\lambda$ values for both two datasets.
The results\footnote{One text could contain multiple religions so the probabilities do not sum to one for NLP triggers.} can be seen in Table~\ref{tab: demo_classifier}.
We see that the demographic classifier gives confident outputs on the triggers. 
For example, we see that the trigger \emph{paul long course parish body} is classified as containing \emph{christian} with 0.98 confidence, indicating that the found triggers are highly indicative of demographics.
This is consistent with our perspective in Section~\ref{sec:perspective} that the fairness triggers are encoding fake demographic information to obscure ML models from making biased predictions.

\subsection{Multi-Class Classification}

\begin{figure}[!t]
\centerline{
\begin{tabular}{cccc}
    \hspace*{-2mm} \includegraphics[width=.25\textwidth,height=!]{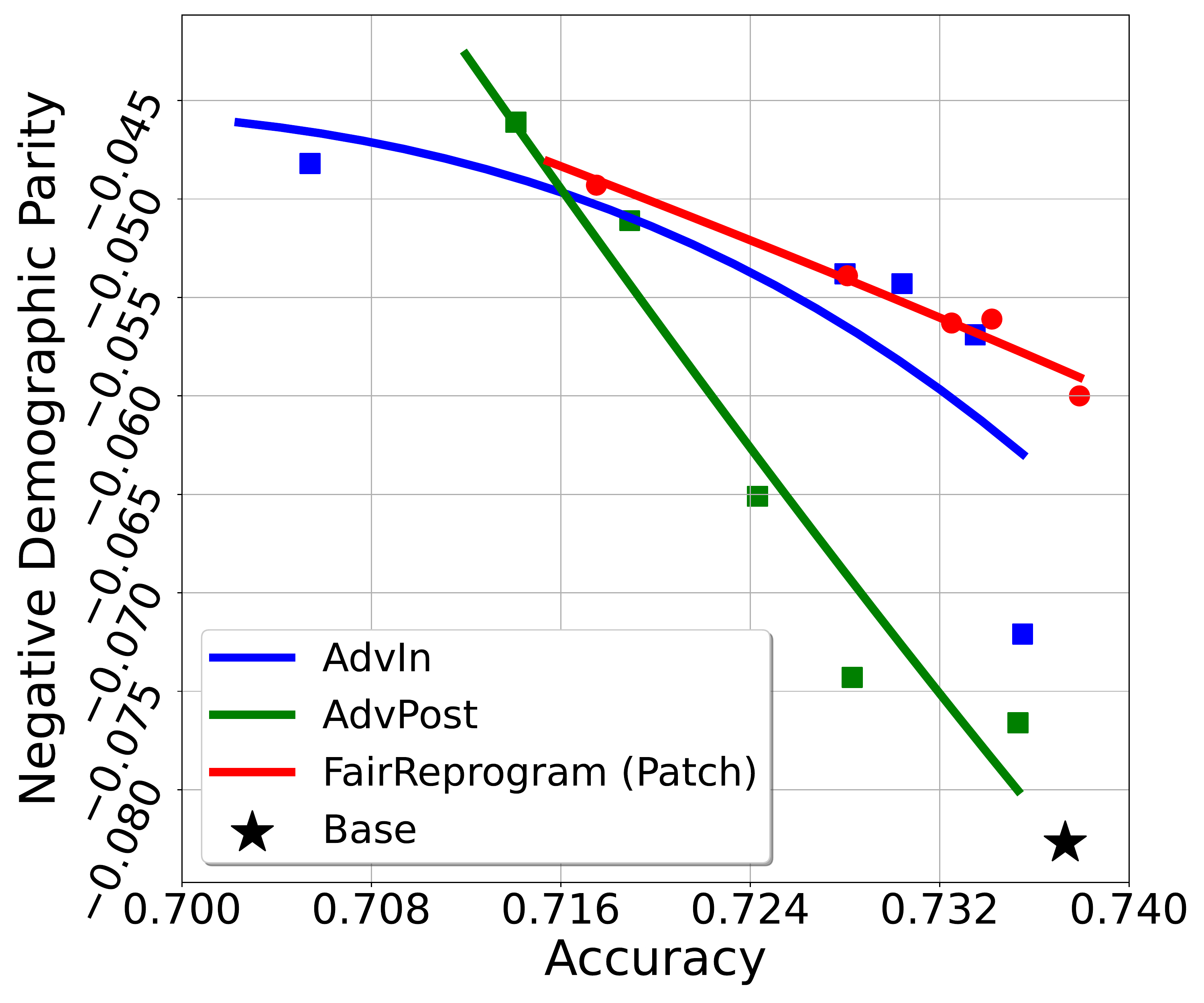} &
    \hspace*{-4mm}  \includegraphics[width=.25\textwidth,height=!]{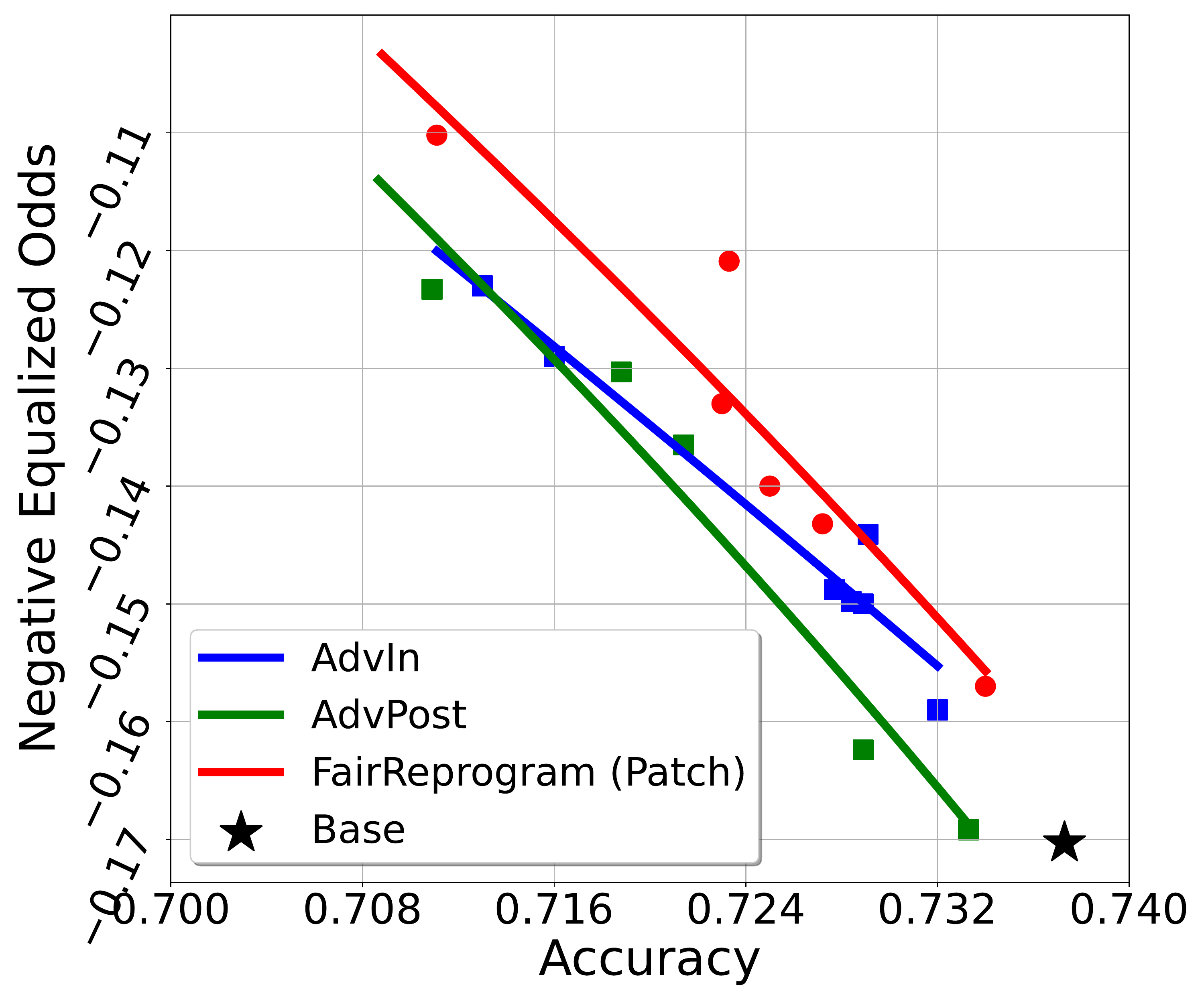} &
    \hspace*{-4mm}  \includegraphics[width=.25\textwidth,height=!]{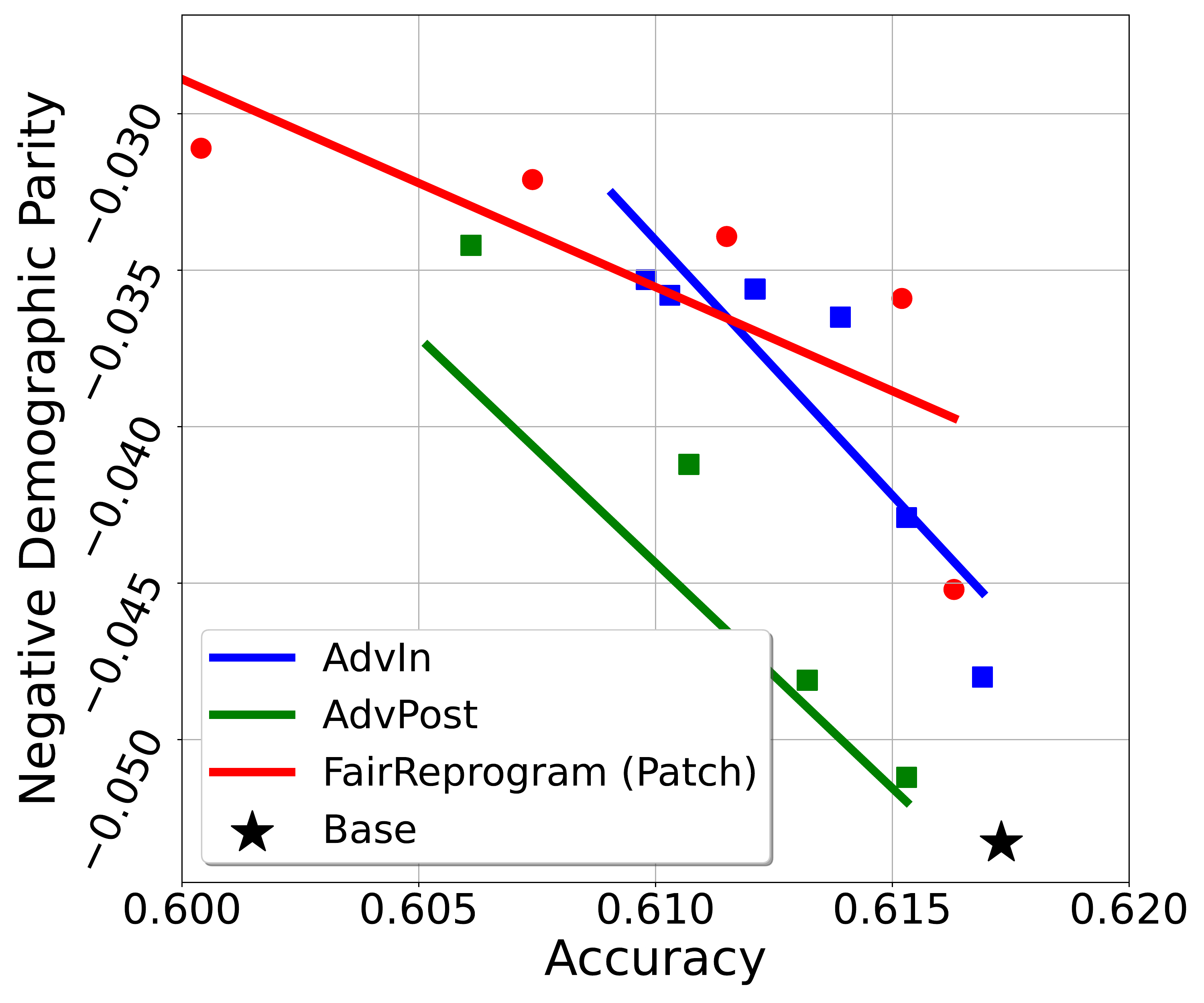} &
    \hspace*{-4mm} \includegraphics[width=.25\textwidth,height=!]{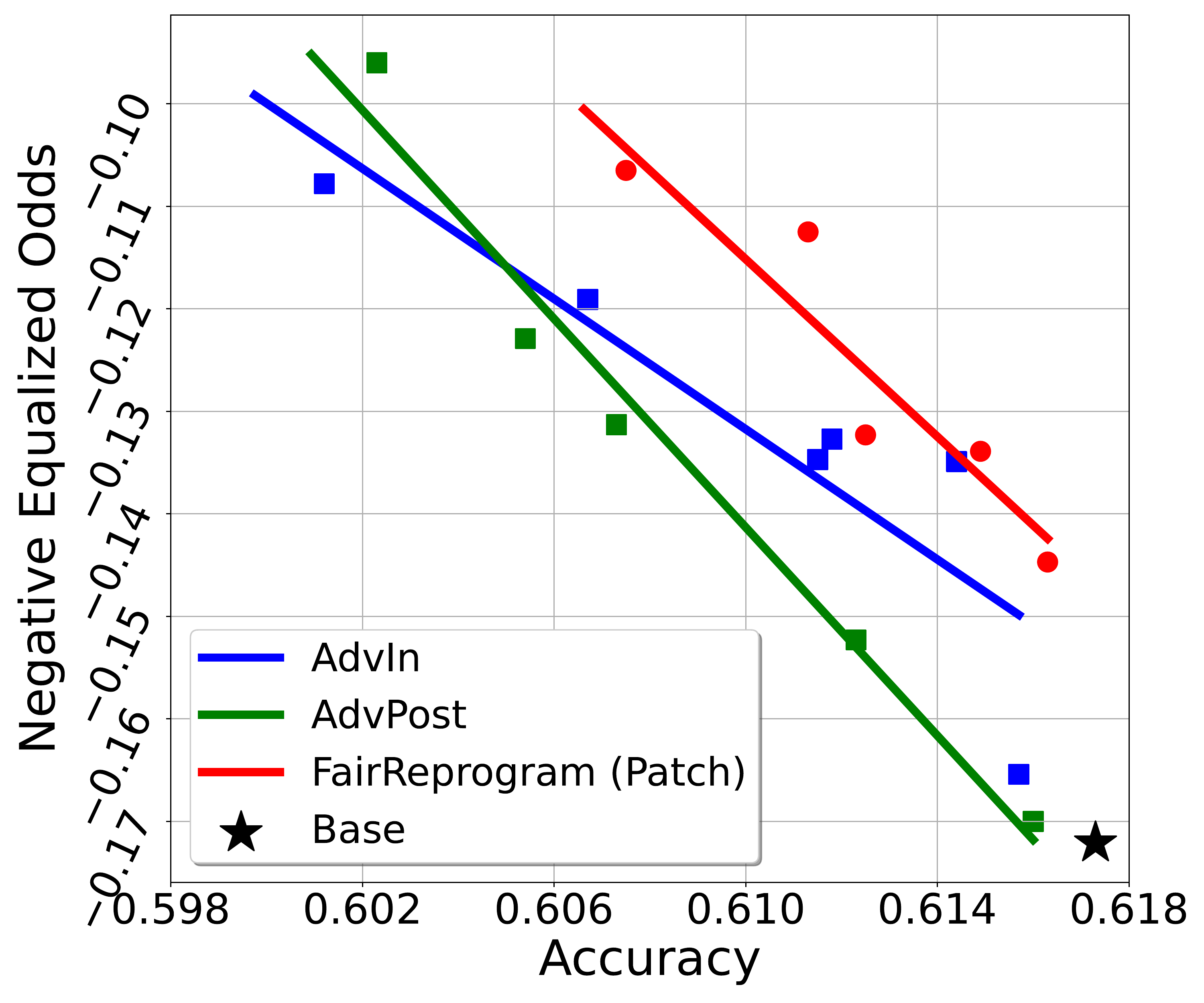} \\
    \footnotesize{(a) 8-class negative DP} &   \footnotesize{(b) 8-class negative EO} & \footnotesize{(c) 16-class negative DP} &   \footnotesize{(d) 16-class negative EO}
\end{tabular}}
\caption{\footnotesize{Performance of multi-class classification.  For (a) and (b), we use the attributes \textit{Blond Hair, Smiling, Attractive} for multi-class construction.  We add an addition attribute \textit{Wavy Hair} for (c) and (d).}}
\label{fig: multi-class}
\end{figure}
To extend our evaluation to a multi-class setting,  we use the \texttt{CelebA} dataset and select $n$ binary attributes that may be spuriously correlated with \emph{gender} \cite{xu2020investigating, dash2020counterfactual, hwang2020fairfacegan}.    Then, following \cite{zhuang2018multi}, we construct data groups by enumerating all $2^n$ possible binary vectors, where each dimension corresponds to a binary attribute.  We index these vectors and treat them as the class labels.  Fig.~\ref{fig: multi-class} shows the accuracy-fairness trade-off curves similar to Fig.~\ref{fig: exp_overview}. It can be observed that our method outperforms the other methods as the \textcolor{red}{red} curves are closer to the top-right corner.   Also, as the class label number increases, the post-processing-based {\advp} falls behind its in-processing counterpart {\adv}, indicating a larger class number may induce more challenges to post-processing methods.

\subsection{Ablation Studies}
\label{sec: ablation_study}

\begin{figure}[!t]
\centerline{
\begin{tabular}{cccc}
    \hspace*{-2mm} \includegraphics[width=.25\textwidth,height=!]{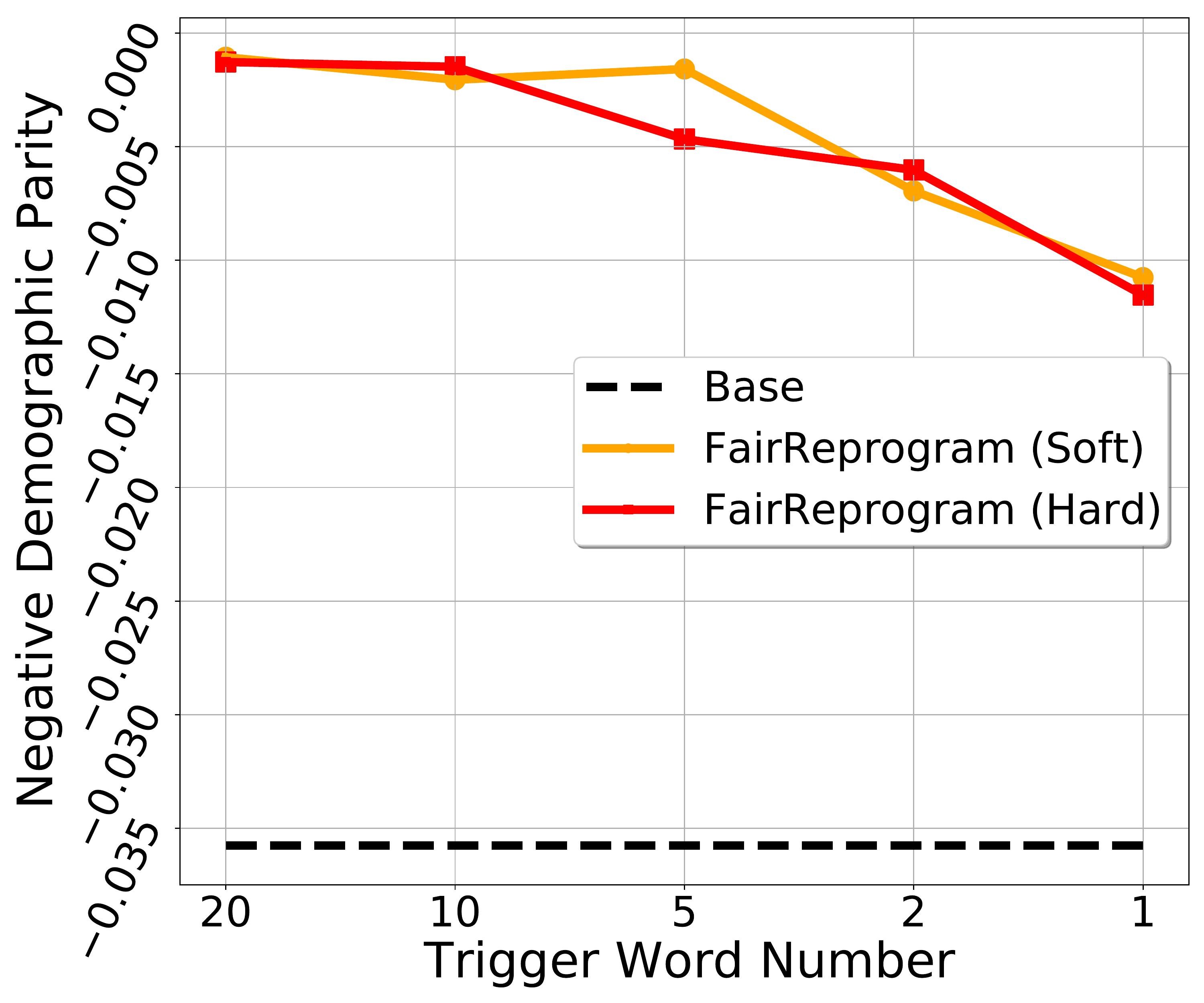} &
    \hspace*{-4mm}  \includegraphics[width=.25\textwidth,height=!]{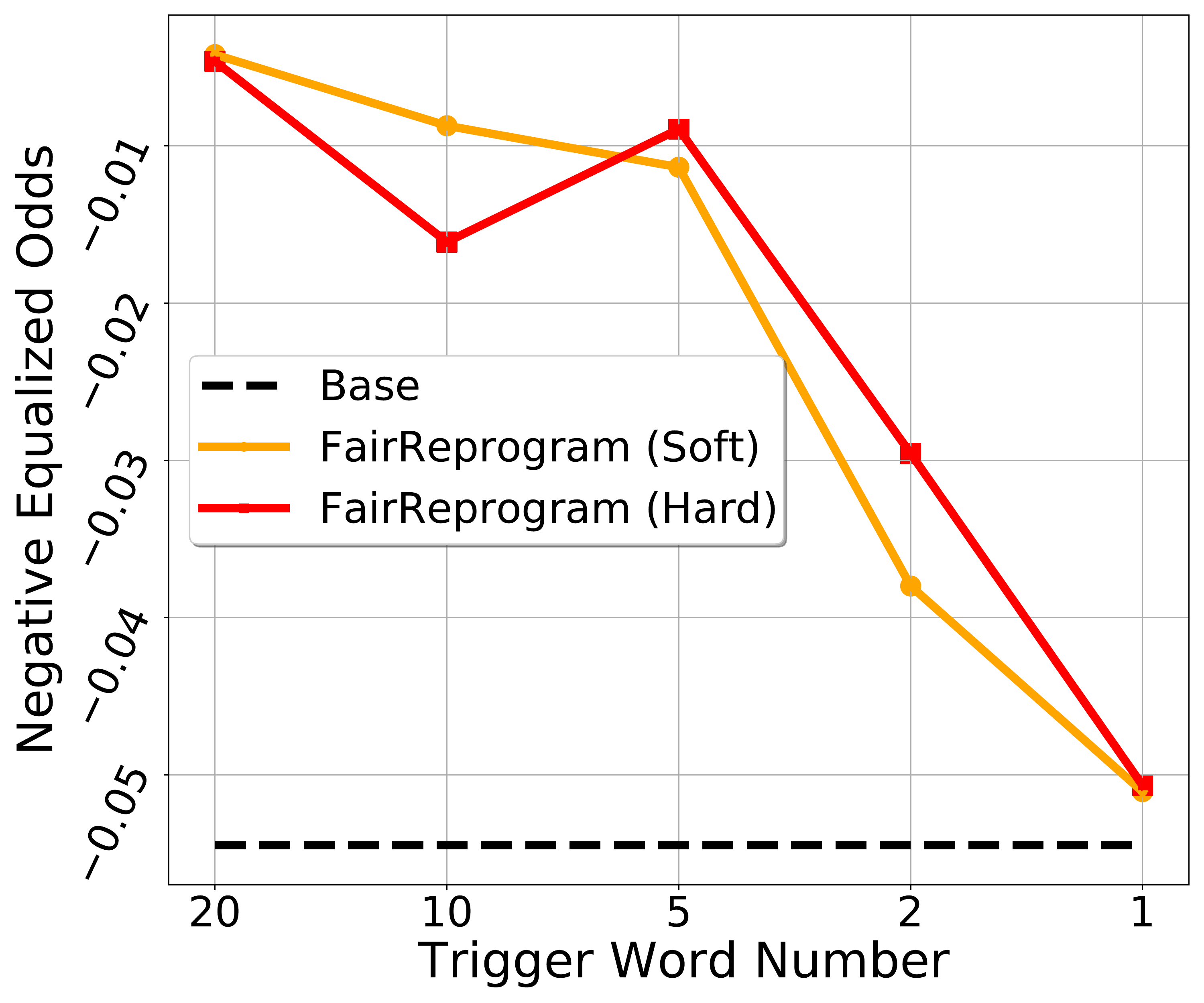} &
    \hspace*{-4mm}  \includegraphics[width=.25\textwidth,height=!]{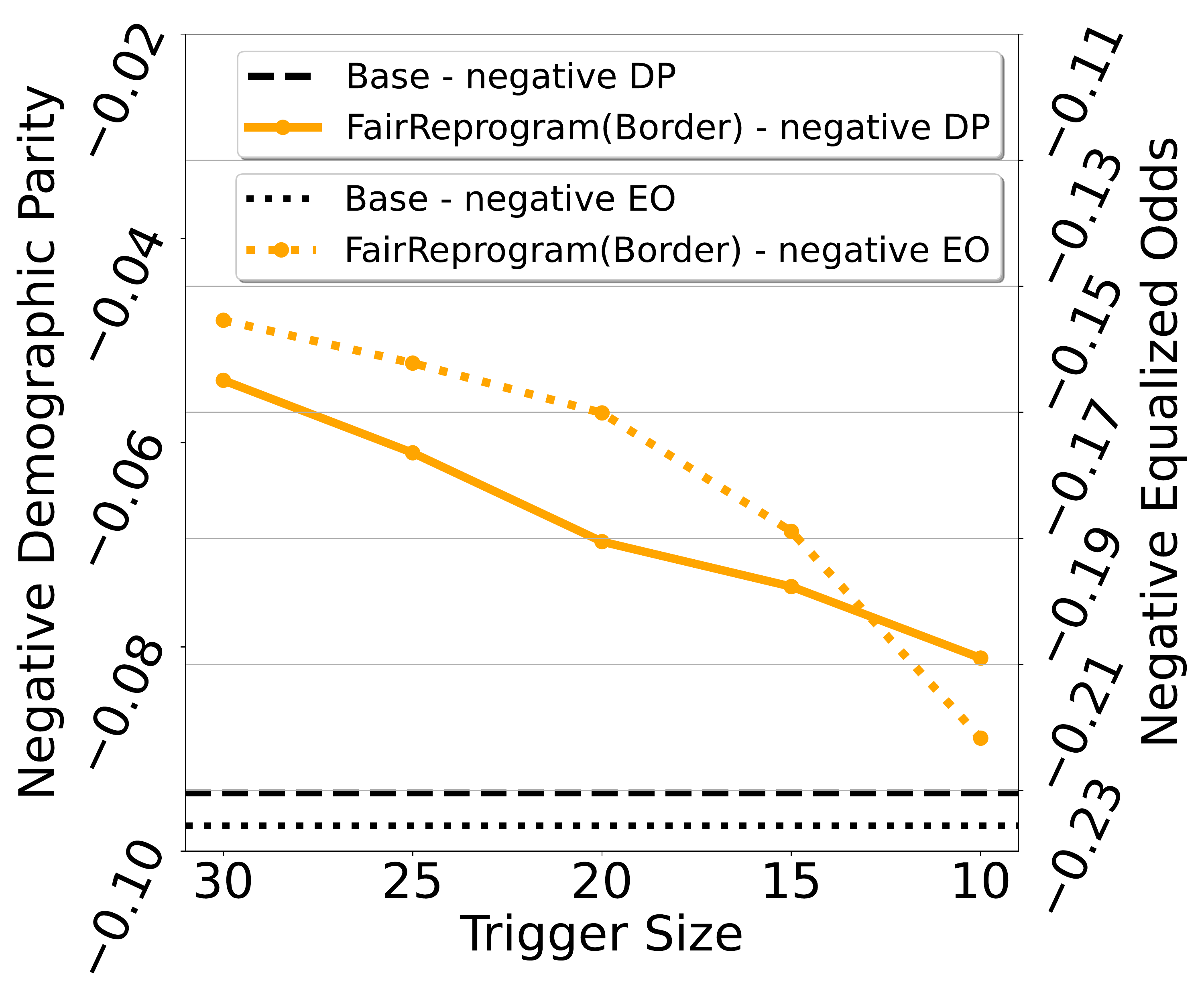} &
    \hspace*{-4mm} \includegraphics[width=.25\textwidth,height=!]{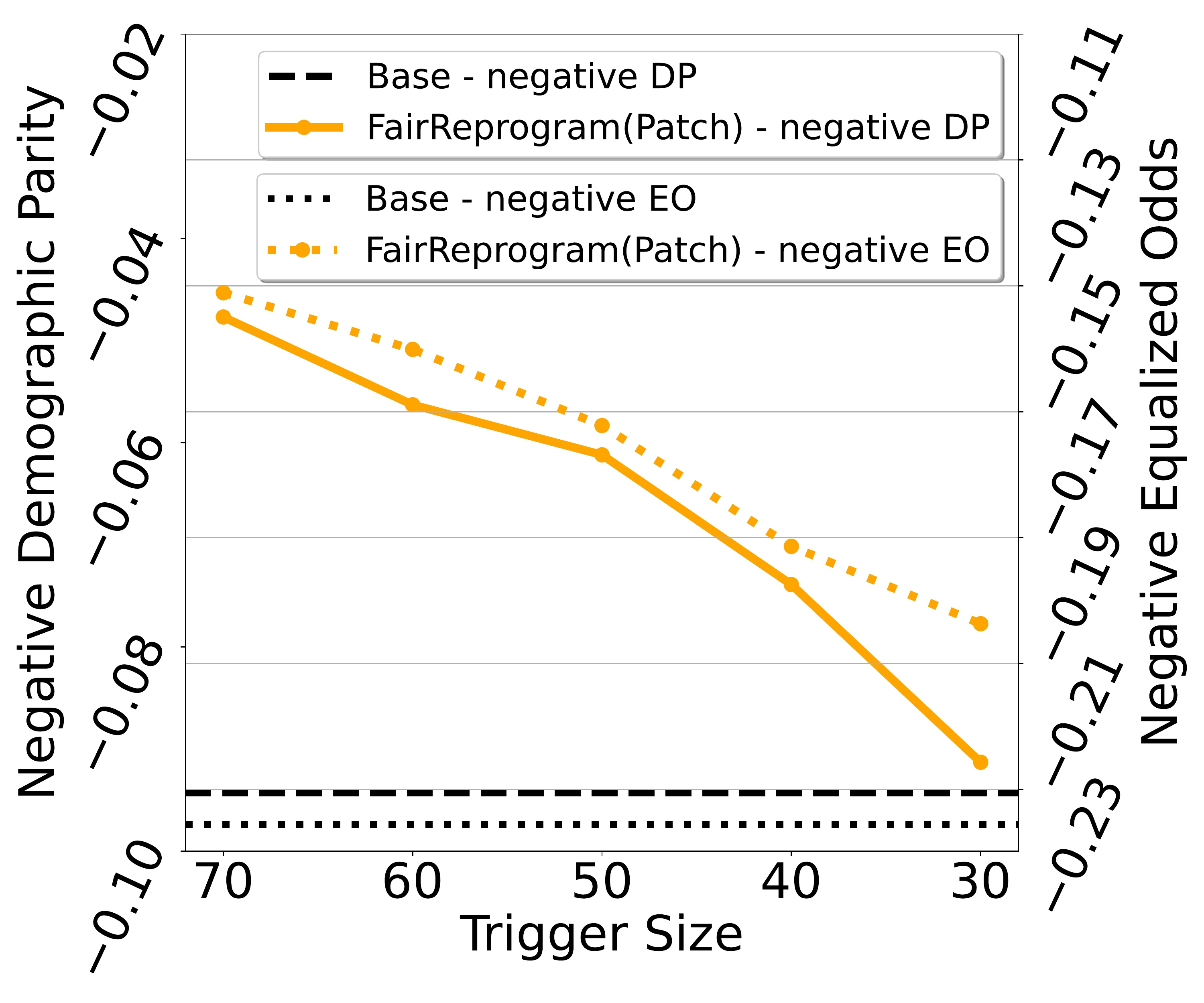} \\
    \multicolumn{2}{c}{\footnotesize{(a) \texttt{Civil Comments}}} & \multicolumn{2}{c}{\footnotesize{(b) \texttt{CelebA}}}
\end{tabular}}
\vspace*{-1mm}
\caption{\footnotesize{Ablation study of the trigger size.
We evaluate the bias scores with different trigger word numbers (\texttt{Civil Comments}) and different trigger size (\texttt{CelebA}) with fixed adversary weight $\lambda$.
}}

\label{fig: size_study}
\end{figure}

We perform an ablation study to investigate the effects of the trigger size.
Specifically, we run experiments with different numbers of trigger words / trigger patch sizes on the NLP / CV dataset.
We set a $\lambda$ value for each method such that all methods achieve comparable bias scores with the largest trigger size.
The detailed $\lambda$ choices can be seen in Appendix~\ref{app: training_detail}.
Then we train the triggers with different sizes in the tuning set using the fixed $\lambda$'s.  For the text trigger as shown in Fig.~\ref{fig: size_study}(a), we see that the negative bias score gets worse as the number of trigger words gets smaller.
However, our method can still improve fairness upon the \textsc{Base} model even with only a one-word trigger.
On the other hand, the results with five trigger words and above are all comparable, indicating that five words is enough to achieve the fairness goal.
Similarly, for the image trigger as shown in Fig.~\ref{fig: size_study}(b), the results suggest a larger trigger would consistently improve fairness.
On the other hand, we show that larger trigger size could hurt accuracy in Appendix~\ref{app: additional_results}, which is similar to the effect of increasing $\lambda$. 

\CR{\subsection{Summary of Additional Results}
We compare our proposed {\ours} with four additional baselines,
and we show the full results with variance in Tab.~\ref{tab: overview}. 
We further compare our method with MMD methods where $\mathcal{L}_{fair}$ in Eq.~\eqref{eq:loss} is replaced with Maximum Mean Discrepancy regularization~\cite{Louizos2016THEVF} to partial out the instability of adversarial training, and the results are shown in Fig.~\ref{fig: rebuttal_mmd}.
We also implement the fairness reprogramming in the black-box setting on \texttt{CelebA} dataset, where the model parameters are not available for training the reprogram, and the results are shown in Fig.~\ref{fig: blackbox}. 
Besides, we show that {\ours} could also be used in tabular data, and the corresponding experiment results on the \texttt{Adult} dataset are shown in Fig.~\ref{fig: rebuttal_tabular}.
}

\section{Conclusion} 
In this paper, we introduce a novel model reprogramming based fairness promoting method, termed {\ours}.
Specifically, {\ours} considers a fixed ML model and optimizes a set of vectors, named fairness trigger, concatenated on inputs to boost model fairness.
We introduce an information-theoretic framework to explain the rationales of why {\ours} can improve model fairness.  
As implied by our theoretic framework as well as our empirical findings, the fairness trigger can effectively mask out the true demographic information with its strong, false demographic information. 
Extensive experiments demonstrate that our method could achieve better fairness improvements to retraining based methods with far-less training cost.
We further empirically show fairness triggers enjoys great transferability and interpretability.
We hope that {\ours} can inspire new fairness learning paradigms that are more feasible and flexible in practice.

\newpage

{{
\bibliographystyle{IEEEbib}

\bibliography{ref}
}}

\newpage

\appendix

\section{Experiment Setup}
\label{app: exp_setup}
\subsection{Dataset Details}
\label{app: data_statistic}
\begin{table*}
    \caption{Statistics of the datasets. ``Pos. (\%)'' column indicates the ratio of positive labels (\emph{i.e.} ``blond hair'' for \texttt{CelebA}, ``toxic'' for \texttt{Civil Comments}).  We use ``Tr.'', ``Tun.'', ``Val.'' and ``test.'' columns to indicate the size of training set, tuning set, validation set and testing set. ``Demographics'' column indicates the considered demographics in each dataset. }
    \label{tab: data_statistic}
    \centering
    \resizebox{1.0\textwidth}{!}{
    \begin{tabular}{c|c|c|c|c|c|c|c}
        \toprule[1pt]
        \midrule
        Dataset & Task & Pos. (\%) & Tr. & Tun. & Val. & Test. & Demographics \\
        \midrule
        \texttt{CelebA} & Hair color recognition & 17.4 & 161143 & 1627 & 19867 & 19962 & Gender \\
        \texttt{Civil Comments} & Toxicity classification & 11.3 & 223858 & 45180 & 45180 & 133782 & Gender, Sex orientation, Race, Religion \\
        \midrule
        \bottomrule[1pt]
    \end{tabular}}
\end{table*}
The dataset splitting setting and demographic information of the datasets are shown in Tab.~\ref{tab: data_statistic}.

\subsection{Training Details}
\label{app: training_detail}

We specify the different $\lambda$ values used to generate the curves in Figs.~\ref{fig: exp_overview} and Figs.~\ref{fig: transfer} in Tab.~\ref{tab: overview}.

For \texttt{Civil Comments} in Figs~\ref{fig: limit_data} and \ref{fig: size_study}, we set $\lambda=0.5$ for {\advp}, $\lambda=50.0$ for \textsc{FairReprogram (Soft)} and $\lambda=1000.0$ for \textsc{FairReprogram (Hard)} with the DP measure;  
we set $\lambda=1.0$ for {\advp} and $\lambda=50.0$ for both of our methods with the EO measure.

For \texttt{CelebA}
in Fig.~\ref{fig: limit_data} and Fig.~\ref{fig: size_study}, we set $\lambda=1.0$ for \textsc{FairReprogram (Border)} and \textsc{FairReprogram (Patch)} with DP and we set $\lambda=10.0$ for both with EO. By default, the trigger size of \textsc{FairReprogram (Border)} is set to 20, which corresponds to the width of the trigger frame. The trigger size of \textsc{FairReprogram (Patch)} is fixed to 80, namely the width of the trigger block attached to the original input image. For {\adv} and {\advp}, the $\lambda$ is set to 0.1 in the setting with DP and $\lambda$ is fixed to 0.5 for training with EO.
The value of $\lambda$ is selected so that different methods achieve comparable bias scores.

\newpage

\section{Additional Experiment Results}
\label{app: additional_results}

\subsection{Experiments with Additional Post-processing Baselines}

\begin{table*}[!htb]
    \caption{\footnotesize{Numerical results with standard derivation on \texttt{Civil Comments} and \texttt{CelebA} shown in Fig.~\ref{fig: exp_overview}.
    All reported results are the average of three different random runs.
    We report the negative DP and the negative EO scores correspondingly for the ``Fairness'' column.
    Note that the best models are also selected based on corresponding fairness measures.}
    } 
    \label{tab: overview}
    \centering
    \resizebox{1.0\textwidth}{!}{
    \begin{tabular}{c|ccc|ccc|ccc|ccc}
        \toprule[1pt]
        \midrule
        \multirow{3}{*}{Method} & \multicolumn{6}{c}{\texttt{Civil Comments}} & \multicolumn{6}{|c}{\texttt{CelebA}} \\
        & \multicolumn{3}{c|}{\texttt{Demographic parity}} & \multicolumn{3}{c|}{\texttt{Equalized odds}} & \multicolumn{3}{c|}{\texttt{Demographic parity}} & \multicolumn{3}{c}{\texttt{Equalized odds}} \\
        & $ \lambda$ & Accuracy & Fairness & $ \lambda$ & Accuracy & Fairness  & $ \lambda$ & Accuracy & Fairness   & $ \lambda$ & Accuracy & Fairness\\
        \midrule
        \multirow{1}{*}{\erm} 
        & -  & $0.922 _{\pm 0.004}$ & $-0.036 _{\pm 0.005} $&- &  $0.923_{\pm 0.004}$ & $-0.054 _{\pm 0.025}$
        & - & $0.961_{\pm 0.004}$ & $-0.094_{\pm 0.002}$ & - & $0.961_{\pm 0.004}$ & $-0.231_{\pm 0.008}$\\
        \midrule
        \multirow{5}{*}{\adv} 
        & 0.0 & $0.926 _{\pm 0.004} $& $-0.036 _{\pm 0.025} $& 0.0 & $0.919 _{\pm 0.004} $& $-0.033  _{\pm 0.005}$
        &  0.01   & $0.944_{\pm 0.007}$ & $-0.084_{\pm 0.004}$ &   0.1    & $0.952_{\pm 0.005}$ & $-0.181_{\pm 0.009}$ 
         \\ 
        & 0.1 & $0.899_{\pm 0.014} $ & $-0.016_{\pm 0.004} $ & 0.1 & $0.899 _{\pm 0.006} $& $-0.016  _{\pm 0.009}$
        &  0.05   & $0.938_{\pm 0.005}$ & $-0.085_{\pm 0.007}$ &   0.3    & $0.941_{\pm 0.002}$ & $-0.177_{\pm 0.014}$ 
         \\ 
        & 0.5 & $0.905 _{\pm 0.003}$ &$ -0.018 _{\pm 0.035}$ & 1.0 & $0.919  _{\pm 0.006}$& $-0.023  _{\pm 0.002}$
        &  0.1   & $0.932_{\pm 0.006}$ & $-0.072_{\pm 0.005}$ &   0.5    & $0.936_{\pm 0.005}$ & $-0.175_{\pm 0.012}$ 
         \\ 
        & 5.0 & $0.889 _{\pm 0.007}$ & $-0.001 _{\pm 0.039}$ & 5.0 & $0.889 _ {\pm 0.005} $ & $-0.001 _{\pm 0.007}  $
        &  0.2   & $0.911_{\pm 0.003}$ & $-0.071_{\pm 0.009}$ &   1.0    & $0.913_{\pm 0.007}$ & $-0.168_{\pm 0.009}$ 
         \\ 
        & 20.0 &$ 0.888 _{\pm 0.023} $& $-0.000 _{\pm 0.039} $& 20.0 & $0.888 _ {\pm 0.009} $ & $-0.000 _{\pm 0.006} $
        &  0.3   & $0.897_{\pm 0.002}$ & $-0.064_{\pm 0.002}$ &   2.0    & $0.901_{\pm 0.004}$ & $-0.153_{\pm 0.007}$ 
         \\ 
        \midrule
        \multirow{5}{*}{\advp} 
        & 0.0 & $0.923_{\pm 0.003} $& $-0.035 _{\pm 0.019}$& 0.0& $0.921 _{\pm 0.005}$& $-0.055 _{\pm 0.004}$
        &  0.01   & $0.959_{\pm 0.004}$ & $-0.097_{\pm 0.004}$ &   0.1    & $0.947_{\pm 0.003}$ & $-0209._{\pm 0.007}$ 
         \\ 
        & 0.1 & $0.925 _{\pm 0.002}$& $-0.043_{\pm 0.011} $& 0.2& $0.925 _{\pm 0.002}$& $-0.045 _{\pm 0.002}$
        &  0.05   & $0.947_{\pm 0.009}$ & $-0.083_{\pm 0.005}$ &   0.3    & $0.931_{\pm 0.005}$ & $-0.201_{\pm 0.011}$ 
         \\ 
        & 0.5 & $0.909 _{\pm 0.011}$& $-0.007_{\pm 0.032}$ & 0.4& $0.925 _{\pm 0.002}$& $-0.041 _{\pm 0.004}$
        &  0.1   & $0.931_{\pm 0.009}$ & $-0.074_{\pm 0.007}$ &   0.5    & $0.918_{\pm 0.004}$ & $-0.168_{\pm 0.015}$ 
         \\ 
        & 1.0 & $0.888_{\pm 0.022}$ & $-0.000_{\pm 0.033}$ & 0.7& $0.924 _{\pm 0.002}$& $-0.042 _{\pm 0.005}$
        &  0.2   & $0.917_{\pm 0.003}$ & $-0.069_{\pm 0.005}$ &   1.0    & $0.911_{\pm 0.008}$ & $-0.165_{\pm 0.011}$ 
         \\ 
        & 5.0 & $0.888 _{\pm 0.022}$& $-0.000_{\pm 0.033}$ & 1.0 & $0.888 _{\pm 0.022}$& $-0.000 _{\pm 0.028}$
        &  0.3   & $0.873_{\pm 0.002}$ & $-0.058_{\pm 0.002}$ &   2.0    & $0.899_{\pm 0.002}$ & $-0.141_{\pm 0.013}$ 
         \\ 
        \midrule
        \multirow{1}{*}{\textsc{EqOdds}}
        &  -   & $0.913_{\pm 0.005}$ & $-0.032_{\pm 0.020}$ & - & $0.915_{\pm 0.003}$ & $-0.031_{\pm 0.005} $
        &  -   & $0.919_{\pm 0.009}$ & $-0.047_{\pm 0.005}$ &   -    & $0.919_{\pm 0.009}$ & $-0.172_{\pm 0.009}$ 
         \\
        \multirow{1}{*}{\textsc{CaliEqOdds}}
        & - & $0.922_{\pm 0.003}$ & $-0.044_{\pm 0.023}$ & - & $0.922_{\pm 0.004}$ & $-0.057_{\pm 0.011} $
        &  -   & $0.927_{\pm 0.007}$ & $-0.053_{\pm 0.005}$ &   -    & $0.927_{\pm 0.007}$ & $-0.169_{\pm 0.018}$ 
         \\
        \multirow{1}{*}{\textsc{RejectOption}} 
        & - & $0.886_{\pm 0.028} $& $-0.152_{\pm 0.052}$ & - & $0.874_{\pm 0.017} $& $-0.101_{\pm 0.002} $
        &  -   & $0.934_{\pm 0.003}$ & $-0.089_{\pm 0.004}$ &   -    & $0.934_{\pm 0.003}$ & $-0.189_{\pm 0.015}$  \\
        \multirow{1}{*}{\textsc{DIRemover}} 
        & - & $0.917_{\pm 0.008}$ & $-0.017_{\pm 0.017}$ & - & $0.922_{\pm 0.003} $& $-0.034_{\pm 0.003} $
        &  -   & $0.959_{\pm 0.004}$ & $-0.086_{\pm 0.003}$ &   -    & $0.959_{\pm 0.004}$ & $-0.183_{\pm 0.014}$  \\
        \midrule
        \multirow{5}{*}{\makecell[c]{\ours \\ (\textsc{Soft} / \textsc{Border}) }} 
        & 0.0 & $0.919_{\pm 0.005}$ & $-0.018_{\pm 0.021}$ & 0.0 & $0.920_{\pm 0.004} $& $-0.040_{\pm 0.002} $
        &  0.1   & $0.961_{\pm 0.002}$ & $-0.093_{\pm 0.005}$ &   2.0    & $0.961_{\pm 0.005}$ & $-0.171_{\pm 0.005}$ 
         \\ 
        & 0.5 & $0.911_{\pm 0.007} $& $-0.012_{\pm 0.018} $& 0.1 & $0.916_{\pm 0.007} $& $-0.026_{\pm 0.007}$
        &  0.5   & $0.959_{\pm 0.005}$ & $-0.087_{\pm 0.006}$ &   5.0    & $0.951_{\pm 0.007}$ & $-0.167_{\pm 0.004}$ 
         \\ 
        & 5.0 & $0.913_{\pm 0.008} $& $-0.009_{\pm 0.011} $& 10.0 & $0.918_{\pm 0.005} $& $-0.042_{\pm 0.006} $
        &  1.0   & $0.952_{\pm 0.007}$ & $-0.078_{\pm 0.005}$ &  10.0   & $0.933_{\pm 0.003}$ & $-0.163_{\pm 0.003}$ 
         \\ 
        & 20.0 & $0.901_{\pm 0.014}$ & $-0.005_{\pm 0.023} $& 20.0 & $0.917_{\pm 0.006} $& $-0.025_{\pm 0.012} $
        &  2.0   & $0.929_{\pm 0.003}$ & $-0.075_{\pm 0.004}$ &   20.0    & $0.926_{\pm 0.004}$ & $-0.162_{\pm 0.005}$ 
         \\ 
        & 100.0 & $0.907_{\pm 0.011}$ & $-0.001_{\pm 0.003} $& 50.0 & $0.917_{\pm 0.004} $& $-0.011_{\pm 0.010} $
        &  5.0   & $0.911_{\pm 0.002}$ & $-0.072_{\pm 0.002}$ &   30.0    & $0.918_{\pm 0.002}$ & $-0.161_{\pm 0.003}$ 
         \\ 
        \midrule
        \multirow{5}{*}{\makecell[c]{\ours \\ (\textsc{Hard} / \textsc{Patch}) }} 
        & 0.0 & $0.908_{\pm 0.008}$ & $-0.010_{\pm 0.016}$ & 0.0 & $0.920_{\pm 0.007} $&$-0.039_{\pm 0.001}  $
        &  0.1   & $0.955_{\pm 0.004}$ & $-0.088_{\pm 0.004}$ &   2.0    & $0.955_{\pm 0.004}$ & $-0.178_{\pm 0.011}$ 
         \\ 
        & 0.1 & $ 0.908_{\pm 0.011} $ & $ -0.008_{\pm 0.022} $ & 20.0 &  $0.918_{\pm 0.005} $ &  $-0.034_{\pm 0.002}  $
        &  0.5   & $0.950_{\pm 0.005}$ & $-0.078_{\pm 0.007}$ &   5.0    & $0.946_{\pm 0.008}$ & $-0.161_{\pm 0.009}$ 
         \\ 
        & 10.0 &  $0.906_{\pm 0.012}$  &  $-0.007_{\pm 0.019} $ & 200.0 &  $0.907_{\pm 0.015}  $&  $-0.023_{\pm 0.017}  $
        &  1.0   & $0.934_{\pm 0.005}$ & $-0.060_{\pm 0.003}$ &   10.0    & $0.934_{\pm 0.004}$ & $-0.152_{\pm 0.007}$ 
         \\ 
        & 30.0 & $0.894_{\pm 0.017} $ &  $-0.003_{\pm 0.021} $ & 600.0 & $ 0.902_{\pm 0.013} $ & $ -0.016_{\pm 0.014}  $
        &  2.0   & $0.917_{\pm 0.003}$ & $-0.040_{\pm 0.008}$ &   20.0    & $0.917_{\pm 0.002}$ & $-0.139_{\pm 0.012}$ 
         \\ 
        & 100.0 & $0.893_{\pm 0.015}$ & $-0.002_{\pm 0.017}$ & 1200.0 & $0.901_{\pm 0.017}$ & $-0.014_{\pm 0.011}$ 
        &  5.0   & $0.890_{\pm 0.001}$ & $-0.019_{\pm 0.002}$ &   30.0    & $0.890_{\pm 0.001}$ & $-0.121_{\pm 0.005}$ \\
        \midrule
        \bottomrule[1pt]
    \end{tabular}}
\end{table*}

We further compare our method with four extra post-processing fairness-promoting baselines.

$\bullet$ \textsc{EqOdds}~\cite{Hardt2016EqualityOO}: Method that alters model predictions to meet equalized odds by solving a linear program.

$\bullet$ \textsc{CaliEqOdds}~\cite{Pleiss2017OnFA}: Method that optimizes the model outputs to achieve a relaxed equalized odds objective together with calibration with information withholding.

$\bullet$ \textsc{RejectOption}~\cite{Feldman2015CertifyingAR}: Method that tunes model outputs with more favorable labels to minority groups (vice versa) in the low confidence region of classifiers to achieve better demographic parity. 

$\bullet$ \textsc{DIRemover}~\cite{Kamiran2012DecisionTF}:
Disparate impact remover is proposed as a pre-processing fairness promoting method, which modifies input features with rank-ordering preserving operations. We simply apply the method to modify model predictions as a post-processing method to promote demographic parity.

\noindent \textsc{EqOdds}, \textsc{CaliEqOdds} and \textsc{RejectOption} are trained on the tuning set and then applied on testing set while \textsc{DIRemover} directly tune the model predictions on the testing set. We use the implementation \citep{aif360-oct-2018} for all four baselines.

The results can be seen in Table~\ref{tab: overview}. We see that our method consistently outperforms these baselines with improved fairness-accuracy trade-off. 
For example, we see that \textsc{FairReprogram (Border)} can achieve -0.167 negative EO and 0.951 accuracy in \texttt{CelebA} with $\lambda=5.0$. By contrast, the best-performing post-processing baseline achieves 
much worse accuracy (0.927).
Similar comparisons can also be seen in \texttt{Civil Comments}, where the best post-processing baseline can achieve -0.031 negative EO score and 0.915 accuracy, while our method \textsc{FairReprogram (Soft)} can achieve -0.011 negative EO with 0.917 accuracy with $\lambda=50.0$.

\subsection{Experiments with Additional MMD Baselines}
\CR{
To partial out the instability of the adversarial training, we further compare our method with MMD method, where the adversarial loss $\mathcal{L}_{fair}$ in Eq.~\eqref{eq:loss} is replaced with the Maximum Mean Discrepancy regularization~\cite{Louizos2016THEVF}.
Specifically, we consider \textsc{MmdIn} and \textsc{MmdPost}, where model parameters are trained from scratch in an in-processing manner and fine-tuned in a post-processing manner, respectively, following the settings for adversarial training in Section~\ref{sec: setup}.
The experiment results on the \texttt{Civil Comments} dataset are presented in Fig.~\ref{fig: rebuttal_mmd}.
As we can see, our proposed method {\ours} outperforms the MMD baselines, which can alleviate the concern that fairness reprogramming has a better performance simply because of the instability of adversarial training of the baselines. }

\begin{figure}[thb]
\centerline{
\begin{tabular}{cccc}
    \includegraphics[width=.35\textwidth,height=!]{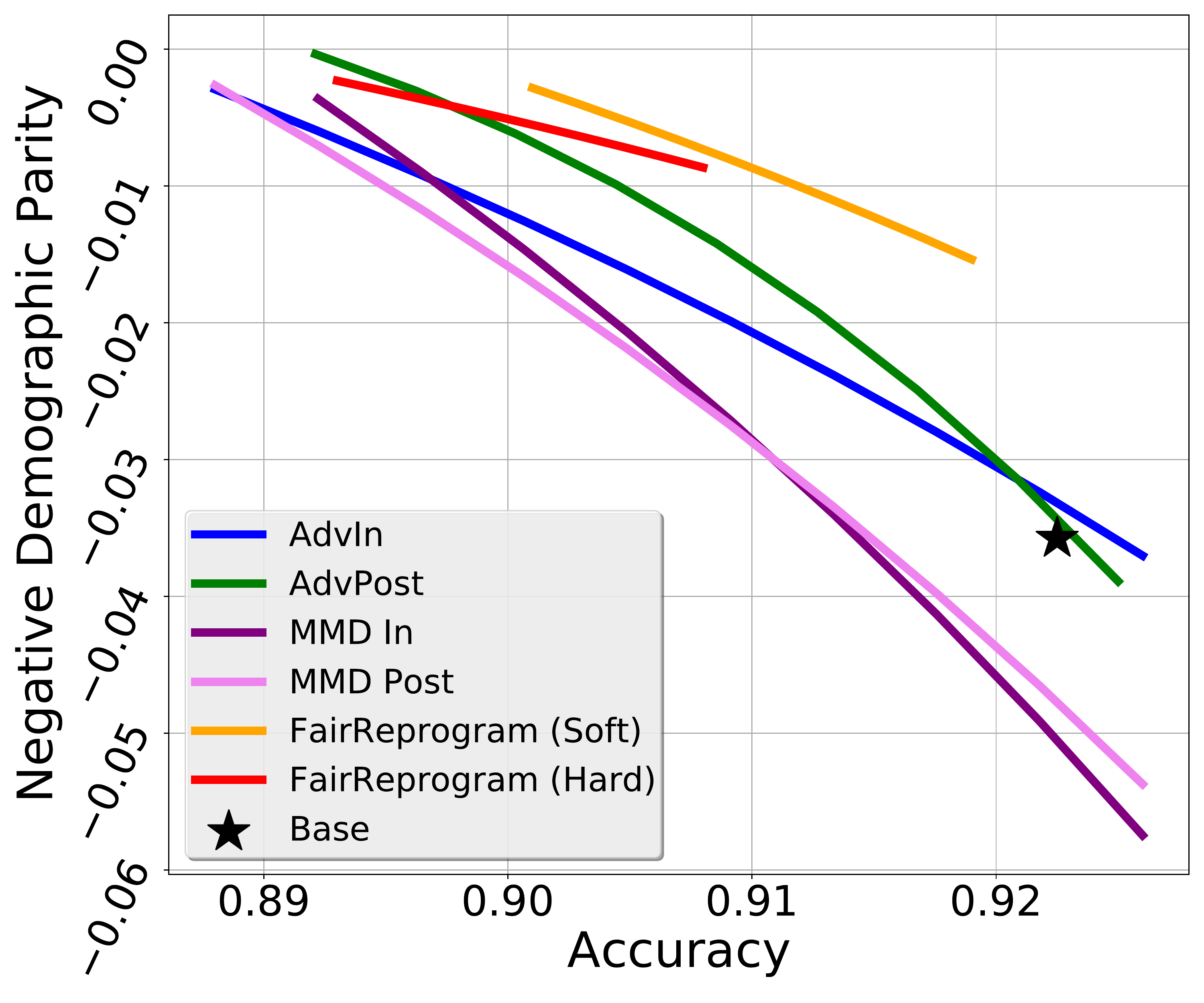}
    &
    \includegraphics[width=.35\textwidth,height=!]{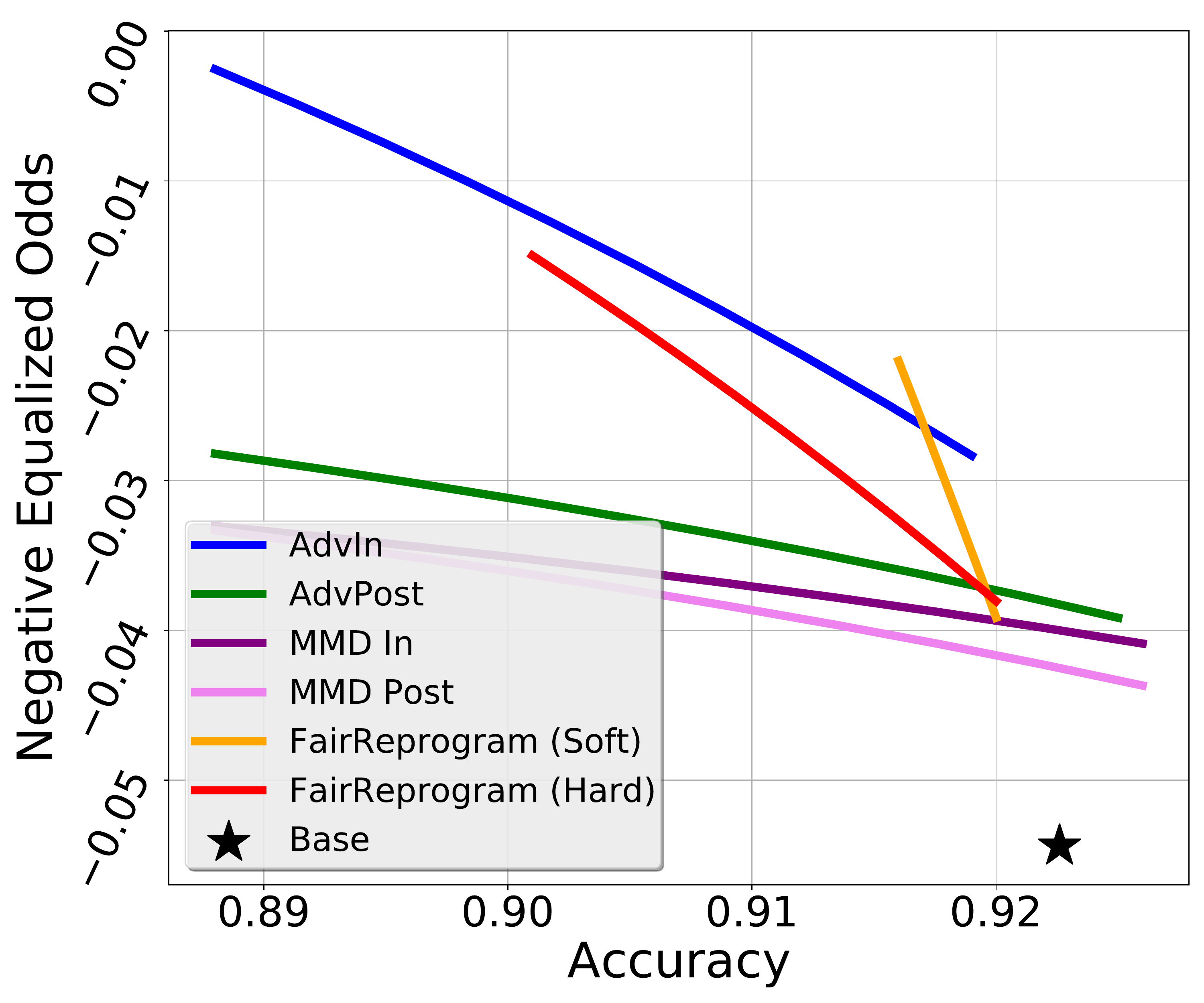} 
    \vspace*{-2mm}\\
    \footnotesize{(a)} 
    & 
    \footnotesize{(b)} 
    
\end{tabular}}
\vspace*{-3mm}
\caption{\footnotesize{
Results on \texttt{Civil Comments} with the new \textsc{MMD} baseline. We report the negative DP (left) and the negative EO (right) scores.  For each method, we vary the trade-off parameter $\lambda$ (as shown in \eqref{eq:loss}) to record the performance.  The closer a dot to the upper-right corner, the better the model is.   We consider four different $\lambda$s for each method. 
The solid curve is the fitted polynomial with order 30.
}}
\vspace*{-3mm}
\label{fig: rebuttal_mmd}
\end{figure}

\subsection{Black-box \textsc{FairReprogram} Generation} Previous experiments are all based on the white-box setting, which assumes access to the complete model information, such as model architectures and parameters. This precludes the use case of reprogramming a well-trained but access-limited model, \emph{e.g.}, a commercial APIs or other query-based software\,\cite{tsai2020transfer}. Thus, we further explore the feasibility of our method in the  black-box setup \cite{tsai2020transfer, zhang2022robustify},  where the gradients of the pre-trained model are estimated using only function queries. We follow the general black-box setting in \cite{tsai2020transfer} and adopt a query number of 30. 
The results are summarized in Fig.~\ref{fig: blackbox}.
As we can see, out algorithm can still improve the fairness without the knowledge of the model information. However, in such a case, the gain in fairness would sacrifice the accuracy largely when compared to our baselines. While in the future work, we will try to mitigate such degradation using more query numbers\cite{tsai2020transfer} and coordinate gradient estimation (CGE) \cite{zhang2022robustify} to achieve more accurate gradient estimation.

\begin{figure}[!htb]
\centerline{
\begin{tabular}{cc}
\hspace*{0mm}\includegraphics[width=.35\textwidth,height=!]{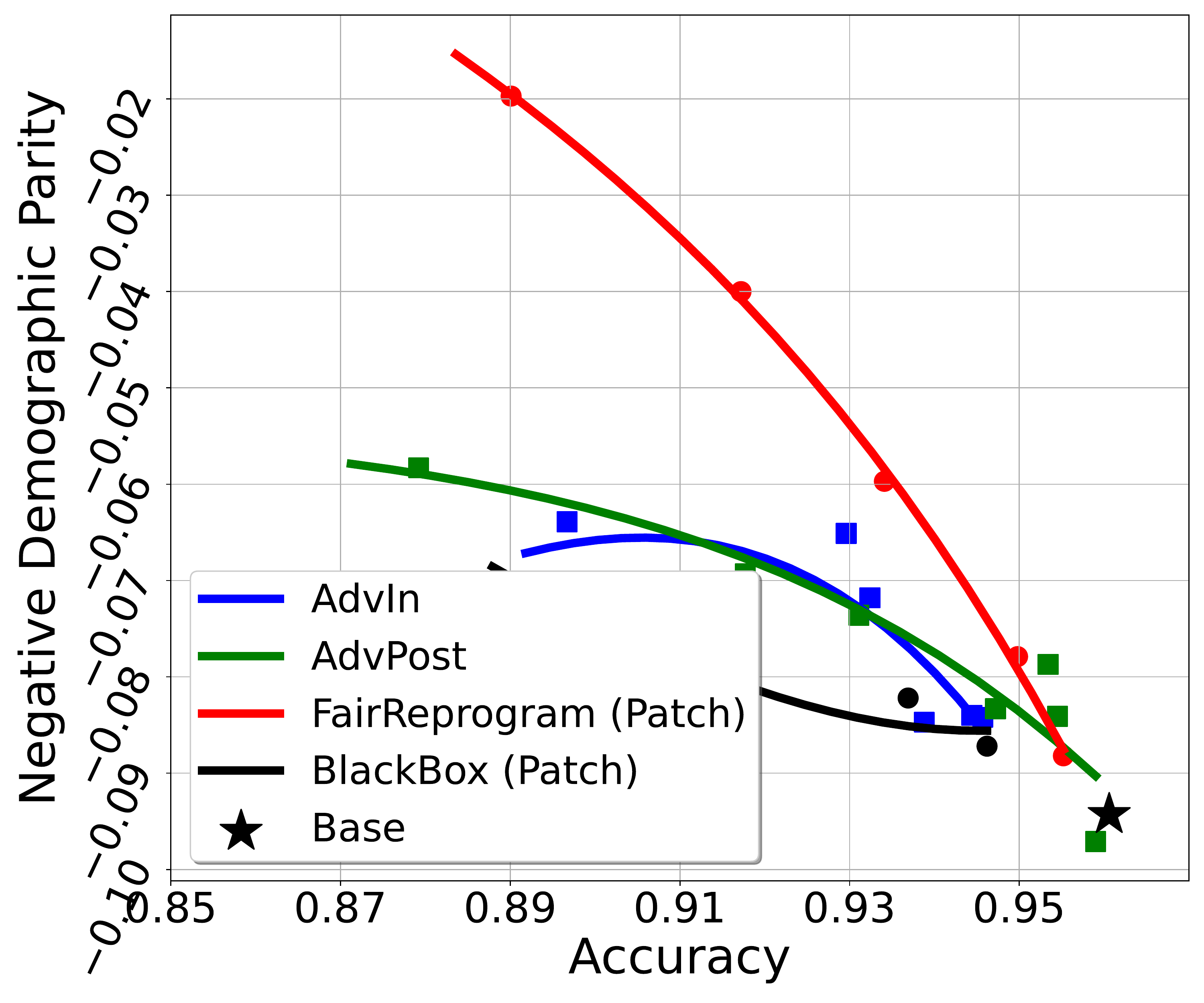}  
&\hspace*{-3mm}\includegraphics[width=.35\textwidth,height=!]{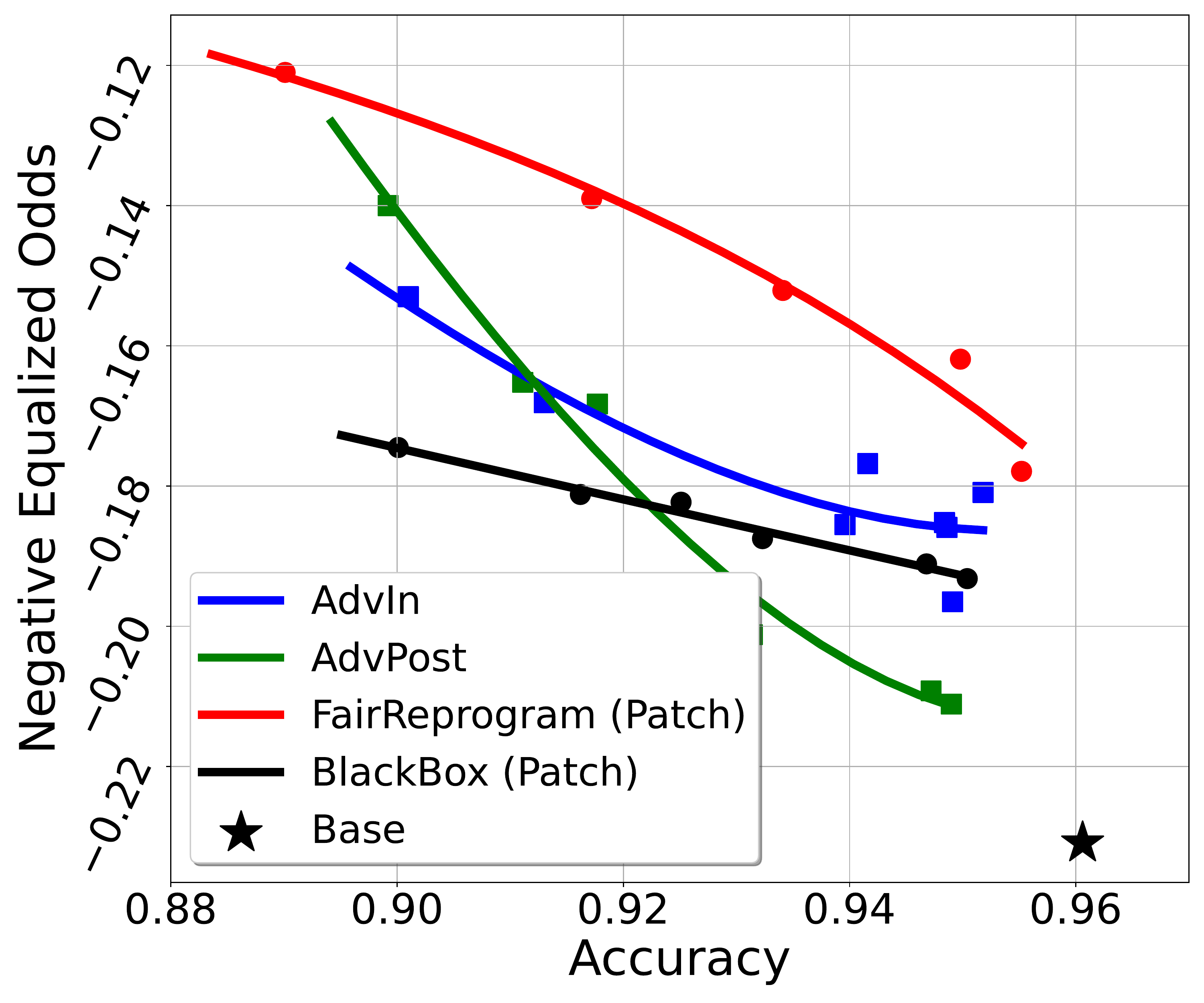}
\end{tabular}}
\caption{\footnotesize{Performance of {\ours} in the black-box setting. The left
Performance of the triggers trained in the black-box setting. Both the reprogrammer and the adversary are trained with query-based estimated gradients.
Different data samples represent different  }}
\label{fig: blackbox}
\end{figure}

\subsection{Results with Standard Derivation}
The numerical results in Figs.~\ref{fig: exp_overview}, \ref{fig: limit_data} and \ref{fig: size_study} with standard derivation are correspondingly presented in Tabs~\ref{tab: overview}, \ref{tab: limit_data} and \ref{tab: size_study}.

\begin{table*}[t]
    \caption{
    \footnotesize
    Numerical results with standard derivation on \texttt{Civil Comments} and \texttt{CelebA} with different tuning data ratio, corresponding to Fig.~\ref{fig: limit_data}. 
    All reported results are the average of three different random runs.
    We report the negative DP and the negative EO scores correspondingly for the ``Fairness'' column.
    We consider a fixed \textsc{Base} model trained with the training set, whose negative bias scores are presented as a black dashed line.
    Then we train other methods with different tuning data ratios to promote fairness of the \textsc{Base} model. 
    }  
    \label{tab: limit_data}
    \centering
    \resizebox{1.0\textwidth}{!}{
    \begin{tabular}{cc|cc|cc|cc|cc}
        \toprule[1pt]
        \midrule
        \multirow{3}{*}{Method} & \multirow{3}{*}{Tuning Data Ratio} &  \multicolumn{4}{c}{\texttt{Civil Comments}} & \multicolumn{4}{|c}{\texttt{CelebA}} \\
        && \multicolumn{2}{c|}{\texttt{Demographic parity}} & \multicolumn{2}{c|}{\texttt{Equalized odds}} & \multicolumn{2}{c|}{\texttt{Demographic parity}} & \multicolumn{2}{c}{\texttt{Equalized odds}} \\
        &&  Accuracy & Fairness &  Accuracy & Fairness  &  Accuracy & Fairness  &  Accuracy & Fairness\\
        \midrule
        \erm 
        & - & $0.922_{\pm0.004}$ & $-0.036_{\pm0.025}$ & $0.923_{\pm0.004}$ & $-0.054_{\pm0.005}$  
        & $0.961_{\pm0.004}$ & $-0.094_{\pm0.002}$ & $0.961_{\pm0.004}$ & $-0.231_{\pm0.008}$
        \\
        \midrule
        \multirow{6}{*}{\makecell[c]{\advp}} 
        & 1.0 
        & $0.909_{\pm0.011}$ & $-0.007_{\pm0.032}$  & $0.888_{\pm0.022}$ & $-0.000 _{\pm0.028} $
        & $0.908_{\pm0.005}$ & $-0.067_{\pm0.005}$ & $0.905_{\pm0.004}$ & $-0.150_{\pm0.007}$
         \\ 
        & 0.5 
        & $0.923_{\pm0.005}$ & $-0.044_{\pm0.009}  $&$ 0.919_{\pm0.005} $  & $-0.053_{\pm0.008}   $
        & $0.915_{\pm0.007}$ & $-0.069_{\pm0.007}$ & $0.911_{\pm0.006}$ & $-0.162_{\pm0.013}$
         \\ 
        & 0.2 
        & $0.923_{\pm0.004}$  & $-0.038_{\pm0.017}$  & $0.919_{\pm0.006} $ &$ -0.061_{\pm0.004} $
        & $0.939_{\pm0.002}$ & $-0.075_{\pm0.006}$ & $0.929_{\pm0.005}$ & $-0.171_{\pm0.010}$
         \\ 
        & 0.1 
        &$ 0.917_{\pm0.007} $&$ -0.038_{\pm0.015} $&$ 0.918_{\pm0.011} $&$ -0.056_{\pm0.013} $
        & $0.943_{\pm0.003}$ & $-0.083_{\pm0.004}$ & $0.933_{\pm0.007}$ & $-0.178_{\pm0.008}$
         \\ 
        & 0.01 
        & $0.922_{\pm0.002} $ & $-0.041_{\pm0.014}$ & $0.920_{\pm0.010}$  & $-0.060_{\pm0.014}$  
        & $0.948_{\pm0.005}$ & $-0.089_{\pm0.005}$ & $0.948_{\pm0.003}$ & $-0.202_{\pm0.012}$
         \\ 
        & 0.001 
        & $0.917_{\pm0.005}$ & $-0.083_{\pm0.018}$  & $0.921_{\pm0.006}$ & $-0.060_{\pm0.009}$  
        & $0.951_{\pm0.007}$ & $-0.091_{\pm0.005}$ & $0.955_{\pm0.002}$ & $-0.229_{\pm0.005}$
         \\ 
        \midrule
        \multirow{6}{*}{\makecell[c]{\ours \\ (\textsc{Soft}/\textsc{Border})}} 
        & 1.0 & $0.917_{\pm0.003}$  & $-0.002_{\pm0.001}$  & $0.917_{\pm0.004}$ & $-0.011_{\pm0.010} $
        & $0.935_{\pm0.003}$ & $-0.066_{\pm0.003}$ & $0.907_{\pm0.004}$ & $-0.153_{\pm0.009}$
         \\ 
        & 0.5 &$ 0.905_{\pm0.009} $&$ -0.002_{\pm0.004} $&$ 0.922_{\pm0.005} $&$ -0.018_{\pm0.013}  $
        & $0.941_{\pm0.003}$ & $-0.070_{\pm0.003}$ & $0.937_{\pm0.003}$ & $-0.162_{\pm0.008}$
         \\ 
        & 0.2 
        &$ 0.911_{\pm0.013}  $&$ -0.002_{\pm0.006} $&$ 0.917_{\pm0.008} $&$ -0.017_{\pm0.011} $
        & $0.947_{\pm0.003}$ & $-0.074_{\pm0.005}$ & $0.935_{\pm0.005}$ & $-0.162_{\pm0.011}$
         \\ 
        & 0.1 
        & $0.905_{\pm0.010}$&$ -0.001_{\pm0.005} $&$ 0.917_{\pm0.000} $&$ -0.025_{\pm0.007} $
        & $0.951_{\pm0.005}$ & $-0.079_{\pm0.004}$ & $0.951_{\pm0.006}$ & $-0.177_{\pm0.009}$
         \\ 
        & 0.01 
        & $ 0.911_{\pm0.007} $&$ -0.003_{\pm0.004} $&$ 0.918_{\pm 0.005} $&$ -0.033_{\pm0.017} $
        & $0.958_{\pm0.003}$ & $-0.087_{\pm0.002}$ & $0.959_{\pm0.003}$ & $-0.197_{\pm0.003}$
         \\ 
        & 0.001 
        & $ 0.908_{\pm0.176} $&$ -0.009_{\pm0.042} $&$ 0.921_{\pm0.181} $&$ -0.044_{\pm0.013} $
        & $0.957_{\pm0.008}$ & $-0.091_{\pm0.003}$ & $0.959_{\pm0.002}$ & $-0.221_{\pm0.008}$
         \\ 
        \midrule
        \multirow{6}{*}{\makecell[c]{\ours \\ (\textsc{Hard}/\textsc{Patch})}} 
        & 1.0 
        &$ 0.897_{\pm0.012} $&$ -0.005_{\pm0.004} $&$ 0.905_{\pm0.006} $&$ -0.009_{\pm0.007} $
        & $0.938_{\pm0.005}$ & $-0.065_{\pm0.014}$ & $0.931_{\pm0.002}$ & $-0.154_{\pm0.004}$
         \\ 
        & 0.5 
        & $0.905_{\pm0.014}$ & $-0.006_{\pm0.026}$ & $0.917_{\pm0.006}$ & $-0.028_{\pm0.007} $
        & $0.932_{\pm0.002}$ & $-0.062_{\pm0.002}$ & $0.937_{\pm0.004}$ & $-0.164_{\pm0.006}$
         \\ 
        & 0.2 
        &$ 0.902_{\pm0.013} $&$ -0.006_{\pm0.017} $&$ 0.909_{\pm 0.020} $&$ -0.025_{\pm0.020} $
        & $0.941_{\pm0.003}$ & $-0.073_{\pm0.005}$ & $0.945_{\pm0.005}$ & $-0.166_{\pm0.013}$
         \\ 
        & 0.1 
        &$ 0.900_{\pm0.014} $&$ -0.005_{\pm0.016} $&$ 0.909_{\pm 0.008} $&$ -0.024_{\pm0.011} $
        & $0.948_{\pm0.006}$ & $-0.079_{\pm0.003}$ & $0.951_{\pm0.002}$ & $-0.183_{\pm0.008}$
         \\ 
        & 0.01 
        &$ 0.896_{\pm0.013} $&$ -0.004_{\pm0.010} $&$ 0.918_{\pm 0.003} $&$ -0.035_{\pm0.005} $
        & $0.967_{\pm0.007}$ & $-0.087_{\pm0.004}$ & $0.955_{\pm0.004}$ & $-0.192_{\pm0.010}$
         \\ 
        & 0.001
        &$ 0.907_{\pm0.007} $&$ -0.008_{\pm0.012} $&$ 0.921_{\pm 0.000} $&$ -0.042_{\pm0.001} $
         & $0.955_{\pm0.004}$ & $-0.089_{\pm0.005}$ & $0.958_{\pm0.005}$ & $-0.228_{\pm0.013}$
         \\ 

        \midrule
        \bottomrule[1pt]
    \end{tabular}}
\end{table*}

\begin{table*}[t]
    \caption{
    \footnotesize
    Numerical results with standard derivation on \texttt{Civil Comments} and \texttt{CelebA} with different trigger size, corresponding to Fig.~\ref{fig: size_study}. 
    We evaluate the bias scores with different trigger word numbers (\texttt{Civil Comments}) and different trigger size (\texttt{CelebA}) with fixed adversary weight $\lambda$.
    All reported results are the average of three random runs.
    We report the negative DP and the negative EO scores correspondingly for the ``Fairness'' column.
    }  
    \label{tab: size_study}
    \centering
    \resizebox{1.0\textwidth}{!}{
    \begin{tabular}{cc|cc|cc|c|cc|cc}
        \toprule[1pt]
        \midrule
        \multirow{3}{*}{Method} & \multirow{3}{*}{Trigger Size} &  \multicolumn{4}{c|}{\texttt{Civil Comments}}& \multirow{3}{*}{Trigger Size} & \multicolumn{4}{c}{\texttt{CelebA}} \\
        && \multicolumn{2}{c|}{\texttt{Demographic parity}} & \multicolumn{2}{c|}{\texttt{Equalized odds}} &  &\multicolumn{2}{c|}{\texttt{Demographic parity}} & \multicolumn{2}{c}{\texttt{Equalized odds}} \\
        &&  Accuracy & Fairness &  Accuracy & Fairness &  &  Accuracy & Fairness  &  Accuracy & Fairness\\
        \midrule
        \erm 
        & - 
        &$ 0.922_{\pm 0.004}$ &$ -0.036_{\pm 0.025} $&$ 0.923_{\pm 0.004} $&$ -0.054_{\pm 0.005}$
        & -
        & $0.961_{\pm0.004}$ & $-0.094_{\pm 0.002}$ & $0.961_{\pm0.004}$ & $-0.231_{\pm0.008}$
        \\
        \midrule
        \multirow{6}{*}{\makecell[c]{\ours \\ (\textsc{Soft}/\textsc{Border})}} 
        & 20 
        & $0.890_{\pm 0.011}$ & $-0.001_{\pm 0.000}$ &$ 0.910_{\pm 0.005}$ &$ -0.004 _{\pm 0.001}$
        & 30
        & $0.914_{\pm0.002}$ & $-0.054_{\pm0.008}$ & $0.903_{\pm0.005}$ & $-0.155_{\pm0.005}$ 
         \\ 
        & 10 
        & $0.906_{\pm 0.004}$ &$ -0.002_{\pm 0.001}$ &$ 0.906_{\pm 0.004}$ &$ -0.009 _{\pm 0.010}$
        & 25
         & $0.933_{\pm0.003}$ &$ -0.061_{\pm0.006}$ & $0.917_{\pm0.003}$ & $-0.162_{\pm0.009}$ 
         \\ 
        & 5 
        & $0.917_{\pm 0.003}$ & $-0.002_{\pm 0.001}$ &$ 0.917_{\pm 0.004}$ &$ -0.011 _{\pm 0.010}$
        & 20
        & $0.939_{\pm0.006}$ & $-0.070_{\pm0.009}$ &  $0.923_{\pm0.007}$ & $-0.170_{\pm0.008}$ 
         \\ 
        & 2 
        &$ 0.912_{\pm 0.002}$ &$ -0.007_{\pm 0.001}$ &$ 0.921_{\pm 0.002}$ & $-0.038 _{\pm 0.009}$
        & 15
        & $0.943_{\pm0.004}$ & $-0.074_{\pm0.06}$ & $0.951_{\pm0.004}$ & $-0.189_{\pm0.011}$ 
         \\ 
        & 1 
        & $0.917_{\pm 0.001}$ &$ -0.011_{\pm 0.000}$ &$ 0.920_{\pm 0.000}$ &$ -0.051 _{\pm 0.002}$
        & 10
        & $0.951_{\pm0.004}$ & $-0.081_{\pm0.008}$ & $0.958_{\pm0.008}$ & $-0.222_{\pm0.012}$ 
         \\ 
        \midrule
        \multirow{6}{*}{\makecell[c]{\ours \\ (\textsc{Hard}/\textsc{Patch})}}
        & 20 
        & $0.890_{\pm 0.002}$ & $-0.001_{\pm 0.001}$ & $0.892 _{\pm 0.001}$& $-0.005 _{\pm 0.005}$
        & 70
        & $0.912_{\pm0.005}$ & $-0.048_{\pm0.005}$ & $0.932_{\pm0.003}$ & $-0.151_{\pm0.006}$ 
         \\ 
        & 10 &$ 0.891_{\pm 0.009}$ &$ -0.001_{\pm 0.003}$ &$ 0.901_{\pm 0.002}$ &$ -0.016 _{\pm 0.005}$
        & 60
        & $0.937_{\pm0.008}$ & $-0.056_{\pm0.004}$ & $0.947_{\pm0.004}$ & $-0.160_{\pm0.011}$ 
         \\ 
        & 5 &$ 0.897_{\pm 0.012}$ & $-0.005_{\pm 0.004}$ & $0.905 _{\pm 0.006}$ & $ -0.009 _{\pm 0.007}$
        & 50
        & $0.935_{\pm0.002}$ & $-0.061_{\pm0.008}$ & $0.954_{\pm0.004}$ & $-0.172_{\pm0.012}$ 
         \\ 
        & 2 &$ 0.905_{\pm 0.007}$ &$ -0.006 _{\pm 0.005}$&$ 0.911_{\pm 0.002}$ &$ -0.030 _{\pm 0.007}$
        & 40
        & $0.944_{\pm0.006}$ & $-0.074_{\pm0.010}$ & $0.959_{\pm0.008}$ & $-0.191_{\pm0.009}$ 
         \\ 
        & 1 &$ 0.913_{\pm 0.002}$ &$ -0.012_{\pm 0.001}$ & $0.921_{\pm 0.003}$ &$ -0.051 _{\pm 0.001}$
        & 30
        & $0.958_{\pm0.004}$ & $-0.091_{\pm0.007}$ & $0.958_{\pm0.002}$ & $-0.204_{\pm0.013}$ 
         \\ 
        \midrule
        \bottomrule[1pt]
    \end{tabular}}
\end{table*}

\CR{
\subsection{Transfer experiments with different tasks and model architectures}
\label{app:transfer}
We further test the transferability of reprogramming to different tasks and model architectures as shown in Fig.~\ref{fig: rebuttal_transfer}. 
Specifically, for transfer setting, 
the reprogram is optimized on the (ResNet-18, \texttt{CelebA}) with the task of predicting the hair color, and evaluated on (ResNet-20, \texttt{CelebA}) with the task of predicting smiling.
For both tasks, the attribute gender is chosen as the demographic information throughout the experiments.
We can see that the trigger still has good transferability with different model architectures.
Meanwhile, we find that the triggers are able to boost the fairness of the model in the task-transfer setting, but the accuracy is traded off more than the original setting. 
}

\begin{figure}[thb]
\centerline{
\begin{tabular}{cccc}
    \includegraphics[width=.35\textwidth,height=!]{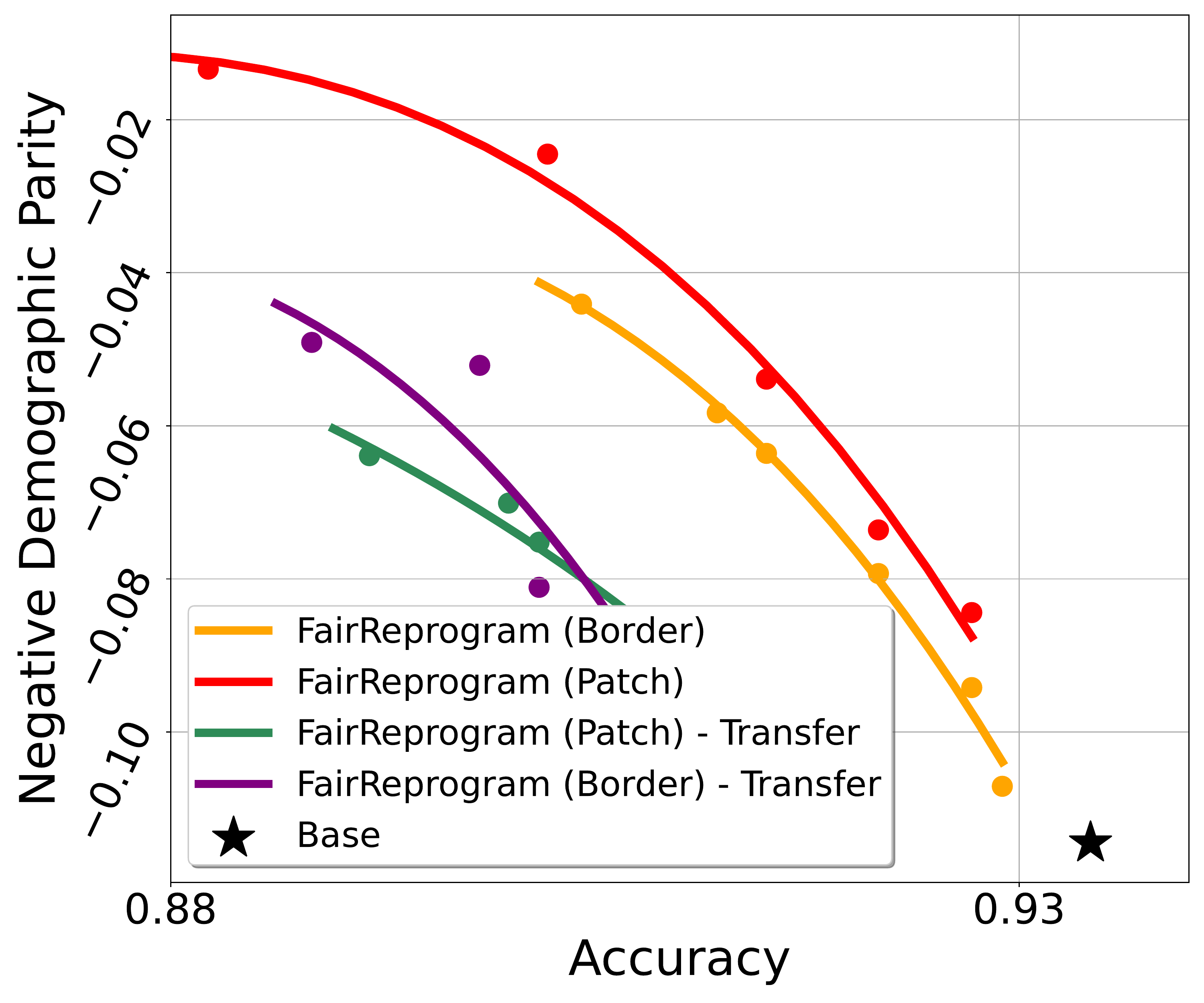} 
    &
    \includegraphics[width=.35\textwidth,height=!]{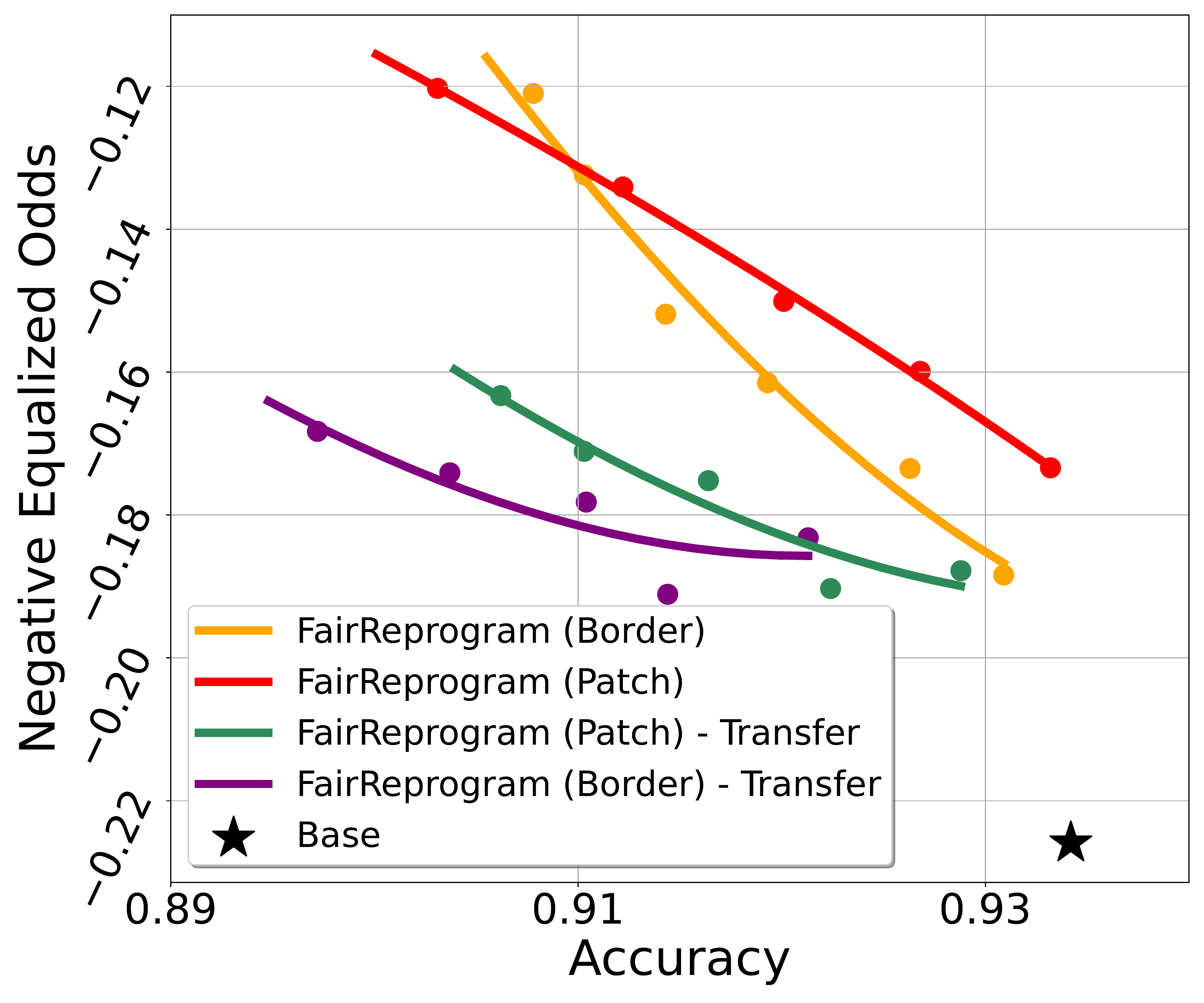} 
    \vspace*{-2mm}\\
    \footnotesize{(a)} 
    & 
    \footnotesize{(b)} 
    
\end{tabular}}
\vspace*{-3mm}
\caption{\footnotesize{Results of the transferability experiment on \texttt{CelebA} dataset with different tasks and model architectures. In each figure, we compare the reprogramming in the transferred setting (curves denoted with `transfer') with the reprogram directly trained on the target task. For transfer setting, the reprogram is optimized on the (ResNet-18, \texttt{CelebA}) with the task of predicting the hair color, and evaluated on (ResNet-20, \texttt{CelebA}) with the task of predicting smiling. For both tasks, the attribute gender is chosen as the demographic information throughout the experiments.}}
\vspace*{-3mm}
\label{fig: rebuttal_transfer}
\end{figure}

\CR{
\subsection{Experiments on Reprogramming Tabular Data}
We show that {\ours} could also be applied to tabular data.
For reprogramming, there are many ways to design triggers according to different tasks and requirements. 
Unlike NLP, where we append the trigger to the input or embeddings, the model for tabular data is sensitive to the input size.
As the tabular data have a fixed input size, we can directly apply the additive trigger to the input data to keep the input dimension unchanged (i.e., adding a perturbation on the original input), just as we adopted in image domains in Fig.~\ref{fig: trigger}.b. 
To verify our argument, we conducted additional experiments on the UCI Adult dataset~\cite{Asuncion2007UCIML} with a two-layer MLP model, and the results are shown in Fig.~\ref{fig: rebuttal_tabular}. 
Our method achieves comparable debiasing performance with the post-processing adversarial training method without modifying any model parameters.
The results suggest that our method could effectively improve model fairness for tabular data.
}

\begin{figure}[thb]
\centerline{
\begin{tabular}{cccc}
    \includegraphics[width=.35\textwidth,height=!]{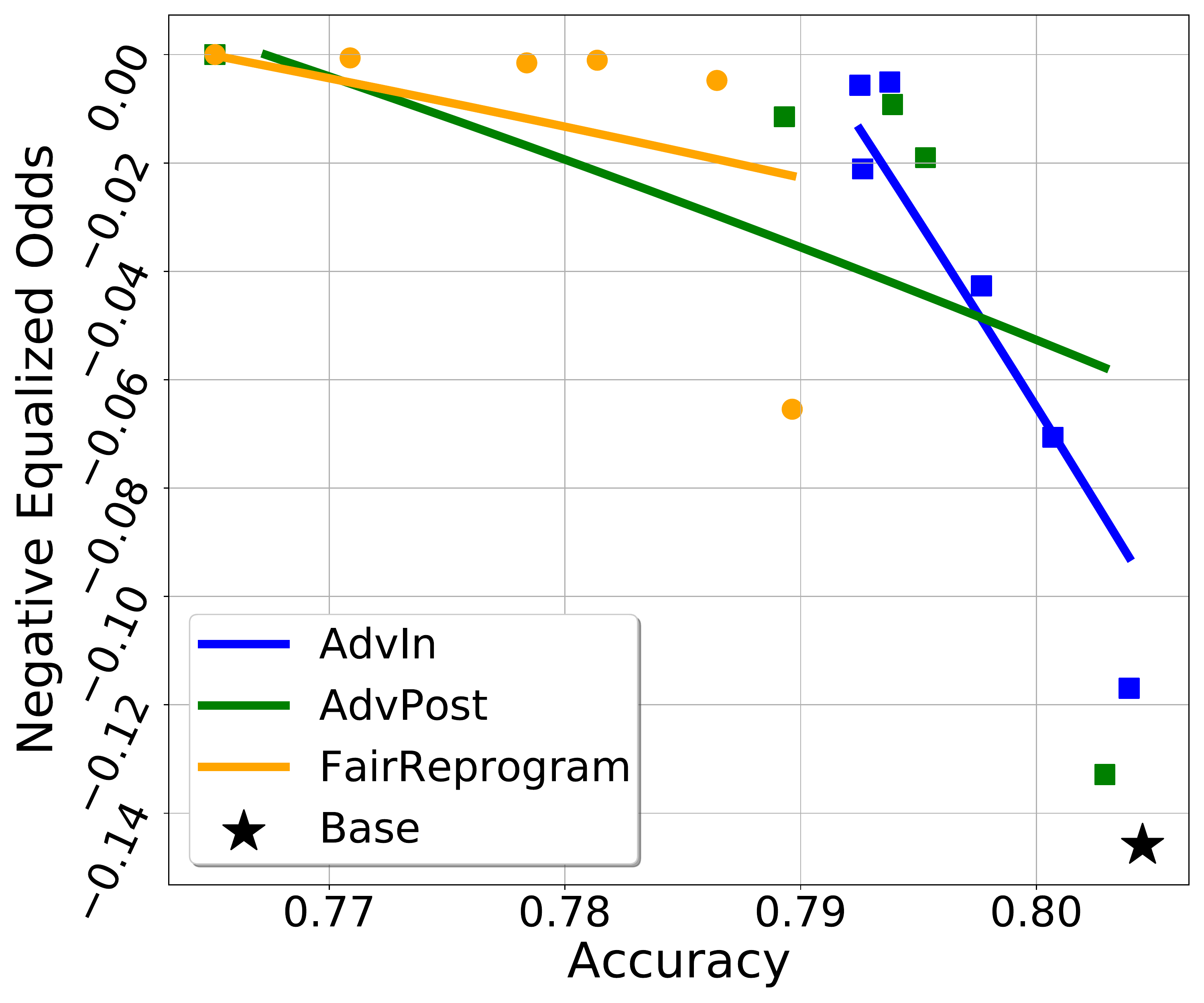}
    &
    \includegraphics[width=.35\textwidth,height=!]{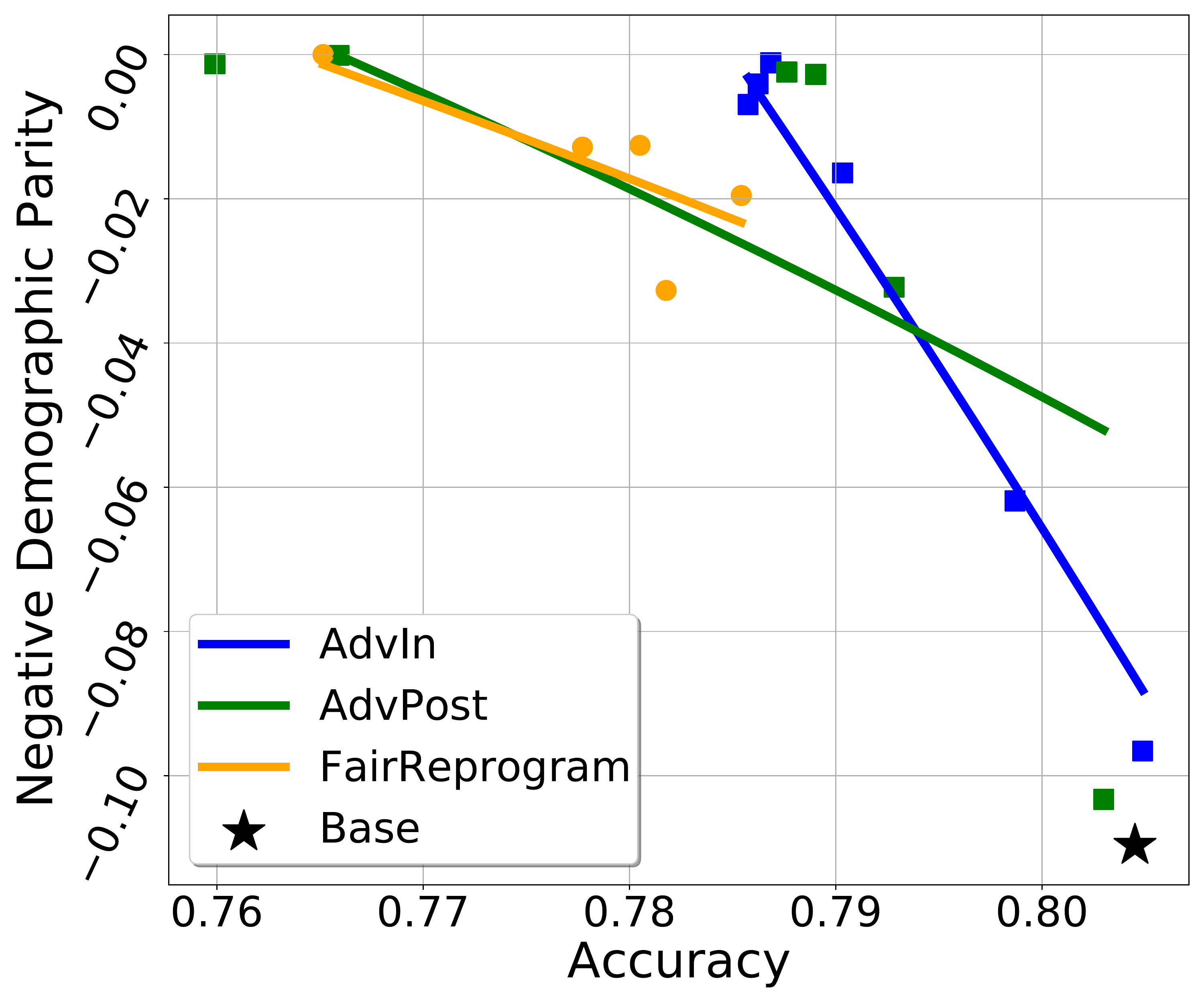}
    \vspace*{-2mm}\\
    \footnotesize{(a)} 
    & 
    \footnotesize{(b)} 
    
\end{tabular}}
\vspace*{-3mm}
\caption{\footnotesize{
Results on \texttt{Adult}. We report the negative DP (left) and the negative EO (right) scores.  For each method, we vary the trade-off parameter $\lambda$ (as shown in \eqref{eq:loss}) to record the performance.  The closer a dot to the upper-right corner, the better the model is.   We consider six different $\lambda$s for each method. 
The solid curve is the fitted polynomial with order 30.
}}
\vspace*{-3mm}
\label{fig: rebuttal_tabular}
\end{figure}

\newpage
\section{Theoretical Proofs}
\label{app: proof}

In this section, we will provide formal proofs to the claims and theorem in the main paper.

\subsection{Sufficient Statistics}

We will show that \e{p_{Y}(\cdot | \bm X^{(y)})} and \e{p_{Z}(\cdot | \bm X^{(z)})} are the sufficient statistics of \e{\bm X^{(y)}} and \e{\bm X^{(z)}} respectively for inferring \e{Y}. Formally, what we need to show is
\begin{equation}
    \small
    p(Y | \bm X^{(y)}) = p(Y | p_{Y}(\cdot | \bm X^{(y)}))
    \label{eq:suff_proof_1}
\end{equation}
and
\begin{equation}
\small
    p(Y | \bm X^{(z)}) = p(Y | p_{Z}(\cdot | \bm X^{(z)}))
    \label{eq:suff_proof_2}
\end{equation}
Eq.~\eqref{eq:suff_proof_1} is an identity. To show Eq.~\eqref{eq:suff_proof_2}:
{\small
\begin{equation*}
    \begin{aligned}
    p(Y | \bm X^{(z)}) &= \mathbb{E}_{Z \sim p_{Z}(\cdot | \bm X^{(z)})}[p(Y | Z, \bm X^{(z)})] \\
    &= \mathbb{E}_{Z \sim p_{Z}(\cdot | \bm X^{(z)})}[p(Y | Z, p_{Z}(\cdot | \bm X^{(z)}))] \\
    &= \mathbb{E}_{Z \sim p_{Z}(\cdot | p_{Z}(\cdot | \bm X^{(z)}))}[p(Y | Z, p_{Z}(\cdot | \bm X^{(z)}))] \\
    &= p(Y | p_{Z}(\cdot | \bm X^{(z)})).
    \end{aligned}
\end{equation*}
}
The second equality is because \e{Y} and \e{\bm X^{(z)}} are independent conditional on \e{Z}, so replacing \e{\bm X^{(z)}} with any functions of \e{\bm X^{(z)}} would not change the conditional probability. The third equality is implied from the identity \e{p(Z | \bm X^{(z)}) = p(Z | p_{Z}(\cdot | \bm X^{(z)}))}.

Using the sufficient statistics, it is very easy to show that \e{p(Y | \bm X)} is a special case of Eq.~\eqref{eq:classifier_assump}:
{\small
\begin{equation*}
  p(Y | \bm X) = p(Y | \bm X^{(y)}, \bm X^{(z)}) =  p(Y | p_{Y}(\cdot | \bm X^{(y)}), p_{Z}(\cdot | \bm X^{(z)})).
\end{equation*}
}%

\subsection{Proof to Thm.~\ref{thm:main}}

We first provide the regularity conditions as stated in Thm.~\ref{thm:main}.
\begin{enumerate}
    \item \textit{Conditional Independence.} The features in \e{\bm X^{(y)}} and \e{\bm X^{(z)}} are independent and identically distributed conditional on \e{Y} and \e{Z} respectively.
    \begin{equation}
        \small
        p^{tr}(\bm X^{(y)} | Y) = \prod_t p^{tr}(\bm X^{(y)}_t | Y), \quad p^{tr}(\bm X^{(z)} | Z) = \prod_t p^{tr}(\bm X^{(z)}_t | Z).
    \end{equation}
    \label{assump:iid}
    
    \item \emph{Infrequent Strong Demographic Features.} The probability of occurrence of features that are very strongly indicative against a certain demographic group is low. Formally \e{\forall z}, \e{\forall \varepsilon > 0}, \e{\exists \sigma > 0}, such that define
\begin{equation}
\small
    \mathcal{S}(\sigma) = \{\bm x^{(z)} \in \mathcal{X}^{(z)} :  p(Z=z | \bm X^{(z)} = \bm x^{(z)})  \leq \sigma\},
    \label{eq:S_def}
\end{equation}
we have
\begin{equation}
        \small
        p(\bm X^{(z)} \in \mathcal{S}(\sigma)) \leq \varepsilon.
        \label{eq:infrequent}
\end{equation}
    \label{assump:infrequent}
    
    \item \emph{Continuous Classifier.} \e{h(\cdot, \cdot)} is continuous with respect to both arguments.
\end{enumerate}

With these assumptions, we will state the following lemma.
\begin{lemma}
Consider the case where \e{Z} takes on \e{K} different values, \emph{i.e.}, there are \e{K} demographic groups. Then
\begin{equation}
\small
    \lim_{p^{tr}(Z=z|\bm X_0^{(z)}=\bm \delta) \rightarrow 1} H(Z |  h(p^{tr}_Y(\cdot | \bm X^{(y)}), p^{tr}_Z(\cdot | \bm X^{(z)} = \tilde{\bm X}^{(z)})), Y) = H(Z |  h(p^{tr}_Y(\cdot | \bm X^{(y)}), c), Y),
\end{equation}
where \e{c} is a {K}-dimensional one-hot vector with the \e{z}-th dimension equal to 1 and 0 elsewhere.
\label{lemma:main}
\end{lemma}
\begin{proof}
According to Assumption~\ref{assump:infrequent} (Eq.~\eqref{eq:infrequent}),
{\small
\begin{equation}
    \forall \varepsilon > 0, \quad \exists 0 < \sigma < 1, \quad p(\bm X^{(z)} \in \mathcal{S}(\sigma)) \leq \frac{\varepsilon}{4H(Z)},
    \label{eq:infrequent2}
\end{equation}
}
where \e{\mathcal{S}(\sigma)} is defined in Eq.~\eqref{eq:S_def}.
On the other hand, consider the following composite function
\begin{equation}
    \small
     H(Z |  h(p^{tr}_Y(\cdot | \bm X^{(y)}), p^{tr}_Z(\cdot | \bm X^{(z)} = [\bm \delta, \bm x^{(z)}])), Y).
     \label{eq:cond_entropy_fixed}
\end{equation}
Note that this is \emph{different} from \e{H(Z |  h(p^{tr}_Y(\cdot | \bm X^{(y)}), p^{tr}_Z(\cdot | \bm X^{(z)} = \tilde{\bm X}^{(z)})), Y)}, which is essentially the expectation of Eq.~\eqref{eq:cond_entropy_fixed} over different values of \e{\bm x^{(z)}}.

Since the conditional entropy is continuous and bounded, and \e{h(\cdot, \cdot)} is continuous over both of its arguments with finite support, Eq.~\eqref{eq:cond_entropy_fixed} is \emph{uniformly} continuous with respect to \e{p^{tr}_Z(\cdot | \bm X^{(z)}))}. Therefore, given the aforementioned \e{\varepsilon},
\begin{equation}
\small
\begin{aligned}
    & \exists 0 < \eta < 1, \quad \forall \bm \delta,  \bm x^{(z)} \mbox{ s.t. } \Vert p_Z(\cdot | \bm X^{(z)} = [\bm \delta, \bm x^{(z)}]) - c \Vert_1  \leq \eta \\
    & \Rightarrow \Big\vert H(Z |  h(p^{tr}_Y(\cdot | \bm X^{(y)}), p^{tr}_Z(\cdot | \bm X^{(z)} = [\bm \delta, \bm x^{(z)}])), Y) - H(Z |  h(p^{tr}_Y(\cdot | \bm X^{(y)} ), c), Y) \Big\vert \leq \frac{\varepsilon}{2}.
\end{aligned}
\label{eq:composite_continuity}
\end{equation}
Now, divide the support of \e{\bm X^{(z)}} into two disjoint sets. For notational conciseness, define
\begin{equation}
    \small
    r(\bm x^{(z)}) = \max_{z' \neq z} \frac{p^{tr}(Z=z')p^{tr}(\bm X^{(z)} = \bm x^{(z)} | Z = z')}{p^{tr}(Z=z)p^{tr}(\bm X^{(z)} = \bm x^{(z)} | Z = z)}.
    \label{eq:rz}
\end{equation}
Then the two sets, denoted as \e{\mathcal{A}} and \e{\mathcal{B}} respectively, are divided according to whether \e{r(\bm x^{(z)})} exceeds a threshold, \emph{i.e.}
\begin{equation}
    \small
    \mathcal{A} = \{\bm x^{(z)} \in \mathcal{X}^{(z)}: r(\bm x^{(z)}) \leq \sigma^{-1} -1 \}, \quad 
    \mathcal{B} = \{\bm x^{(z)} \in \mathcal{X}^{(z)}: r(\bm x^{(z)}) > \sigma^{-1} -1 \}.
    \label{eq:set_ab}
\end{equation}
\e{\forall \bm x^{(z)} \in \mathcal{A}}, define
\begin{equation}
\small
    \zeta = 1 - \left[ \frac{(1-\eta/2)^{-1}-1}{(K-1)(\sigma^{-1}-1)G} + 1 \right]^{-1}, \quad \mbox{where } G = \max_{z' \neq z}\frac{p(Z = z)}{p(Z = z')}.
    \label{eq:zeta_def}
\end{equation}
Then we will show that
\begin{equation}
\small
    \forall \bm x^{(z)} \in \mathcal{A}, \forall \bm \delta  \mbox{ s.t. } p^{tr}(Z=z|\bm X_0^{(z)}=\bm \delta) \geq 1 - \zeta \quad \Rightarrow \quad \Vert p^{tr}_Z(\cdot | \bm X^{(z)} = [\bm \delta, \bm x^{(z)}]) - c \Vert_1  \leq \eta,
    \label{eq:bound_delta}
\end{equation}
and hence Eq.~\eqref{eq:composite_continuity} holds. This is because, according to the Bayesian rule,
\begin{equation}
    \small
    p^{tr}(Z=z|\bm X_0^{(z)}=\bm \delta) = \frac{p^{tr}(Z = z)p^{tr}(\bm X_0^{(z)}=\bm \delta | Z = z)}{\sum_{z'\neq z} p^{tr}(Z = z')p^{tr}(\bm X_0^{(z)}=\bm \delta | Z = z')}.
    \label{eq:bayes}
\end{equation}
Therefore
\begin{equation}
    \small
    \begin{aligned}
        p^{tr}(Z=z|\bm X_0^{(z)}=\bm \delta) \geq 1 - \zeta \quad &\Rightarrow \quad 1 + \sum_{z' \neq z} \frac{p^{tr}(Z = z')p^{tr}(\bm X_0^{(z)}=\bm \delta | Z = z')}{p^{tr}(Z = z)p^{tr}(\bm X_0^{(z)}=\bm \delta | Z = z)} \leq (1 - \zeta)^{-1} \\
        & \Rightarrow \quad \frac{p^{tr}(Z = z')p^{tr}(\bm X_0^{(z)}=\bm \delta | Z = z')}{p^{tr}(Z = z)p^{tr}(\bm X_0^{(z)}=\bm \delta | Z = z)} \leq (1 - \zeta)^{-1} - 1, \forall z' \neq z \\
        & \Rightarrow \quad \frac{p^{tr}(\bm X_0^{(z)}=\bm \delta | Z = z')}{p^{tr}(\bm X_0^{(z)}=\bm \delta | Z = z)} \leq G[(1 - \zeta)^{-1} - 1], \forall z' \neq z.
    \end{aligned}
    \label{eq:ratio_bound}
\end{equation}
As a result,
\begin{equation}
    \small
    \begin{aligned}
        p^{tr}(Z = z | \bm X^{(z)} = [\bm \delta, \bm x^{(z)}]) ^{-1} &= 1 + \sum_{z' \neq z} \frac{p^{tr}(Z = z')p^{tr}(\bm X^{(z)}=\bm x^{(z)} | Z = z')p^{tr}(\bm X_0^{(z)}=\bm \delta | Z = z')}{p^{tr}(Z = z)p^{tr}(\bm X^{(z)}=\bm x^{(z)} | Z = z)p^{tr}(\bm X_0^{(z)}=\bm \delta | Z = z)} \\
        & \leq 1 + r(\bm x^{z}) \sum_{z' \neq z} \frac{p^{tr}(\bm X_0^{(z)}=\bm \delta | Z = z')}{p^{tr}(\bm X_0^{(z)}=\bm \delta | Z = z)} \\
        & \leq 1 + (\sigma^{-1}-1) \sum_{z' \neq z} \frac{p^{tr}(\bm X_0^{(z)}=\bm \delta | Z = z')}{p^{tr}(\bm X_0^{(z)}=\bm \delta | Z = z)} \\
        & \leq 1 + (\sigma^{-1}-1)(K-1)G[(1 - \zeta)^{-1} - 1] = (1 - \eta / 2)^{-1},
    \end{aligned}
    \label{eq:bound_delta_x}
\end{equation}
where the first line is implied from the Bayesian rule and assumption~\ref{assump:iid} (similar to Eq.~\eqref{eq:bayes}); the second line is implied from the definition of \e{r(\bm x^{z})} as in Eq.~\eqref{eq:rz}; the third line is due to the definition of set \e{\mathcal{A}} as in Eq.~\eqref{eq:set_ab} (note that the scope of Eq.~\eqref{eq:bound_delta} is confined to \e{ \forall \bm x^{(z)} \in \mathcal{A}}); the last line is implied from Eq.~\eqref{eq:ratio_bound} and the definition of \e{\zeta} as in Eq.~\eqref{eq:zeta_def}.

It then follows that
\begin{equation}
\small
    \begin{aligned}
    \Vert p^{tr}_Z(\cdot | \bm X^{(z)} = [\bm \delta, \bm x^{(z)}]) - c \Vert_1 &= 1 - p^{tr}(Z=z | \bm X^{(z)}) + \sum_{z' \neq ' } p^{tr}(Z=z' | \bm X^{(z)}) \\
    &= 2(1 - p^{tr}(Z=z | \bm X^{(z)})) \\
    & \leq \eta,
    \end{aligned}
\end{equation}
where the first line is implied from the definition of the one-hot vector \e{c} as well as the fact that the probability mass function is alwasy between 0 and 1; the second line is given by the fact that any probability mass functions sum to 1; and the last line is given by Eq.~\eqref{eq:bound_delta_x}. This concludes the proof to Eq.~\eqref{eq:bound_delta}.

Next, notice that
\begin{equation}
    \small
    \begin{aligned}
    & H(Z |  h(p^{tr}_Y(\cdot | \bm X^{(y)}), p^{tr}_Z(\cdot | \bm X^{(z)} = \tilde{\bm X}^{(z)})), Y) \\
    = & \sum_{\bm x^{(z)} \in \mathcal{X}^{(z)}} H(Z |  h(p^{tr}_Y(\cdot | \bm X^{(y)}), p^{tr}_Z(\cdot | \bm X^{(z)} = [\bm \delta, \bm x^{(z)}])), Y) p(\bm X^{(z)} = \bm x^{(z)}) \\
    = & \sum_{\bm x^{(z)} \in \mathcal{A}} H(Z |  h(p^{tr}_Y(\cdot | \bm X^{(y)}), p^{tr}_Z(\cdot | \bm X^{(z)} = [\bm \delta, \bm x^{(z)}])), Y) p(\bm X^{(z)} = \bm x^{(z)}) \\
    & + \sum_{\bm x^{(z)} \in \mathcal{B}} H(Z |  h(p^{tr}_Y(\cdot | \bm X^{(y)}), p^{tr}_Z(\cdot | \bm X^{(z)} = [\bm \delta, \bm x^{(z)}])), Y) p(\bm X^{(z)} = \bm x^{(z)}).
    \end{aligned}
\end{equation}
Thus
\begin{equation}
    \small
    \begin{aligned}
    &\Big\vert H(Z |  h(p^{tr}_Y(\cdot | \bm X^{(y)}), p^{tr}_Z(\cdot | \bm X^{(z)} - \tilde{\bm X}^{(z)})), Y) - H(Z |  h(p^{tr}_Y(\cdot | \bm X^{(y)}), c), Y) \Big\vert \\
    =&  \sum_{\bm x^{(z)} \in \mathcal{A}} \Big\vert H(Z |  h(p^{tr}_Y(\cdot | \bm X^{(y)}), p^{tr}_Z(\cdot | \bm X^{(z)} = [\bm \delta, \bm x^{(z)}])), Y) - H(Z |  h(p^{tr}_Y(\cdot | \bm X^{(y)} ), c), Y) \Big\vert p(\bm X^{(z)} = \bm x^{(z)}) \\
    +&  \sum_{\bm x^{(z)} \in \mathcal{B}} \Big\vert H(Z |  h(p^{tr}_Y(\cdot | \bm X^{(y)}), p^{tr}_Z(\cdot | \bm X^{(z)} = [\bm \delta, \bm x^{(z)}])), Y) - H(Z |  h(p^{tr}_Y(\cdot | \bm X^{(y)} ), c), Y) \Big\vert p(\bm X^{(z)} = \bm x^{(z)})
    \end{aligned}
    \label{eq:decompose}
\end{equation}
In the following, we will bound the two terms respectively. For the first term in Eq.~\eqref{eq:decompose}, Eq.~\eqref{eq:bound_delta} applies because \e{\bm x^{(z)} \in \mathcal{A}}. Therefore, according to Eq.~\eqref{eq:bound_delta} and \eqref{eq:composite_continuity}, when \e{p^{tr}(Z=z|\bm X_0^{(z)}=\bm \delta) \geq 1 - \zeta}, we have
\begin{equation}
    \small
    \begin{aligned}
    & \sum_{\bm x^{(z)} \in \mathcal{A}} \Big\vert H(Z |  h(p^{tr}_Y(\cdot | \bm X^{(y)}), p^{tr}_Z(\cdot | \bm X^{(z)} = [\bm \delta, \bm x^{(z)}])), Y) - H(Z |  h(p^{tr}_Y(\cdot | \bm X^{(y)} ), c), Y) \Big\vert p(\bm X^{(z)} = \bm x^{(z)}) \\
    \leq & \sum_{\bm x^{(z)} \in \mathcal{A}} \frac{\varepsilon}{2} p(\bm X^{(z)} = \bm x^{(z)}) \leq \frac{\varepsilon}{2}.
    \end{aligned}
    \label{eq:bound_set_a}
\end{equation}
For the second term, notice that when \e{\bm x^{(z)} \in \mathcal{B}}, \e{r(\bm x^{(z)}) > \sigma^{-1}-1} (according to Eq.~\eqref{eq:set_ab}). So it follows that
\begin{equation}
    \small
    \begin{aligned}
     p(Z=z | \bm X^{(z)} = \bm x^{(z)})^{(-1)} &= 1 + \sum_{z' \neq z} \frac{p^{tr}(Z = z')p^{tr}(\bm X_0^{(z)}=\bm \delta | Z = z')}{p^{tr}(Z = z)p^{tr}(\bm X_0^{(z)}=\bm \delta | Z = z)} \\
     & \geq 1 + r(\bm x^{(z)}) \geq \sigma^{-1}
     \end{aligned}
\end{equation}
According to \eqref{eq:infrequent2}, this implies
\begin{equation}
    \small
    p(\bm X^{(z)} \in \mathcal{B}) \leq p(\bm X^{(z)} \in \mathcal{S}(\sigma)) \leq \frac{\varepsilon}{4H(Z)},
\end{equation}
and further
\begin{equation}
    \small
    \begin{aligned}
     &\sum_{\bm x^{(z)} \in \mathcal{B}} \Big\vert H(Z |  h(p^{tr}_Y(\cdot | \bm X^{(y)}), p^{tr}_Z(\cdot | \bm X^{(z)} = [\bm \delta, \bm x^{(z)}])), Y) - H(Z |  h(p^{tr}_Y(\cdot | \bm X^{(y)} ), c), Y) \Big\vert p(\bm X^{(z)} = \bm x^{(z)}) \\
     \leq& \Big[H(Z |  h(p^{tr}_Y(\cdot | \bm X^{(y)}), p^{tr}_Z(\cdot | \bm X^{(z)} = [\bm \delta, \bm x^{(z)}])), Y) + H(Z |  h(p^{tr}_Y(\cdot | \bm X^{(y)} ), c), Y)\Big] p(\bm X^{(z)} = \bm x^{(z)}) \\
     \leq& 2H(Z) p(\bm X^{(z)} = \bm x^{(z)}) \leq \frac{\varepsilon}{2}.
    \end{aligned}
    \label{eq:bound_set_b}
\end{equation}
Plugging Eqs.~\eqref{eq:bound_set_a} and \eqref{eq:bound_set_b} into Eq.~\eqref{eq:decompose}, we can finally establish that \e{\forall \varepsilon > 0}, \e{\exists \zeta > 0} (one possible \e{\zeta} as defined in Eq.~\eqref{eq:zeta_def}), when \e{p^{tr}(Z=z|\bm X_0^{(z)}=\bm \delta) \geq 1 - \zeta}, we have
\begin{equation}
    \small
    \Big\vert H(Z |  h(p^{tr}_Y(\cdot | \bm X^{(y)}), p^{tr}_Z(\cdot | \bm X^{(z)} - \tilde{\bm X}^{(z)})), Y) - H(Z |  h(p^{tr}_Y(\cdot | \bm X^{(y)}), c), Y) \Big\vert \leq \varepsilon.
\end{equation}
Hence this concludes the proof to Lemma~\ref{lemma:main}.
\end{proof}
With Lemma~\ref{lemma:main}, we are ready to prove Thm~\ref{thm:main}.
\begin{proof}
Note that
\begin{equation}
    \small
    \begin{aligned}
     H(Z |  h(p^{tr}_Y(\cdot | \bm X^{(y)}), c), Y) \geq H(Z |  p^{tr}_Y(\cdot | \bm X^{(y)}), Y)  = H(Z | Y).
    \end{aligned}
    \label{eq:entropy_ineq1}
\end{equation}
The inequality sign is given by the data processing inequality; the equality is given by the fact that \e{Z} and \e{\bm X^{(y)}} are independent conditional on \e{Y}. On the other hand,
\begin{equation}
    \small
    H(Z |  h(p^{tr}_Y(\cdot | \bm X^{(y)}), c), Y) \leq H(Z | Y).
    \label{eq:entropy_ineq2}
\end{equation}
Combining Eqs.~\eqref{eq:entropy_ineq1} and \eqref{eq:entropy_ineq2}, we have
\begin{equation}
\small
    H(Z |  h(p^{tr}_Y(\cdot | \bm X^{(y)}), c), Y) = H(Z | Y).
\end{equation}

According to Lemma~\ref{lemma:main}, when \e{p^{tr}(Z=z|\bm X_0^{(z)}=\bm \delta) \rightarrow 1},
\begin{equation}
    \small
    \begin{aligned}
     H(Z |  h(p^{tr}_Y(\cdot | \bm X^{(y)}), p^{tr}_Z(\cdot | \bm X^{(z)} = \tilde{\bm X}^{(z)})), Y) & \rightarrow H(Z |  h(p^{tr}_Y(\cdot | \bm X^{(y)}), c), Y) = H(Z | Y).
    \end{aligned}
\end{equation}
\end{proof}

\CR{
\subsection{Discussion on Feature Disentanglement Assumption}
In Section~\ref{sec:perspective}, we made a simplifying assumption that all features could be divided into two disentangled groups, \emph{i.e.}, $\bm X = [\bm X^y, \bm X^z]$, which are governed by the output label $Y$ and demographic information $Z$, respectively. 
The corresponding data generation process could be seen in Figure~\ref{fig:theory}.
On the other hand, however, if features are entangled in practice, we show that {\ours} can still provide false demographic information to overshadow the true demographics in Table~\ref{tab: demo_classifier}. }

\newpage

\section{Broader Impact}
\label{app: broader_impact}
Although there has been a proliferation of works in promoting ML fairness, most methods require training or finetuning the existing models to meet certain fairness notions.
However, this could bring large computational and storage costs, low data efficiency, and model privacy issues with those large-scaled trained models.

Inspired by recent advances in model reprogramming techniques, we propose a new generic post-processing fairness learning framework.
Specifically, we consider a fixed ML model and optimize a fairness trigger that is appended to the inputs with a min-max formulation.
The proposed method enjoys a better fairness-accuracy trade-off compared with vast fairness promoting baselines with far less training costs.

Despite the effectiveness of our method, we note that our method still has some limitations.
As a future research remark, our method still requires demographic annotations to remove biases, which could be hard to acquire in practice.
It remains an open problem to develop a fairness-promoting technique without the use of demographics annotations.

We do not observe any potential negative societal impacts of our method.
Instead, we believe that the outcome of our work could help enhance fairness of AI systems in a computationally-efficient and constraint-least manner. It can also provide broad positive impacts on diverse areas where AI techniques are applied.

\end{document}